\documentclass[11pt]{article}
\usepackage[margin=1in]{geometry}
\usepackage[utf8]{inputenc}
\usepackage[T1]{fontenc}
\usepackage{textcomp}
\usepackage{amsmath,amssymb,amsthm,amsfonts}
\usepackage{booktabs}
\usepackage{multirow}
\usepackage{graphicx}
\usepackage{xcolor}
\usepackage{array}
\usepackage{enumitem}
\usepackage[section]{placeins}
\usepackage{tabularx}
\newcolumntype{L}{>{\raggedright\arraybackslash}X}
\usepackage{tikz}
\usetikzlibrary{arrows.meta,fit,backgrounds,calc,positioning}
\usepackage[hidelinks]{hyperref}
\usepackage[numbers,sort&compress]{natbib}
\usepackage{authblk}
\usepackage{caption}
\theoremstyle{definition}
\newtheorem{definition}{Definition}
\theoremstyle{plain}
\newtheorem{observation}{Observation}
\definecolor{cexec}{HTML}{2a78d6}   
\definecolor{cauth}{HTML}{e87ba4}   
\definecolor{cok}{HTML}{008300}     
\definecolor{cwarn}{HTML}{eda100}   
\definecolor{cink}{HTML}{1B1B1B}
\definecolor{cgy}{HTML}{6B7280}
\newcommand{\best}[1]{\textbf{#1}}

\title{\vspace{-6mm}\textbf{Judging Is Not Enumerating:\\
Silent Omissions in LLM-Authored Acceptable Sets}}

\author[1]{Wenhui Chen\thanks{\texttt{mc35092@um.edu.mo}}}
\author[2]{Jianlin Chen\thanks{\texttt{202330450231@mail.scut.edu.cn}}}
\author[1]{Ziyao Lin\thanks{\texttt{mc35081@um.edu.mo}}}
\author[3]{Peiji Long\thanks{\texttt{longpeiji@gmail.com}}}
\author[1]{Chi Man Vong\thanks{Corresponding author. \texttt{cmvong@um.edu.mo}}}
\affil[1]{University of Macau}
\affil[2]{South China University of Technology}
\affil[3]{Independent Researcher}
\date{}

\begin{document}
\maketitle

\begin{abstract}
Language models are increasingly promoted from examinees to examiners: they write the test suites, answer keys, rubrics, and reward functions that define correctness for other systems. We measure the capability that role assumes and find it lacking under the protocol the role is usually deployed with, one-shot greedy authoring with no test-time reasoning. That scope condition changes the remedy, so we state it first: where a compact rule exists, test-time reasoning closes the gap outright on our algorithmic construction (\S\ref{sec:frontier}), a far cheaper fix than the probe gating we propose where it does not.

We use four reference constructions: two with \emph{complete} finite truth (mechanically decidable predicates over presented lists), one with a hardened executable reference (HumanEval+/MBPP+), and one with an explicitly incomplete lexical reference (WordNet). Across all four, models \emph{judge} whether a candidate belongs far better than they \emph{author} the set itself, enumerating it where it is finite and writing the suite that induces it where it is not. On the incompleteness-proof algorithmic construction the gap is $+0.34$ to $+0.29$ F1 across a $24{\times}$ parameter range and does not close over it. On executable code, models judging correctness at F1 $0.74$--$0.90$ author suites admitting only $19$--$42\%$ of oracle-correct solutions, over-rejecting by inventing requirements the specification never states.

A control locates the deficit. Asked to emit the \emph{predicate} rather than its extension, the same models reach F1 $\approx0.99$ with every generated predicate executable, above even their own judging. The failure is not missing knowledge or an inability to specify. It is an inability to \emph{materialise} the region a specification induces, which is exactly what keys, rubrics, and suites demand.

The dominant error is \emph{omission}, and omission resists audit: an over-inclusion is a written token a reviewer can challenge, while a missing member is an absence whose discovery is the authoring problem itself. Models detect planted over-inclusions $6$--$7\times$ more often than planted omissions, and a production deployment of $43{,}227$ scored items fails omission-first at $10{:}1$. Wired into RLVR, an authored key costs $1.9$ points of accuracy where the acceptable set is mechanically exact and $18.5$ points WordNet-relative on the lexical task, each in all six paired seeds (two-sided $p{=}0.031$).

One mitigation works: discard any authored verifier that rejects a known-correct probe. False rejection falls from $58$--$92\%$ to ${\leq}5\%$, but only $5$--$39\%$ of suites survive. Most of that loss is recoverable. Rewriting each wrong expected value to what a reference execution returns raises yield $3.3$--$10.6\times$ across four author families, by the largest factor for the weakest author, whose survivors go from $43\%$ vacuous to none. What survives measures the split: the model chooses discriminating inputs well ($94\%$ of wrong solutions still caught once expectations are corrected) and computes their expected outputs badly.
\end{abstract}

\section{Introduction}
\label{sec:intro}

A quiet role reversal is underway in how language-model systems are built.
Models no longer merely produce answers that some fixed verifier grades; increasingly they \emph{write the verifier}: unit tests generated from a docstring, answer keys attached to authored assessment items, rubrics used as reward signals \citep{gunjal2025rubrics}, and machine-checkable acceptance criteria wired directly into reinforcement-learning-with-verifiable-rewards (RLVR) pipelines.
Once authored, such a specification becomes the operational definition of ``correct'' for everything downstream: benchmark labels, agent feedback, reward gradients.
The literature has studied at length whether a given verifier can be gamed \citep{helff2026gaming,ray2026fuzzing}, how noisy verdicts propagate into RLVR \citep{cai2025noisyrewards}, and how verifiers falsely reject correct behavior, as when rule-based math checkers reject valid answer formats \citep{huang2025pitfalls} or benchmark suites fail valid solutions \citep{chowdhury2024swebench}.
All of it varies or audits the verifier while treating its origin as exogenous.
This paper asks the prior question, in the specific form our measurements can answer: \textbf{can a language model \emph{materialise} the acceptable set that a correctness specification denotes?} The question is not whether it can state a specification; \S\ref{sec:intensional} shows it can, near-perfectly. Recent benchmarks report symptoms of the answer being no: model-generated rubrics that lag human ones \citep{zhou2026rubricbench}, generated test suites too weak or misaligned to trust \citep{ma2025rethinking,ficek2025scoring}. We add more than another symptom. We establish the effect under \emph{complete} ground truth, where no incompleteness confound survives, identify its mechanism as silent omission, and measure what it costs as an RLVR reward.

The question has a clean formal core.
A specification over a candidate space $V$ is a membership predicate with acceptable set $S^{*}\subseteq V$; executing it means answering ``is $x\in S^{*}$?'' for one given $x$.
\emph{Authoring} it means producing an artifact that \emph{induces} an acceptance region $\hat S\subseteq V$, and it comes in two forms.
The \emph{extensional} form emits the members themselves, so $\hat S$ is the emitted list; it exists only where $S^{*}$ is small and finite, and it is what our three word-list constructions measure.
The \emph{intensional} form emits a predicate (for programs, a test suite), and $\hat S$ is whatever that predicate accepts under execution; where $V$ is unbounded, as for code, it is the only authoring interface there is.
In every arm we measure the same quantity: the fidelity of $\hat S$ to $S^{*}$, scored as a classifier over the same candidates on which we score the model's own execution, so ``the gap'' is one statistic throughout.
The two forms fail differently, and one further variable decides which failure is reachable: whether $S^{*}$ admits a \emph{compact intensional characterisation available to the author}.
On the mechanical constructions it does: the defining rule is stated in the prompt, so restating it as an executable predicate is near-free (F1 $\approx0.99$, \S\ref{sec:intensional}), and the measured deficit isolates extensional enumeration: emitting an unordered set under a causal mask, deciding element-by-element what to emit and, critically, when to \emph{stop}, without ever inspecting the candidates not sampled \citep{vinyals2016order,welleck2020consistency,hou2025seqenum}.
On code no such rule exists to restate: the specification is prose and examples. The forced intensional interface therefore fails in its own mode, inventing requirements the specification never states; the induced region ends up too small there too.
The unified claim is therefore not ``models cannot write predicates'': it is that models execute membership far better than they author artifacts whose induced region matches $S^{*}$, except where authoring collapses into restating a compact rule they were already given.
The generative-AI-paradox literature has established that models can generate what they cannot discriminate \citep{west2024paradox,jiang2024selfincorrect,song2024mindgap}.
The authoring problem is the harsher dual: here the model must \emph{generate a discriminator}, so every weakness of set-construction lands inside the artifact that will later define correctness for other systems.

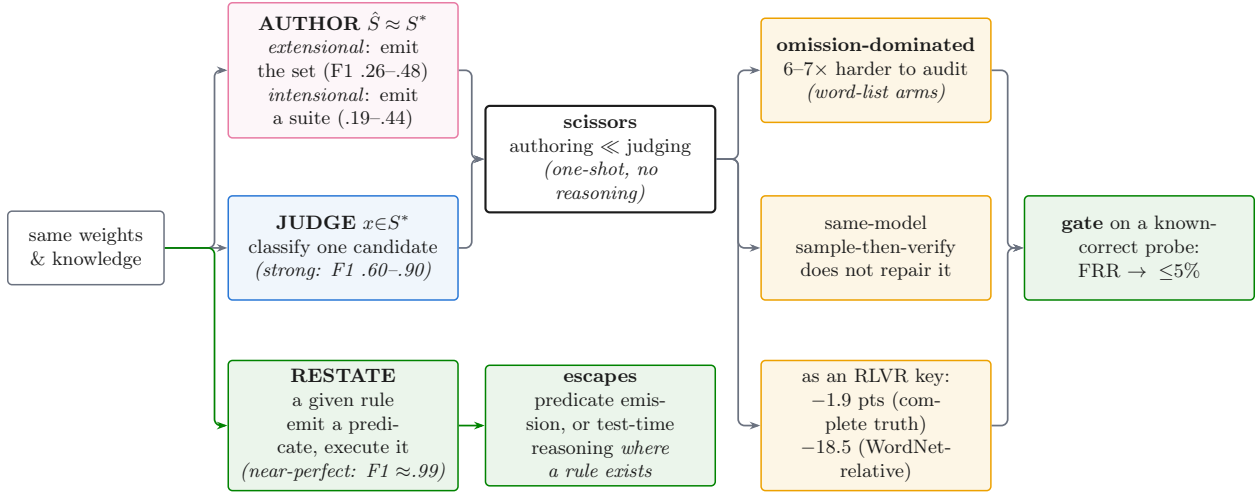
\begin{figure}[tbp]
\centering
\resizebox{\textwidth}{!}{%
\begin{tikzpicture}[
  >={Stealth[length=5pt,width=3.6pt]}, font=\footnotesize,
  bx/.style={draw=cgy, line width=0.7pt, rounded corners=2pt, align=center,
             inner xsep=5pt, inner ysep=4pt, fill=white, text=cink,
             minimum height=17mm, text width=34mm},
  src/.style={bx, text width=22mm, minimum height=12mm},
  au/.style={bx, draw=cauth, fill=cauth!7}, ex/.style={bx, draw=cexec, fill=cexec!7},
  sc/.style={bx, draw=cink, line width=1pt}, warn/.style={bx, draw=cwarn, fill=cwarn!10},
  ok/.style={bx, draw=cok, fill=cok!8},
  df/.style={->, line width=0.8pt, draw=cgy, rounded corners=3pt},
  dfg/.style={df, draw=cok},
]
\node[src] (w) at (0,0) {same weights\\\& knowledge};
\node[au] (au)  at (4.2, 2.9) {\textbf{AUTHOR} $\hat S\!\approx\!S^{*}$\\\emph{extensional}: emit the set (F1 .26--.48)\\\emph{intensional}: emit a suite (.19--.44)};
\node[ex] (ex)  at (4.2, 0.0) {\textbf{JUDGE} $x{\in}S^{*}$\\classify one candidate\\\emph{(strong: F1 .60--.90)}};
\node[ok] (inz) at (4.2,-2.9) {\textbf{RESTATE} a given rule\\emit a predicate, execute it\\\emph{(near-perfect: F1 $\approx$.99)}};
\node[sc] (scz) at (8.4, 1.45) {\textbf{scissors}\\authoring $\ll$ judging\\\emph{(one-shot, no reasoning)}};
\node[ok] (esc) at (8.4,-2.9)  {\textbf{escapes}\\predicate emission, or test-time\\reasoning \emph{where a rule exists}};
\node[warn] (sil) at (12.9, 2.9) {\textbf{omission-dominated}\\$6$--$7\times$ harder to audit\\\emph{(word-list arms)}};
\node[warn] (rep) at (12.9, 0.0) {same-model\\sample-then-verify\\does not repair it};
\node[warn] (rl)  at (12.9,-2.9) {as an RLVR key:\\$-1.9$ pts (complete truth)\\$-18.5$ (WordNet-relative)};
\node[ok] (mit) at (17.2,0) {\textbf{gate} on a known-\\correct probe:\\FRR $\to{\leq}5\%$};
\coordinate (wb) at (2.1,0);
\draw[df] (w.east) -- (2.1,0) -- (2.1, 2.9) -- (au.west);
\draw[df] (w.east) -- (ex.west);
\draw[dfg] (w.east) -- (2.1,0) -- (2.1,-2.9) -- (inz.west);
\draw[df] (au.east) -- (6.3, 2.9) -- (6.3, 1.45) -- (scz.west);
\draw[df] (ex.east) -- (6.3, 0.0) -- (6.3, 1.45) -- (scz.west);
\draw[dfg] (inz.east) -- (esc.west);
\draw[df] (scz.east) -- (10.65, 1.45) -- (10.65, 2.9) -- (sil.west);
\draw[df] (scz.east) -- (10.65, 1.45) -- (10.65, 0.0) -- (rep.west);
\draw[df] (scz.east) -- (10.65, 1.45) -- (10.65,-2.9) -- (rl.west);
\draw[df] (sil.east) -- (15.05, 2.9) -- (15.05, 0.0) -- (mit.west);
\draw[df] (rl.east)  -- (15.05,-2.9) -- (15.05, 0.0) -- (mit.west);
\end{tikzpicture}}
\caption{The paper in one view, and its boundaries. The same weights expose three interfaces to a correctness specification (column 2). Judging membership is strong; \emph{authoring} an artifact whose induced region should equal $S^{*}$ is much weaker in both of its realisations, extensional enumeration on the word-list constructions and test-suite authoring on code, and does not catch up over the scale range tested; restating a rule the prompt already supplied, as an executable predicate, is near-perfect. The deficit is therefore specific to one-shot materialisation of a region with no compact rule to restate, and it has two escapes (green path) that both require such a rule. Where none exists the failure is omission-dominated, resists local review, survives same-model self-repair, and costs accuracy as an RLVR key.}
\label{fig:overview}
\end{figure}

Measuring this cleanly requires ground truth that no LLM judge supplies, since using a judge to grade judge-authorship is circular.
We therefore build the entire measurement on \emph{constructed} ground truth, but no single construction is above suspicion.
A lexical construction (WordNet synsets \citep{miller1995wordnet}, where the acceptable set for ``a word meaning $g$'' is enumerated by the lexicon) is intuitive but \emph{incomplete}: the lexicon omits some genuine synonyms, so a model that authors a correct-but-lexicon-absent word is scored wrong, deflating apparent authoring precision.
We therefore anchor the paper on constructions where incompleteness is \emph{impossible}, and use agreement across all three, not any one, as the load-bearing evidence.
\textbf{(a) Complete truth:} mechanically decidable predicates (``contains a double letter'', ``has $\geq 8$ letters'') over a \emph{presented finite word list}; the acceptable set is exactly the list members the predicate accepts, computable and complete, so there is nothing for incompleteness to hide.
\textbf{(b) Executable truth:} HumanEval+ \citep{liu2023evalplus}, where a hardened oracle test suite decides code correctness by execution.
\textbf{(c) Lexical truth:} WordNet, retained as a third, independent construction whose \emph{recall}-side findings are incompleteness-robust (a missing lemma cannot inflate recall) and whose precision we report as a lower bound.
Every number in this paper is computed against one of these constructions; no model output is ever graded by a model.
The confound that could explain away any one construction (lexical incompleteness, oracle weakness, predicate triviality) is orthogonal across the three, so their agreement is not.

\paragraph{Findings in brief.}
Across the Qwen2.5 family (3B--72B) with replication on Qwen3, Llama, gemma, Mistral, and Phi-4: on the incompleteness-proof algorithmic construction judging rises $0.60\!\to\!0.77$ with scale while enumerating rises $0.26\!\to\!0.48$, so enumerating improves but does not catch up (gaps ${+}0.34/{+}0.25/{+}0.29/{+}0.29$, all CI pairs disjoint). On executable code the separation is largest: models judging at F1 $0.74$--$0.90$ author suites admitting only $19$--$42\%$ of correct solutions. Omission dominates the residual error and is $6$--$7\times$ harder to surface in review than over-inclusion. Asked for a predicate instead of an extension, the same models reach F1 $\approx0.99$. Used as an RLVR reward, an authored key costs $1.9$ points against a complete oracle, in every one of six paired seeds. Table~\ref{tab:claims} maps every claim to its evidence and its status.

\paragraph{Contributions.}
Figure~\ref{fig:overview} previews the argument end-to-end and Table~\ref{tab:claims} maps each claim to its supporting evidence (flagging the one falsified and two unconfirmed predictions).
(i) The judging--authoring asymmetry as a measured, judge-free, scale-resolved phenomenon. Authoring means extensional enumeration on the word-list constructions and, for the executable construction, the acceptance region an authored suite induces, since over an unbounded candidate space authoring is intensional by necessity (\S\ref{sec:formal}). The result is \emph{triangulated} across construction types whose confounds are mutually orthogonal: mechanically complete truth, executable truth, and lexical truth. No single incompleteness artifact can explain all three. The complete-truth construction also corrects an impression native to any lexicon-only study, that authoring does not scale. It does scale, just sub-critically, never catching execution over the range tested.
(ii) the silent-omission mechanism: a 6--7$\times$ detectability asymmetry that explains why \emph{subtractive} review pipelines shift authored specifications further toward omission rather than repairing them;
(iii) production-scale field evidence that model-authored answer keys fail omission-first and that even expert judging destabilizes exactly at the specification boundary;
(iv) two interventions, sample-then-verify repair (falsified) and RLVR propagation ($-1.9$ pts against a \emph{complete} oracle; $-18.5$ to $-32.0$ pts WordNet-relative across five policies spanning a $24{\times}$ scale range and a different family, the tax present at every scale, formalized as a two-sided reward channel, Observation~\ref{prop:reward}), plus a robustness sweep over scale, six model families/generations (Qwen2.5, Qwen3, Llama, gemma, Mistral, Phi-4), two executable benchmarks, two complete-truth constructions, and four authoring prompts, showing the gap robust across all;
(v) a working mitigation, gating authored verifiers on a known-correct probe, which cuts code false-rejection from $58$--$92\%$ to $\leq5\%$, turning the diagnosis into an actionable prescription; and (vi) a \emph{repair} for the majority the gate discards: rewriting each wrong expected value to what a reference execution returns recovers $3.3$--$10.6\times$ more usable suites across four author families, most where the gate works least, which additionally separates the two operations suite authoring requires and shows only one of them fails.

\section{Formal Setup}
\label{sec:formal}
\begin{table}[tbp]
\centering\scriptsize
\setlength{\tabcolsep}{4pt}
\renewcommand{\arraystretch}{1.0}\begin{tabularx}{\textwidth}{@{}>{\raggedright\arraybackslash}p{0.255\textwidth} L >{\raggedright\arraybackslash}p{0.15\textwidth}@{}}
\toprule
Claim & Evidence & Status\\
\midrule
Enumerating $\ll$ judging, one-shot; gap does not close with scale & complete truth $+0.34/{+}0.25/{+}0.29/{+}0.29$ (3B--72B, CIs disjoint) & confirmed, triangulated\\
The deficit is \emph{enumeration}, not knowledge or specification & same models write the predicate for the same sets at F1 ${\approx}0.99$ & confirmed, triangulated\\
Part of the gap is serialisation, not set construction & JSON-checkbox gains $+0.09/{+}0.02/{+}0.13$; $42\%$ of 3B tokens off-list & confirmed, single constr.\\
Authoring scales but sub-critically & algorithmic authoring F1 $0.26{\to}0.48$ & confirmed, single constr.\\
Omission is \emph{witness-poor}: harder to audit than over-inclusion & net of clean-key false alarms, $d_{\mathrm{over}}{=}{+}0.30$ [.22,.40] vs.\ $d_{\mathrm{omit}}{=}{+}0.09$ [$-$.01,.20] & confirmed, two constr.$^{\dagger}$\\
Field signature at production scale & 43K corpus: key errors omission-dominated $10{:}1$ & observational\\
Authored code suites over-reject correct implementations & authored F1 $0.07$--$0.60$ vs.\ execution $0.35$--$0.90$; suites admit $19$--$44\%$ of correct code & confirmed, triangulated\\
Sample-then-verify closes the scissors & repaired F1 $0.13$--$0.21\le$ one-shot & \textbf{falsified}\\
Authored keys cost accuracy as RLVR rewards & lexical $-18.5$ pts \emph{WordNet-relative}, $d_z{=}6.5$ & supported, 6/6 seeds, $p{=}0.031^{\ddagger}$\\
\quad --- same effect on \emph{complete} truth & $-1.9$ pts, $d_z{=}1.4$; small because the key is accurate (recall $0.93$) & supported, 6/6 seeds, $p{=}0.031^{\ddagger}$\\
Gap robust across scale, family, generation, benchmark, prompt & 8 models, 2 benchmarks, 4 constructions, 4 prompts & confirmed, triangulated\\
Test-time reasoning removes the gap where a compact rule exists & GPT-5.1 $+0.259\!\to\!{+}0.008$ (CI covers 0) & \textbf{closes}\\
Reasoning does not reliably repair test-suite authoring & Opus $+0.274\!\to\!{+}0.096$ (still ${>}0$); GPT-5.1 no improvement & mixed\\
Known-correct-probe gate mitigates & FRR $0.58$--$0.92{\to}{\leq}0.05$, at $5$--$39\%$ yield & confirmed$^{\S}$\\
Trace repair recovers most of the discarded suites & yield $3.3\times$ to $10.6\times$ over 4 author families, all at FRR ${\leq}0.06$, catch-wrong ${\geq}0.94$ & confirmed, 4 authors\\
Low gated false rejection is partly pool-proximity & vs.\ a non-Qwen pool, gated FRR $0.010{\to}0.064$ & confirmed\\
\bottomrule
\end{tabularx}
\caption{Claim--evidence map; each row's section gives the full evidence. Every construction number is computed against constructed ground truth; no LLM judge grades any construction measurement, except the field row (\S\ref{sec:field}), which is judge-relative by nature. One registered prediction was falsified, two sub-predictions were not confirmed, and a monotone scale trend in the RLVR tax was retracted after adding 72B. \emph{Status}: ``triangulated'' = holds across construction types with orthogonal confounds; ``single constr.'' = construct-generality untested. Markers: $^{\dagger}$this review condition only; $^{\ddagger}$small against a complete oracle, large only against the incomplete lexical reference; $^{\S}$weakest authors excepted, and repaired by the row below.}
\label{tab:claims}
\end{table}
Fix a finite candidate space $V$ and a specification whose acceptable set is $S^{*}\subseteq V$.
Two interfaces expose the same underlying knowledge.

\begin{definition}[Execution and authoring]
\label{def:interfaces}
The \emph{execution} interface, given a query $x\in V$, returns a membership judgment $\hat m(x)\in\{0,1\}$ approximating $\mathbb{1}[x\in S^{*}]$; it is scored by membership F1 over a query distribution that includes every member of $S^{*}$ (positives) and curated hard negatives.
The \emph{authoring} interface returns a set $\hat S\subseteq V$ approximating $S^{*}$; it is scored by set F1, $\;F_1(\hat S,S^{*})=\frac{2\,|\hat S\cap S^{*}|}{|\hat S|+|S^{*}|}$, with recall $\rho=|\hat S\cap S^{*}|/|S^{*}|$ and precision $\pi=|\hat S\cap S^{*}|/|\hat S|$.
\end{definition}

\begin{figure}[tbp]
\centering
\begin{tikzpicture}[font=\footnotesize, >={Stealth[length=4pt,width=3pt]}]
  \node[draw=cgy, line width=0.6pt, rounded corners=3pt, minimum width=9.4cm, minimum height=4.6cm] (V) at (0,0){};
  \node[cgy, font=\scriptsize] at (0,2.0) {$V$ (candidate space)};
  \fill[cok!14] (-1.0,0.1) ellipse (2.35 and 1.3);
  \fill[cauth!14] (1.0,0.1) ellipse (2.35 and 1.3);
  \begin{scope}
    \clip (-1.0,0.1) ellipse (2.35 and 1.3);
    \fill[cok!28!cauth!20] (1.0,0.1) ellipse (2.35 and 1.3);
  \end{scope}
  \draw[cok, line width=1pt] (-1.0,0.1) ellipse (2.35 and 1.3);
  \draw[cauth, line width=1pt] (1.0,0.1) ellipse (2.35 and 1.3);
  \node[cok, font=\scriptsize\bfseries, anchor=east] at (-2.5,1.45) {$S^{*}$ oracle};
  \node[cauth, font=\scriptsize\bfseries, anchor=west] at (2.5,1.45) {$\hat S$ authored};
  \node[align=center, text=cink] at (-2.45,0.1) {omission\\$S^{*}\!\setminus\!\hat S$};
  \node[align=center, text=cink, font=\footnotesize\bfseries] at (0,0.1) {TP\\$\hat S\!\cap\!S^{*}$};
  \node[align=center, text=cink] at (2.45,0.1) {over-incl.\\$\hat S\!\setminus\!S^{*}$};
  \node[red!70, font=\scriptsize, align=center] at (-1.95,-1.75) {witness-poor\\(starvation)};
  \node[cok, font=\scriptsize, align=center] at (1.95,-1.75) {review sees\\= pollution};
  \draw[->, red!55, line width=0.6pt] (-1.95,-1.4)--(-1.95,-0.62);
  \draw[->, cok, line width=0.6pt] (1.95,-1.4)--(1.95,-0.62);
\end{tikzpicture}
\caption{The geometry of an authored set. The authoring interface emits $\hat S$; scored against the oracle $S^{*}$ it has recall $\rho$ and precision $\pi$. Its two error regions behave oppositely under audit: over-inclusions are named tokens a reviewer or an RLVR reward can act on, while omissions are unnamed. These crescents are the starvation and pollution terms of Observation~\ref{prop:reward}; the witness-poorness of the left crescent is Observation~\ref{prop:omission}.}
\label{fig:geometry}
\end{figure}
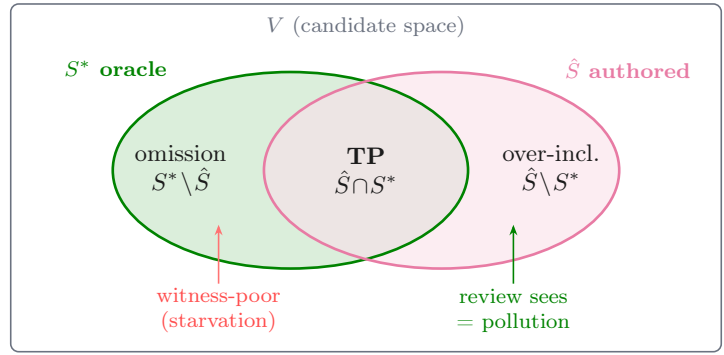

Both interfaces are queried on the same items and draw on the same weights; only the \emph{output type} differs, a pointwise bit versus an entire set emitted under a causal mask with a learned stopping rule.
\emph{Scope of the construct.} What we call ``authoring'' throughout is producing an artifact whose induced acceptance region $\hat S$ is meant to be $S^{*}$, and what we score is always the fidelity of $\hat S$, as a classifier over the same candidates used to score execution. Three of our four constructions realise this \textbf{extensionally}, emitting the members themselves, which is what answer keys and enumerated rubrics require; the executable construction realises it \textbf{intensionally}, since over the unbounded space of programs no extensional interface exists. What the paper does \emph{not} claim is that models cannot write predicates \emph{whose rule they were given}: \S\ref{sec:intensional} shows that restating such a rule executably is near-perfect. A test suite is also a predicate in the broad sense, and there the same models do fail. The difference is that its expected behaviour must be derived from prose and examples rather than restated (\S\ref{sec:intensional}, closing paragraph). The deficit is in materialising $\hat S$ when no such compact rule is available to restate: by enumeration where the set is finite, and by suite authoring where it is not.
The paper's object of study is the gap $\Delta = \mathrm{F1}_{\mathrm{exec}} - \mathrm{F1}_{\mathrm{auth}}$ as a function of scale, and its decomposition through $\rho$ and $\pi$ (Figure~\ref{fig:geometry}).
Two structural facts, established below, make the gap consequential rather than cosmetic: omission errors (the $S^{*}\setminus\hat S$ side) are \emph{witness-poor}: no local review can certify their absence (Observation~\ref{prop:omission}), and when $\hat S$ is used as a reward the gap becomes a two-sided reward channel that displaces the RLVR optimum (Observation~\ref{prop:reward}; Figure~\ref{fig:channel}): both are argued in Appendix~\ref{app:proofs}.

\section{The Judging--Enumerating Scissors}
\label{sec:scissors}

\subsection{Why the two interfaces differ: set construction vs.\ pointwise judgment}
\label{sec:whydiffer}

\begin{figure}[tbp]
\centering
\begin{tikzpicture}[
  >={Stealth[length=4.5pt,width=3.2pt]}, font=\footnotesize,
  bx/.style={draw=cgy, line width=0.6pt, rounded corners=2pt, align=center, inner xsep=4pt,
             inner ysep=3pt, minimum height=7mm, fill=white, text=cink},
  lm/.style={bx, draw=cink, line width=0.9pt, minimum width=15mm},
  ex/.style={bx, draw=cexec, fill=cexec!7}, au/.style={bx, draw=cauth, fill=cauth!7},
  df/.style={->, line width=0.7pt, draw=cgy}, lab/.style={font=\scriptsize, text=cgy},
]
\node[lab] at (-1.0,1.5) {\textbf{execution}};
\node[ex, text width=20mm] (ei) at (1.7,1.5) {spec $+$ one\\candidate $x$};
\node[lm] (elm) at (5.0,1.5) {LM};
\node[ex] (eo) at (8.0,1.5) {$\hat m(x)\!\in\!\{0,1\}$};
\draw[df] (ei)--(elm); \draw[df] (elm)--(eo);
\node[lab, align=center] at (11.3,1.5) {one pass,\\balanced 2-class};
\node[lab] at (-1.0,-0.2) {\textbf{authoring}};
\node[au, text width=14mm] (ai) at (1.3,-0.2) {spec only};
\node[lm] (alm) at (3.7,-0.2) {LM};
\node[au] (a1) at (5.7,-0.2) {$e_1$};
\node[au] (a2) at (6.9,-0.2) {$e_2$};
\node[au] (a3) at (8.1,-0.2) {$\cdots$};
\node[bx, draw=cwarn, fill=cwarn!12] (as) at (9.5,-0.2) {\texttt{stop?}};
\draw[df] (ai)--(alm); \draw[df] (alm)--(a1);
\draw[df] (a1)--(a2);\draw[df] (a2)--(a3);\draw[df] (a3)--(as);
\draw[df, densely dashed, draw=cgy!70] (a2.north) to[out=120,in=60] (a1.north);
\node[lab, align=center] at (11.6,-0.2) {enumerate $\hat S$,\\no lookahead,\\decide when to stop};
\end{tikzpicture}
\caption{The two interfaces to the same weights. \emph{Execution} (Eq.~\ref{eq:exec}) is one forward pass yielding a single bit on a curated, balanced query. \emph{Authoring} (Eq.~\ref{eq:auth}) emits the acceptable set left-to-right under a causal mask, conditioning each element on the prefix (dashed) and choosing when to stop, with no access to candidates it has not yet sampled. The paper measures how far apart these two readouts of the same knowledge sit.}
\label{fig:interfaces}
\end{figure}
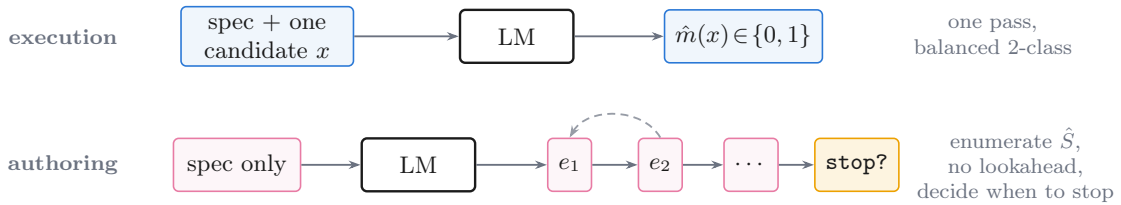

The two interfaces of Definition~\ref{def:interfaces} draw on the same knowledge but factor through radically different decoding problems (Figure~\ref{fig:interfaces}).
Execution answers a single presented query with one categorical decision,
\begin{equation}\label{eq:exec}
  \hat m(x)\;=\;\mathbb{1}\!\left[\,p_\theta(\texttt{yes}\mid \mathrm{spec},x) > \tfrac12\,\right],
\end{equation}
a balanced binary judgment on a curated query (each true member against a hard negative).
Authoring must instead emit an entire set autoregressively, choosing both the elements and, critically, when to \emph{stop}, with no lookahead over candidates it has not yet sampled:
\begin{equation}\label{eq:auth}
  p_\theta(\hat S=\langle e_1,\dots,e_k\rangle \mid \mathrm{spec})
   \;=\; \underbrace{p_\theta(\texttt{stop}\mid \mathrm{spec},e_{1:k})}_{\text{when to stop}}\;
         \prod_{i=1}^{k}\underbrace{p_\theta(e_i\mid \mathrm{spec},e_{<i})}_{\text{what to emit next}} .
\end{equation}
Scored against $S^{*}$, authoring F1 is the harmonic mean of its recall and precision,
\begin{equation}\label{eq:f1}
  \mathrm{F1}_{\mathrm{auth}}=\frac{2\rho\pi}{\rho+\pi}=\Bigl(\tfrac12\rho^{-1}+\tfrac12\pi^{-1}\Bigr)^{-1}\le 2\min(\rho,\pi),
\end{equation}
so it is throttled by whichever of recall or precision is , which is what makes the WordNet precision confound (\S\ref{sec:confound}) able to flatten $\mathrm{F1}_{\mathrm{auth}}$ even as $\rho$ climbs.

The algorithmic construction (\S\ref{sec:construction}) is designed to isolate the cost of Eq.~\eqref{eq:auth} from knowledge. Because the predicate is mechanically decidable over a presented list of $L$ words, the execution arm evaluates it pointwise and the induced ``execution-composed'' set $\hat S_{\mathrm{exec}}=\{x\in\mathrm{list}:\hat m(x)=1\}$ is exactly what a perfect enumerator would emit. This gives a clean decomposition.

\begin{observation}[Set-emission penalty; an interpretation of the measurement]\label{prop:setpenalty}
On a construction where the full candidate universe $U$ is \emph{presented to both arms}, so candidate discovery is controlled: the gap $\Delta=\mathrm{F1}_{\mathrm{exec}}-\mathrm{F1}_{\mathrm{auth}}$ measures the cost of emitting the acceptable subset \emph{in one pass} relative to judging each presented candidate, holding per-candidate knowledge fixed. It is an \emph{upper bound} on set-construction cost, not a pure measure: it still bundles one-shot batch labeling, coverage, and stopping (Eq.~\ref{eq:auth}). The matched-budget decomposition (App.~\ref{app:matched}) separates these components: emission-vs-labeling, batching, and the programmatic-union floor, which is near zero by construction.
\end{observation}
\noindent This is why the complete-truth constructions matter: the 15-item list is shown to both arms, so the residual $\Delta$ is not free-recall/candidate-discovery but one-shot set emission, and on the arithmetic construction execution is near-perfect so almost all of $\Delta$ is emission cost. For the \emph{open-set} constructions (code, lexical), where $U$ is not presented, a second and larger force is candidate discovery, compounded by base-rate dilution: authoring emits into a space where true members have base rate $b\ll\tfrac12$, so precision $\pi=\frac{b\,\mathrm{TPR}}{b\,\mathrm{TPR}+(1-b)\,\mathrm{FPR}}$ collapses as $b\to0$, whereas execution is scored on a curated balanced set ($b=\tfrac12$). We therefore attribute the open-set gap to discovery$+$emission jointly, and isolate the emission component only on the complete-truth constructions.

\subsection{Judge-free constructions}
\label{sec:construction}
Every construction instantiates the same two tasks over the same items, so that only the interface to the model's knowledge differs:
\textbf{Authoring (open-set)}: produce the complete acceptable set, scored by precision/recall/F1 of the normalized predicted set against the constructed oracle set; and
\textbf{Execution (closed-set)}: judge one presented candidate's membership (yes/no), posed for every oracle member (positive) and every hard negative, scored as membership F1 on the positive class plus raw accuracy.
Authoring F1 and execution membership F1 are the two comparable axes of the scissors. Construction and training details are in Appendix~\ref{app:details}, worked examples in Appendix~\ref{app:examples}, full per-model results in Appendix~\ref{app:full}, and a per-predicate breakdown in Appendix~\ref{app:predicates}.
We use three construction types whose incompleteness confounds are mutually orthogonal (the complete-truth type is instantiated by two constructions, string and arithmetic; four in total).

\textbf{(a) Algorithmic (complete truth).}
Each item presents a finite list of $15$ real English words (drawn from WordNet lemmas, length 3--11, deterministically shuffled, seed 0) and one mechanically decidable predicate: \emph{contains a double letter}, \emph{has $\geq 8$ letters}, \emph{contains ``th''}, \emph{ends in ``s''}, cycled across $n{=}240$ items, lists resampled to keep a non-degenerate accept/reject split.
The acceptable set is exactly $\{w\in\text{list}: \text{predicate}(w)\}$: computable, finite, and \emph{complete by construction}, so incompleteness cannot inflate any error.
Authoring asks the model to list every list-word satisfying the predicate; execution asks, for each list-word, whether it satisfies the predicate.
Any authoring$<$execution gap here isolates set-enumeration and stopping from world knowledge, because the predicate needs none.

\textbf{(b) Executable (execution truth).}
HumanEval+ \citep{liu2023evalplus}: authoring a test suite from a docstring, executing candidate-correctness judgment against an oracle-verified solution pool; full design and results in \S\ref{sec:code}.

\textbf{(c) Lexical (WordNet).}
Synsets with 3--8 single-word lemmas and a usage example, excluding highly polysemous targets ($>4$ senses), $n{=}240$ after deterministic sampling (seed 0).
Each item anchors a sense by gloss $g$ and example $e$; the oracle set is the synset's lemma set; hard negatives are co-hyponyms, hypernym lemmas, and antonyms.
Authoring: \emph{``List every single English word that can mean `$g$' (as in: `$e$').''}
Execution: \emph{``Can the word `$c$' mean `$g$' (as in: `$e$')? yes/no.''}
This construction is intuitive but \emph{incomplete}, the lexicon omits some real synonyms, so we treat its precision as a lower bound and lean on its incompleteness-robust recall side (\S\ref{sec:confound}).

Models: Qwen2.5-Instruct at 3B/7B/14B/72B \citep{qwen2025qwen25} (a single family, so scale is the only variable), with cross-family replication from Qwen3-14B/32B, Llama-3.2-3B/3.1-8B, Mistral-7B, and Phi-4 (algorithmic and code) and gemma-4-12B (lexical); greedy decoding throughout.

\subsection{Results: the scissors on complete truth}
\label{sec:scissors-results}

\begin{figure}[tbp]
\centering
\includegraphics[width=\textwidth]{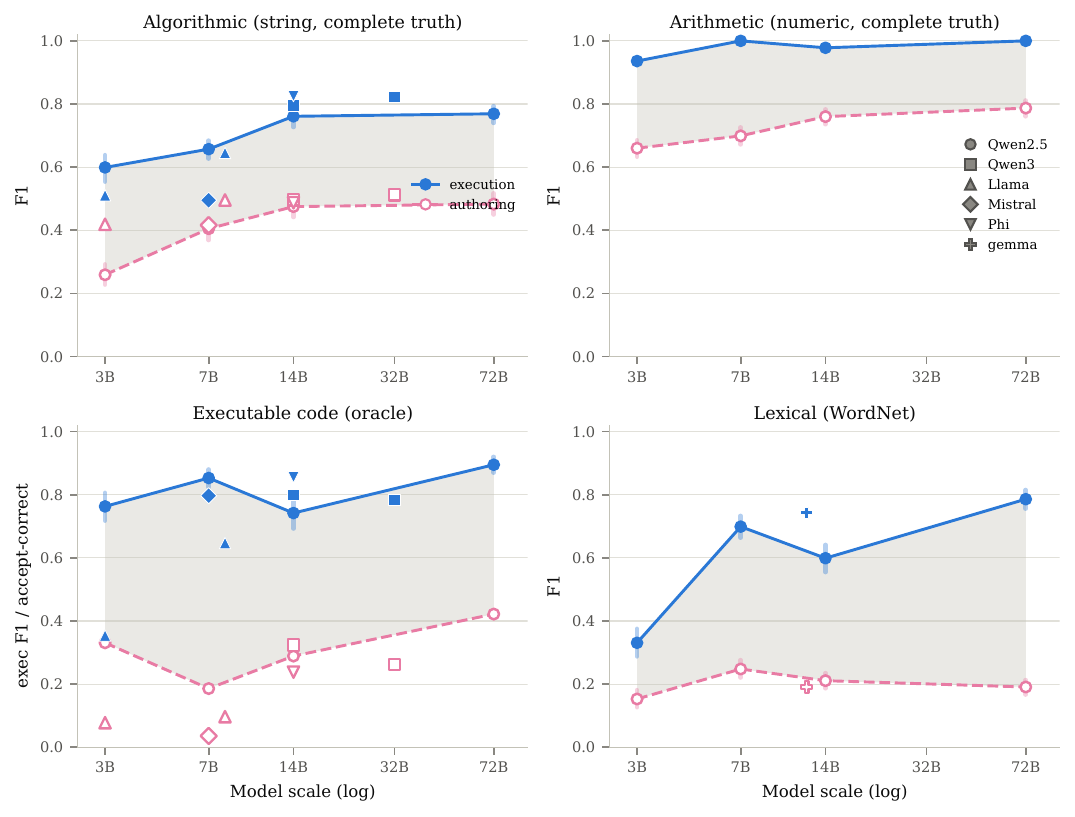}
\caption{The judging--enumerating scissors across four LLM-judge-free constructions and six model families/generations. Colour encodes interface (blue judging; magenta authoring, i.e.\ extensional enumeration on the word-list panels, test-suite authoring on the code panel, per \S\ref{sec:formal}); marker shape encodes family. Qwen2.5 is drawn as connected lines with 95\% bootstrap CI bands and the shaded gap; the other families are points at their parameter count. In every model--construction cell evaluated, judging sits above enumerating and the gap does not close over the tested scale range (coverage matrix in App.~\ref{app:full}). On executable code the separation is largest.}
\label{fig:scissors}
\end{figure}

\begin{table}[tbp]
\centering\small
\setlength{\tabcolsep}{6pt}
\begin{tabular}{llcccc}
\toprule
Family & Model & Authoring F1 [95\% CI] & Execution F1 [95\% CI] & Exec.\ acc. & Gap [95\% CI]\\
\midrule
\multirow{4}{*}{Qwen2.5}
 & 3B  & 0.259 [.23,.29] & 0.599 [.55,.64] & 0.826 & $+0.340$ [.29,.39]\\
 & 7B  & 0.405 [.37,.44] & 0.657 [.63,.68] & 0.735 & $+0.252$ [.22,.29]\\
 & 14B & 0.475 [.44,.51] & 0.761 [.73,.79] & 0.855 & $+0.286$ [.25,.32]\\
 & 72B & \best{0.483} [.45,.52] & \best{0.769} [.74,.79] & \best{0.884} & $+0.286$ [.25,.32]\\
\midrule
\multirow{2}{*}{Llama-3}
 & 3.2-3B & 0.439 [.41,.47] & 0.567 [.52,.61] & --- & $+0.128$ [.09,.17]\\
 & 3.1-8B & 0.485 [.46,.51] & 0.676 [.64,.71] & --- & $+0.190$ [.16,.22]\\
\bottomrule
\end{tabular}
\caption{Complete truth (algorithmic, $n{=}240$, greedy; bootstrap CIs). The acceptable set is computed by a mechanical predicate over a presented finite list, so no incompleteness is possible and every authoring error is genuine. Enumerating improves with scale ($0.26\!\to\!0.48$) but judging stays above it, every CI pair is disjoint, and the gap does not close (${+}0.34,{+}0.25,{+}0.29,{+}0.29$). Llama-3 reproduces the ordering at both sizes.}
\label{tab:algo}
\end{table}

Table~\ref{tab:algo} is the paper's central measurement, chosen because it is immune to the incompleteness confound (\S\ref{sec:confound}): the predicate decides membership mechanically over a list the model is shown, so a word is acceptable iff the predicate accepts it: there is no ``missing synonym'' to mislabel.
Three readings.
\emph{First}, the asymmetry is present at every scale and statistically clean: every authoring--execution CI pair is disjoint, in both families.
\emph{Second}, and correcting the impression a lexicon-only study would give, \textbf{authoring is not a flat line}: it improves substantially with scale (Qwen $0.259{\to}0.483$).
What does not happen is \emph{catching up}: the gap narrows from $0.34$ (3B) to $0.25$ (7B), then settles at $\approx0.29$ for 14B--72B rather than closing.
Authoring does improve with scale, over the small-model range even faster than execution, but the two never converge: the gap stops shrinking well above zero.
\emph{Third}, the effect is not Qwen-specific: Llama-3.2-3B and Llama-3.1-8B show the same execution$>$authoring ordering with disjoint CIs (gaps $+0.13$, $+0.19$); the gap is smaller for Llama than for Qwen because it pairs weaker execution with stronger authoring at this scale (Llama-3.2-3B authors $0.44$ vs.\ Qwen2.5-3B's $0.26$), so both differences shrink the gap.
Figure~\ref{fig:scissors} shows the same structure holding across all four constructions and all six model families/generations simultaneously.

\paragraph{A second complete-truth construction (arithmetic).}
To check that the complete-truth result is not specific to string predicates, we repeat the design with \emph{numeric} predicates (\emph{is even}, \emph{divisible by 3}, \emph{$>50$}, \emph{prime}) over presented integer lists: again mechanically decidable, again zero incompleteness, but disjoint in surface form and knowledge from the word-list construction (Table~\ref{tab:numeric}). The scissors is unchanged: execution reaches $0.94$--$1.00$ (the predicates are trivial to check pointwise) while authoring plateaus at $0.66$--$0.79$, a gap of ${+}0.28,{+}0.30,{+}0.22,{+}0.21$ with disjoint CIs. That execution is near-perfect here sharpens the reading of Observation~\ref{prop:setpenalty}: since pointwise membership is essentially solved and the candidate list is shown to both arms, the residual $\approx0.25$ gap is dominated by one-shot set emission rather than knowledge or discovery (decomposed in App.~\ref{app:matched}).

\begin{table}[htbp]
\centering\small
\setlength{\tabcolsep}{7pt}
\begin{tabular}{lccc}
\toprule
Model & Authoring F1 & Execution F1 & Gap [95\% CI]\\
\midrule
Qwen2.5-3B  & 0.660 & 0.936 & $+0.276$ [.25,.30]\\
Qwen2.5-7B  & 0.699 & 1.000 & $+0.301$ [.27,.33]\\
Qwen2.5-14B & 0.760 & 0.978 & $+0.218$ [.19,.24]\\
Qwen2.5-72B & \best{0.787} & 1.000 & $+0.213$ [.19,.24]\\
\bottomrule
\end{tabular}
\caption{A second complete-truth construction: numeric predicates over presented integer lists ($n{=}240$, 2{,}000-resample CIs). Execution is near-perfect (the predicates are pointwise-trivial), so the persistent $0.21$--$0.30$ gap is dominated by one-shot set emission (Observation~\ref{prop:setpenalty}; App.~\ref{app:matched}), independent of the algorithmic string construction.}
\label{tab:numeric}
\end{table}

\subsection{The executable construction: authored suites vs.\ oracles}
\label{sec:code}

The lexical and algorithmic constructions have sharp ground truth but a skeptic may ask whether the asymmetry survives where the community actually authors verifiers, unit tests, with correctness decided by \emph{execution}, the strongest ground truth available.
This is the executable construction of \S\ref{sec:construction}, and it produces the largest scissors in the paper.

\paragraph{Design.}
From HumanEval+ \citep{liu2023evalplus} (164 problems) we build an oracle-verified solution pool: 2{,}464 candidate solutions sampled from Qwen2.5-7B/14B ($K{=}8$, $T{=}0.8$), each graded by the full EvalPlus suite, yielding \textbf{1{,}560 oracle-correct} and 904 oracle-wrong solutions, a diverse pool of independently certified implementations.
The two arms, both scored by sandboxed execution:
\textbf{Execution}: present the docstring and one pooled solution and ask the model to judge correctness (yes/no), scored vs.\ the oracle label as membership F1;
\textbf{Authoring}, the model writes a test suite (\texttt{def check(candidate)}) from the docstring alone, and we run oracle-labeled solutions through it, reporting \emph{accept-correct} (fraction of oracle-correct solutions the suite admits), \emph{canonical acceptance} (does the suite even admit the reference solution), \emph{catch-wrong} (fraction of oracle-wrong solutions rejected; the suite must not be trivially permissive), and \emph{FRR$\mid$sane} (false-rejection rate among suites that pass the canonical check).
Both arms are run at all four Qwen scales and for Llama-3.2-3B/3.1-8B.

\begin{table}[tbp]
\centering\small
\setlength{\tabcolsep}{5pt}
\begin{tabular}{llccccc}
\toprule
Family & Model & Exec.\ F1 [95\% CI] & Accept-corr.\ & Canon.\ acc.\ [95\% CI] & Catch-wrong & FRR$\mid$sane\\
\midrule
\multicolumn{7}{l}{\emph{HumanEval+} (164 problems)}\\
\multirow{4}{*}{Qwen2.5}
 & 3B  & 0.763 [.72,.81] & 0.331 & 0.348 [.28,.42] & 0.937 & 0.043\\
 & 7B  & 0.853 [.82,.88] & 0.186 & 0.201 [.14,.27] & 0.980 & 0.051\\
 & 14B & 0.742 [.69,.79] & 0.289 & 0.268 [.20,.34] & 0.971 & 0.025\\
 & 72B & \best{0.895} [.87,.92] & \best{0.422} & \best{0.390} [.32,.46] & 0.966 & \best{0.014}\\
\midrule
\multirow{2}{*}{Llama-3}
 & 3.2-3B & 0.580 [.51,.64] & 0.209 & 0.213 [.15,.27] & 0.963 & 0.007\\
 & 3.1-8B & 0.766 [.72,.81] & 0.101 & 0.134 [.09,.19] & 0.976 & 0.000\\
\midrule
\multicolumn{7}{l}{\emph{MBPP+} (378 problems; pool built the same way)}\\
\multirow{4}{*}{Qwen2.5}
 & 3B  & 0.715 [.68,.75] & 0.357 & 0.288 [.24,.33] & 0.940 & 0.028\\
 & 7B  & 0.828 [.81,.85] & 0.211 & 0.156 [.12,.20] & 0.975 & 0.000\\
 & 14B & 0.787 [.76,.82] & 0.270 & 0.217 [.18,.26] & 0.972 & 0.026\\
 & 72B & 0.871 [.85,.89] & 0.441 & 0.339 [.29,.39] & 0.966 & 0.011\\
\bottomrule
\end{tabular}
\caption{Executable truth (HumanEval+ as a hardened executable reference, not a decision procedure for program equivalence; pool of 1{,}560 reference-passing / 904 reference-failing solutions; execution F1 with bootstrap CIs). Models that judge correctness at F1 up to $0.90$ author suites admitting only $19$--$42\%$ of correct solutions while still catching $94$--$99.8\%$ of wrong ones. Accept-correct is a recall, so it is not comparable to execution F1; Table~\ref{tab:unified} scores both arms with the same statistic. The population is bimodal: among suites that accept the reference, false rejection collapses to $\leq5\%$.}
\label{tab:code}
\end{table}

\begin{table}[tbp]
\centering\small
\setlength{\tabcolsep}{4.5pt}
\begin{tabular}{llcccccc}
\toprule
 & & \multicolumn{4}{c}{Authored suite as a classifier} & Exec.\ & Gap\\
\cmidrule(lr){3-6}\cmidrule(lr){7-7}\cmidrule(lr){8-8}
Bench & Model & Prec.\ & Rec.\ & F1 & MCC & F1 & $\Delta$F1 [95\% CI]\\
\midrule
\multirow{6}{*}{HE+}
 & Qwen 3B  & 0.900 & 0.331 & 0.484 & 0.306 & 0.763 & $+0.279$ [.18,.38]\\
 & Qwen 7B  & 0.942 & 0.186 & 0.311 & 0.243 & 0.853 & $+0.542$ [.45,.64]\\
 & Qwen 14B & 0.944 & 0.289 & 0.442 & 0.317 & 0.742 & $+0.300$ [.20,.40]\\
 & Qwen 72B & 0.955 & 0.422 & \best{0.586} & \best{0.417} & \best{0.895} & $+0.309$ [.23,.39]\\
 & Llama 3B & 0.907 & 0.209 & 0.340 & 0.236 & 0.580 & $+0.240$ [.14,.34]\\
 & Mistral 7B & 0.776 & 0.054 & 0.101 & 0.063 & 0.801 & $+0.700$ [.62,.77]\\
\midrule
\multirow{4}{*}{MBPP+}
 & Qwen 3B  & 0.913 & 0.357 & 0.514 & 0.330 & 0.715 & $+0.201$ [.14,.26]\\
 & Qwen 7B  & 0.936 & 0.211 & 0.344 & 0.254 & 0.828 & $+0.484$ [.42,.55]\\
 & Qwen 14B & 0.945 & 0.270 & 0.421 & 0.302 & 0.787 & $+0.366$ [.30,.43]\\
 & Qwen 72B & 0.958 & 0.441 & \best{0.604} & \best{0.429} & \best{0.871} & $+0.267$ [.22,.33]\\
\bottomrule
\end{tabular}
\caption{Same statistic, same pool. Both interfaces scored as binary classifiers over the identical oracle-labelled pool, so the gap is a paired difference of one quantity rather than execution F1 minus authoring recall. CIs are problem-cluster bootstraps, since solutions nest in problems. This matters asymmetrically: the execution arm contributes one independent judgment per solution, whereas the authoring arm contributes one artifact per problem whose verdicts on that problem's solutions are perfectly correlated by construction, so its effective sample size is the 164 (HumanEval+) or 378 (MBPP+) problems, not the pooled solution count. Resampling problems rather than solutions is what makes the two arms comparable; a solution-level bootstrap would understate the authoring arm's uncertainty by roughly the average cluster size. Every cell keeps a positive gap with a CI excluding zero under the clustered procedure. Authored suites are high-precision, low-recall: what they accept is almost always correct, but they admit only $4$--$44\%$ of correct solutions.}
\label{tab:unified}
\end{table}

\paragraph{Results.}
Authored suites do not err at the margin; they \emph{over-reject} (Table~\ref{tab:code}). Per the scope definition of \S\ref{sec:formal}, this arm scores the acceptance region an intensional artifact \emph{induces}, not an enumeration. The two levels then separate cleanly. At the artifact level the dominant error is over-specification: invented requirements and wrong expected outputs, with $70\%$ of suites rejecting the canonical solution by assertion (App.~\ref{app:examples}), not missing test cases. At the induced-set level the consequence is omission: correct implementations fall outside the acceptance region. The code construction therefore shares the paper's consequence, an acceptable set too small and invisibly so, while differing in mechanism, and we use it only at that consequence level: the enumeration mechanics of \S\ref{sec:whydiffer} and the detectability results of \S\ref{sec:silent} are established on the word-list arms and are not claimed for suites.
Across scale, models that judge correctness at execution F1 $0.74$--$0.90$ author suites admitting only $19$--$42\%$ of oracle-verified correct solutions while catching $94$--$99.8\%$ of wrong ones. Code authoring is non-monotone in scale (accept-correct $0.33/0.19/0.29/0.42$ at 3B/7B/14B/72B, so only 72B exceeds 3B) and stays far below judging everywhere; Llama's suites are stricter still ($8$--$10\%$).
The conditional structure is the sharpest finding: given that a suite accepts the canonical solution, its false-rejection rate on other correct implementations collapses to $\leq5\%$. The authored population is bimodal, a small sane minority and a large majority whose assertions encode incorrect expected behavior. Validating a generated verifier on known-correct solutions before trusting it, the discipline human re-annotation applies to benchmark suites \citep{chowdhury2024swebench}, is the difference between a $\leq5\%$ and a $58$--$92\%$ false-rejection instrument.

\emph{Second executable benchmark (MBPP+).}
Repeating the entire protocol on MBPP+ \citep{austin2021mbpp,liu2023evalplus} (378 problems, pool built the same way) leaves the scissors unchanged: execution F1 $0.72$--$0.87$ against accept-correct $0.21$--$0.44$, the same sub-critical scaling, the same near-zero FRR among canonical-sane suites, and the same high catch rate (Table~\ref{tab:code}, lower block).

\emph{Error-type audit.}
A dedicated one-shot 14B study (a separate, stricter prompt; canonical acceptance $20.7\%$) classifies every rejection by exception type and shows the failure is semantic, not mechanical: of 164 suites, 115 ($70\%$) run correctly and fail on assertions, meaning the suite executes but its author-computed expected outputs are wrong, while only 15 ($9\%$) crash. Among rejection events from broken suites, $87.5\%$ are assertion failures against $11.3\%$ runtime errors. Seven in ten authored verifiers are thus wrong-by-assertion: syntactically healthy artifacts confidently encoding incorrect expected behavior. This is also why catch-wrong stays high while accept-correct collapses. An over-constrained suite rejects almost everything, so its ability to reject wrong code is no evidence that it admits right code; since specificity stays $0.94$--$0.98$, the low accept-correct reflects genuine under-acceptance rather than a permissive suite inflating a balanced score.
Related SE work makes the direction plausible, finding that generated oracles encode the author's misreading of the spec \citep{oracles2024actual} and that benchmark tests unfairly reject valid solutions \citep{chowdhury2024swebench}, but it has not measured authoring against a certified-diverse correct pool with the model's own judging as the paired baseline.

\subsection{A working mitigation: gate on known-correct probes}
\label{sec:mitigation}
The repair arm (\S\ref{sec:repair}) shows the failure cannot be fixed by better authoring or self-verification.
But the bimodal structure of Table~\ref{tab:code}, in which broken suites reject the reference too while sane suites are near-perfect, implies a cheap \emph{external} gate: trust a model-authored suite only if it accepts a solution already known to be correct.
The canonical reference solution is the cheapest such probe, and the gate is one execution. We should not oversell its availability: on a benchmark a reference implementation is given, but on a real task obtaining one may be the very problem being solved, a single hand-verified input/output example is much weaker than a reference implementation, and gates correspondingly less. The honest scope is: wherever any known-correct behaviour can be exhibited, use it as a gate; that is strictly better than trusting the suite, and it is not always obtainable.
Table~\ref{tab:mitigation} and Figure~\ref{fig:mitigation} quantify the trade-off on the HumanEval+ suites (probe-strength sweep in Appendix~\ref{app:mitigation}).
Gating on the reference solution converts a 58--92\% false-rejection instrument into a $\leq5\%$ one (0\% for Llama), at the cost of \emph{yield}: only 20--39\% of Qwen suites (7--10\% of Llama's) survive the gate.
Holding out $k{=}1$--$3$ pooled correct solutions as probes instead of the canonical gives the same picture (FRR $0.01$--$0.08$).

\begin{table}[tbp]
\centering\small
\setlength{\tabcolsep}{5pt}
\begin{tabular}{llcccccc}
\toprule
Family & Model & Ungated FRR & Gated FRR & Yield & Gated catch-wrong & Gated BA & Vacuous\\
\midrule
\multirow{4}{*}{Qwen2.5}
 & 3B  & 0.669 & 0.043 & 0.348 & 0.826 & 0.892 & 0.053\\
 & 7B  & 0.814 & 0.051 & 0.201 & 0.914 & 0.932 & \best{0.000}\\
 & 14B & 0.711 & \best{0.025} & 0.268 & 0.919 & 0.947 & 0.023\\
 & 72B & 0.578 & \best{0.014} & \best{0.390} & 0.919 & \best{0.953} & 0.016\\
\midrule
\multirow{2}{*}{Llama-3}
 & 3.2-3B & 0.791 & 0.007 & 0.213 & 0.770 & 0.882 & 0.086\\
 & 3.1-8B & 0.899 & \best{0.000} & 0.134 & 0.857 & 0.929 & \best{0.000}\\
\midrule
Mistral & 7B-v0.3 & 0.946 & \best{0.000} & 0.049 & 0.542 & 0.771 & 0.250\\
\bottomrule
\end{tabular}
\caption{A working mitigation (HumanEval+, same authored suites as Table~\ref{tab:code}): discard any suite that rejects a known-correct probe. False rejection falls from $0.58$--$0.92$ to $\leq0.05$, at the cost of yield. The gate is not passed by being vacuous: survivors still reject $83$--$100\%$ of wrong solutions with $\leq5\%$ accepting everything. It is only as good as the author, though: for Mistral-7B a quarter of survivors are vacuous and gated catch-wrong falls to $0.54$. That defect is repaired rather than probed around: on an independently authored Mistral suite set whose gate leaves $43\%$ of survivors vacuous at catch-wrong $0.471$, trace repair returns catch-wrong $1.000$ with no vacuous survivors (\S\ref{sec:repair-trace}). (Llama's rows were originally far worse; that was a tokenization defect on our side, not the models'; see App.~\ref{app:bosfix}.)}
\label{tab:mitigation}
\end{table}

\begin{figure}[tbp]
\centering
\includegraphics[width=0.66\textwidth]{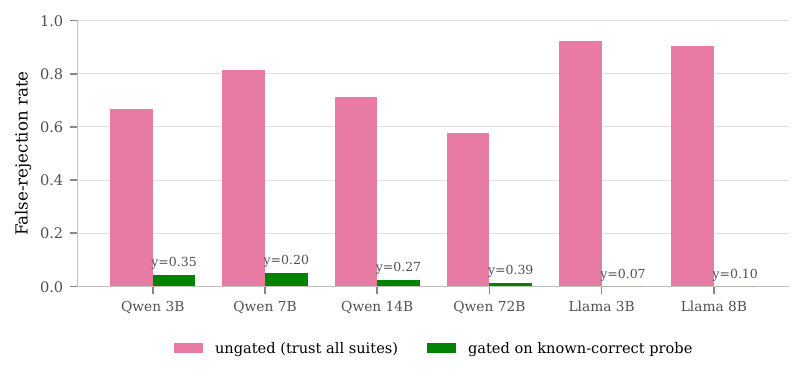}
\caption{The known-correct-probe gate (\S\ref{sec:mitigation}). Trusting all authored suites gives a $0.58$--$0.92$ false-rejection rate (magenta); discarding suites that reject the canonical reference collapses it to $\leq0.05$ (green). The annotation $y$ is yield, the fraction of authored suites that survive the gate; most do not, which is itself a restatement of the fail-closed result.}
\label{fig:mitigation}
\end{figure}

This is the actionable form of the paper's thesis: a model-authored verifier is usable only after external validation against known-correct behavior. The gate is not a fix for authoring but a filter, admitting the rare sane suite and rejecting the common broken one. Its cost is the discarded majority, and \S\ref{sec:repair-trace} shows that most of that majority need not be discarded at all.

\subsection{Repair instead of discard: execution can supply what authoring cannot}
\label{sec:repair-trace}
The gate throws away four suites in five, and our own audit says it need not: among rejection events from broken suites, $87.5\%$ are assertion failures and only $11.3\%$ runtime errors (\S\ref{sec:code}). An assertion failure means the model chose an input and then computed the wrong expected output for it, and that output is what executing a known-correct solution supplies. So rather than discard a suite that fails the gate, rewrite each wrong expectation to the value the reference returns, and re-gate. This also performs the decomposition \S\ref{sec:intensional} argues for: suite authoring needs an \emph{extensional} operation, choosing which inputs to check, and an \emph{intensional} one, computing what each should return. Repair hands the intensional half to execution, so what a repaired suite still catches measures the model's input selection with expectations held perfect.

Two rules keep the measurement honest. \textbf{No assertion is ever dropped} in the headline condition, since deletion would raise yield for free by making suites vacuous. And because repairing from the reference and then gating on it is tautological, \textbf{every reported quantity is measured on solutions the repair never saw}: false rejection on the held-out EvalPlus-correct pool, catch-wrong on the EvalPlus-wrong pool. A repair that ``succeeded'' by accepting everything would show as catch-wrong collapsing.

The discard is largely unnecessary (Table~\ref{tab:repair}). For the Qwen2.5-14B suites yield rises from $0.207$ to $0.690$, a $3.3\times$ recovery whose interval does not overlap the gate's, while false rejection ($0.010\!\to\!0.036$) and catch-wrong ($0.951\!\to\!0.946$) stay within overlapping intervals. Repair is a no-op where nothing was broken: of the $34$ suites that already passed, none regress and their metrics are unchanged to three decimals. \emph{The recovery grows as the author weakens.} Re-authoring the same $164$ specifications with three further families, the recovery factor rises monotonically as the gate's yield falls: $3.3\times$ at Qwen2.5-14B, $4.7\times$ at Llama-3.2-3B, $6.9\times$ at Llama-3.1-8B, $10.6\times$ at Mistral-7B. Every repaired population lands in the same band (FRR $0.035$--$0.059$, catch-wrong $\geq0.94$), so what differs across authors is how much there is to recover, not the quality of what is recovered. Mistral is the sharpest case and the one \S\ref{sec:mitigation} flagged as the gate's weak point: its survivors go from $43\%$ vacuous at catch-wrong $0.471$ to none vacuous at $1.000$. The known-\emph{wrong} probe we proposed for weak authors is unnecessary if their failed suites are repaired instead of filtered.

Three readings follow. \emph{(i) The extensional half is not the problem.} With expectations made perfect, the model's chosen inputs still reject $94.6\%$ of wrong solutions: it picks discriminating inputs well and computes their expected outputs badly, as \S\ref{sec:intensional} predicted. \emph{(ii) The residual failure is a different authoring error.} Every suite still failing after value repair contains an invented \emph{exception contract}, an \texttt{assert False} sentinel demanding the function raise on inputs the specification never mentions. That is a promise the spec never made rather than a miscomputed value, so substitution cannot touch it. Neutralising those sentinels deletes rather than corrects, so we report it only as a ceiling (yield $0.819$ at catch-wrong $0.942$); the neutralisation is self-limiting, since an \texttt{assert False} is reached only when the reference did \emph{not} raise. \emph{(iii) It does not rescue self-repair.} \S\ref{sec:repair} showed the model cannot fix its own suites; here the repairing agent is the \emph{interpreter}, not the model. Together the two are the paper's prescription in miniature: an authored artifact is salvageable, but only by an external oracle, never by its author.

This also bears on the word ``omission''. Our vocabulary comes from the word-list constructions, where the model genuinely cannot produce the missing member. The code failure has the same shape, an induced acceptance region that is wrong, but a different mechanism, and readers should not carry the lexical one into the code results. The scope is likewise narrow: repair requires a reference implementation, the assumption \S\ref{sec:mitigation} already flags as unavailable when obtaining one \emph{is} the task. Where it holds, discarding is the wrong default.

\emph{One caveat we can now quantify, and it cuts against us.} Both gate and repair are scored against a pool of oracle-verified solutions generated by Qwen2.5-7B/14B, so a Qwen-authored suite has been meeting code written in a style close to its own. A second pool from three non-Qwen families under an identical sampling budget ($1{,}781$ verified-correct solutions over $142$ problems) raises false rejection in every arm: gated Qwen2.5-14B $0.010\!\to\!0.064$, and repaired $0.036\!\to\!0.075$ (Qwen2.5-14B), $0.035\!\to\!0.123$ (Llama-3.1-8B), $0.059\!\to\!0.161$ (Mistral-7B). Validity is not what changed, since every solution in either pool passes the full plus-suite; the rise measures how much of the low figure came from stylistic proximity. The gate's headline should therefore read ``${\leq}0.05$ against implementations resembling the pool's authors, rising to ${\approx}0.06$--$0.16$ against more distant ones''. The qualitative conclusion is unchanged, since the worst cross-family figure is far below the $0.58$--$0.92$ of trusting suites ungated, and it sharpens what the residual over-rejection is: authored suites encode incidental properties of one implementation.

\begin{table}[tbp]
\centering\small
\setlength{\tabcolsep}{5pt}
\begin{tabular}{llccc}
\toprule
Construction & Direction & $D$ & $F$ & $d=D-F$ [95\% CI]\\
\midrule
\multirow{2}{*}{Algorithmic ($n{=}114$)}
 & over-inclusion & 0.943 & 0.640 & \best{$+0.303$} [$+.215,+.395$]\\
 & omission       & 0.636 & 0.544 & $+0.092$ [$-.013,+.197$]\\
\midrule
\multirow{2}{*}{Lexical ($n{=}240$)}
 & over-inclusion & 0.858 & 0.654 & \best{$+0.204$} [$+.133,+.273$]\\
 & omission       & 0.569 & 0.842 & $\mathbf{-0.273}$ [$-.337,-.208$]\\
\bottomrule
\end{tabular}
\caption{Repair recovers most of what the gate discards, and recovers more the worse the author is (HumanEval+, $n{=}164$ suites per author, problem-cluster bootstrap; the solution pool is held fixed so the author is the only variable). Trace repair rewrites each wrong expected value to what the reference solution actually returns. Every repaired population lands in a narrow band (FRR $0.035$--$0.059$, catch-wrong $\geq0.94$) while the gate alone keeps between $4\%$ and $21\%$ of suites, so the recovery factor is largest exactly where the gate is least usable. FRR is measured on held-out EvalPlus-correct solutions and catch-wrong on EvalPlus-wrong ones, never on the reference the repair read, so the numbers are not tautological; a repair that ``worked'' by accepting everything would show as catch-wrong collapsing. For Qwen2.5-14B, additionally neutralising the $63$ invented \texttt{assert False} contracts raises yield to $0.819$ at catch-wrong $0.942$; because that deletes rather than corrects an assertion it is a ceiling, not the method.}
\label{tab:repair}
\end{table}

\subsection{Frontier: newer generations and closed models}
\label{sec:frontier}
A natural rebuttal is that a stronger or newer model closes the gap, so it is a temporary artifact of mid-tier open models.
We test this two ways: a full \emph{generation} step on open weights (Qwen3 vs.\ Qwen2.5 at matched 14B, plus Qwen3-32B), and two \emph{closed frontier} models, OpenAI GPT-5.1 and Anthropic Claude Opus 4.8, run via API with reasoning disabled to match the one-shot protocol (Table~\ref{tab:closedfrontier}).
Table~\ref{tab:frontier} and Figure~\ref{fig:frontier} show the gap is invariant to the generation.
At matched 14B scale, Qwen3 improves both interfaces over Qwen2.5 (algorithmic execution $0.761{\to}0.795$, authoring $0.475{\to}0.497$; code execution $0.742{\to}0.800$, accept-correct $0.289{\to}0.324$) but the authoring--execution gap does not shrink, if anything it widens slightly ($+0.286{\to}+0.298$ algorithmic).
All results here are one-shot and greedy, with reasoning disabled where it is optional, so they characterize a \emph{specific protocol}: parameter count is varied, inference compute is held near-constant (the compute axis is varied separately in App.~\ref{app:matched}, where $16\times$ sampling raises union recall but not past execution), and we do not claim the gap would survive arbitrary decoding, reasoning budgets, or future training methods. Within that protocol, no open-weight axis we vary, whether scale (\S\ref{sec:scissors-results}), family (Llama, gemma), or generation (Qwen3), makes authoring catch execution; at the closed frontier the strongest models \emph{narrow} the gap on the simplest constructions (Opus algorithmic, GPT-5.1 numeric) but neither closes both, and both retain a large gap on code (Table~\ref{tab:closedfrontier}).

\begin{table}[tbp]
\centering\small
\setlength{\tabcolsep}{6pt}
\begin{tabular}{llcccc}
\toprule
& & Qwen2.5-14B & \textbf{Qwen3-14B} & Qwen2.5-72B & \textbf{Qwen3-32B}\\
\midrule
\multirow{3}{*}{Algorithmic}
 & Authoring F1   & 0.475 & 0.497 & 0.483 & 0.512\\
 & Execution F1   & 0.761 & 0.795 & 0.769 & 0.822\\
 & Gap            & $+0.286$ & $+0.298$ & $+0.286$ & $+0.309$\\
\midrule
\multirow{2}{*}{Code}
 & Execution F1     & 0.742 & 0.800 & 0.895 & 0.783\\
 & Accept-correct   & 0.289 & 0.324 & 0.422 & 0.262\\
\bottomrule
\end{tabular}
\caption{Frontier / cross-generation control. A full generation newer (Qwen3 vs.\ Qwen2.5), at matched 14B scale, lifts both interfaces but does not close the authoring--execution gap (it widens slightly on the complete-truth construction). Over the one-shot protocol tested, the gap does not shrink with scale, family, or generation (we do not claim invariance beyond that protocol; cf.\ the reasoning condition below).}
\label{tab:frontier}
\end{table}

\begin{figure}[tbp]
\centering
\includegraphics[width=0.82\textwidth]{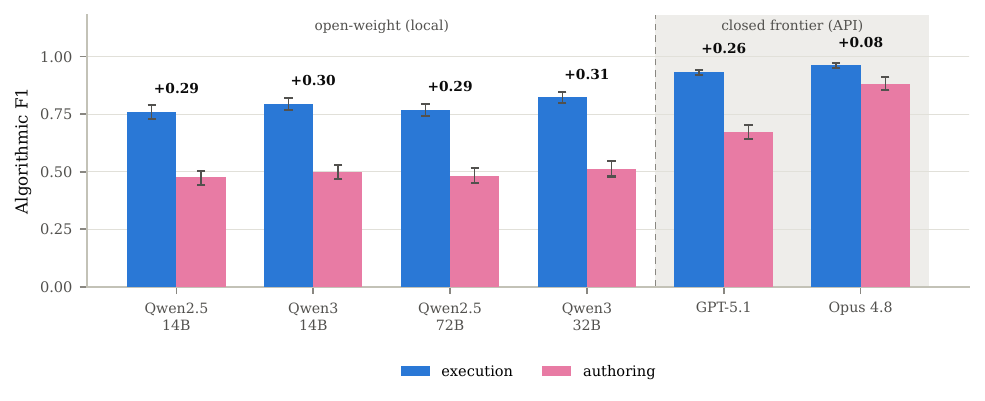}
\caption{The gap survives a model generation and the jump to the closed frontier. Left of the divider (open weights, local): Qwen3 raises both interfaces over Qwen2.5 at matched scale without shrinking the gap (bold numbers). Right (closed frontier, API, reasoning disabled to match the one-shot protocol): both models push judging to $0.93$--$0.96$ while enumerating lags, by $+0.26$ for GPT-5.1 and $+0.08$ for Opus, which nearly closes this construction but retains $+0.14$ on numeric and a large gap on code.}
\label{fig:frontier}
\end{figure}

\begin{table}[tbp]
\centering\small
\setlength{\tabcolsep}{9pt}
\begin{tabular}{llcc}
\toprule
Construction & metric & GPT-5.1 & Claude Opus 4.8\\
\midrule
\multirow{3}{*}{Algorithmic} & Authoring F1 & 0.673 & 0.881\\
 & Execution F1 & 0.932 & 0.961\\
 & Gap & $+0.259$ & $+0.080$\\
\midrule
Numeric & Gap & $+0.041$ & $+0.140$\\
\midrule
\multirow{2}{*}{Code (HE+)} & Execution F1 & 0.872 & 0.933\\
 & Accept-correct & 0.361 & 0.441\\
\midrule
\multicolumn{4}{l}{\emph{Reasoning \textbf{enabled}} (off-protocol; \texttt{effort=low}, $n{=}120$, algorithmic)}\\
\multirow{3}{*}{Algorithmic} & Authoring F1 & 0.976 & 0.964\\
 & Execution F1 & 0.984 & 0.999\\
 & Gap & $\best{+0.008}$ [$-$.01,.02] & $+0.035$ [.02,.05]\\
\bottomrule
\end{tabular}
\caption{Closed frontier (GPT-5.1, Claude Opus 4.8; $n{=}240$ per complete-truth construction, HumanEval+ pool for code). Frontier capability narrows the gap on the simplest constructions without abolishing it: each strongest model nearly closes one of them, neither closes both, and both retain a large gap on code. With reasoning enabled (bottom block) the algorithmic gap closes for GPT-5.1 (CI covers zero) and nearly so for Opus; \S\ref{sec:frontier} discusses why, and Table~\ref{tab:reasoncode} shows the same intervention on test suites. Gemini 3.x is omitted: its API mandates reasoning, so it cannot run the matched reasoning-off protocol.}
\label{tab:closedfrontier}
\end{table}

\begin{table}[tbp]
\centering\small
\setlength{\tabcolsep}{6pt}
\begin{tabular}{llcccccc}
\toprule
 & & \multicolumn{4}{c}{Authored suite as a classifier} & Exec.\ & Gap\\
\cmidrule(lr){3-6}\cmidrule(lr){7-7}\cmidrule(lr){8-8}
Model & Reasoning & Prec.\ & Rec.\ & F1 [95\% CI] & MCC & F1 & $\Delta$F1 [95\% CI]\\
\midrule
\multirow{2}{*}{GPT-5.1} & off & \best{1.000} & 0.387 & 0.558 [.42,.67] & 0.420 & 0.867 & $+0.309$ [.20,.44]\\
 & on & 0.974 & 0.316 & 0.477 [.34,.60] & 0.343 & 0.918 & $+0.441$ [.32,.58]\\
\midrule
\multirow{2}{*}{Opus 4.8} & off & 0.963 & 0.490 & 0.650 [.54,.75] & 0.454 & 0.924 & $+0.274$ [.18,.39]\\
 & on & 0.957 & \best{0.787} & \best{0.864} [.79,.92] & \best{0.683} & \best{0.960} & \best{$+0.096$} [.04,.17]\\
\bottomrule
\end{tabular}
\caption{Does test-time reasoning repair \emph{test-suite} authoring? Matched $2\times2$ on one fixed 80-problem HumanEval+ subsample, \texttt{effort=low}; both arms scored as classifiers on the same pool, CIs are problem-cluster bootstraps. It substantially helps one model and closes nothing. Opus's reduction is clean, its two $\Delta$F1 intervals being disjoint, yet the residual still excludes zero. GPT-5.1's point estimate moves the wrong way while recall and canonical acceptance fall, but its intervals overlap, so we claim only no improvement. Contrast Table~\ref{tab:closedfrontier}, where reasoning closes the algorithmic gap outright.}
\label{tab:reasoncode}
\end{table}

At the closed frontier (Table~\ref{tab:closedfrontier}), the scissors narrows but does not disappear. GPT-5.1 and Opus~4.8 execute the specification at F1 $0.93$--$0.96$ (algorithmic) and $0.87$--$0.93$ (code); their authoring improves markedly over open weights, and each strongest model nearly closes one simple construction: Opus on algorithmic ($0.881$ authoring, $+0.08$ gap) and GPT-5.1 on numeric ($+0.04$), but neither closes both (GPT-5.1 stays $+0.26$ on string predicates, Opus $+0.14$ on numeric), and both admit under half of oracle-correct code ($36$--$44\%$). Frontier capability therefore relocates the gap to the harder authoring domain rather than eliminating it; on open weights it persists at every scale.

\paragraph{Test-time reasoning closes it, where a rule can be executed.}
We then ran the strongest form of the ``a better model fixes it'' rebuttal that we had earlier only flagged: the same two models with \emph{reasoning enabled} (\texttt{effort=low}, $n{=}120$, algorithmic construction; Table~\ref{tab:closedfrontier}, bottom). It does fix it here. GPT-5.1's authoring rises $0.673\!\to\!0.976$ and its gap falls to $+0.008$ with a bootstrap CI covering zero; Opus reaches $0.964$ authoring and a $+0.035$ gap. We report this as a \emph{negative result for the universality of our headline}: on a construction whose acceptable set is generated by a mechanically executable rule, one-shot greedy decoding is not the only regime that matters, and test-time reasoning removes the deficit.
The mechanism is the one \S\ref{sec:intensional} identifies, which is why we find the result coherent rather than contradictory: a scratchpad lets the model \emph{apply the rule candidate-by-candidate}, internally reconstructing the intensional-then-execute route instead of emitting the set in a single pass. Reasoning and predicate-emission are two ways to buy the same escape.
\paragraph{But on test suites the escape is unreliable.}
We then ran the same condition where no compact rule generates the acceptable set: authoring a \emph{test suite}, in a matched $2\times2$ on a fixed 80-problem HumanEval+ subsample (Table~\ref{tab:reasoncode}). Reasoning does not transfer cleanly, and scoring both arms with the same statistic (as in Table~\ref{tab:unified}) makes the result precise. For \textbf{Opus} the improvement is large and statistically clean: authored-suite F1 $0.650\!\to\!0.864$ and $\Delta$F1 $+0.274\!\to\!+0.096$, with \emph{disjoint} problem-cluster intervals ($[.18,.39]$ vs.\ $[.04,.17]$), but the residual gap still excludes zero, so reasoning \emph{shrinks} this construction's gap without closing it. For \textbf{GPT-5.1} there is no improvement: F1 $0.558\!\to\!0.477$ and $\Delta$F1 $+0.309\!\to\!+0.441$, with recall $0.387\!\to\!0.316$ and canonical acceptance falling, consistent with a scratchpad making it \emph{over-specify} more, the failure of \S\ref{sec:code} amplified rather than repaired, though its two intervals overlap, so we claim only that reasoning does not help it, not that it significantly hurts. Precision stays $0.96$--$1.00$ and execution improves for both ($0.867\!\to\!0.918$, $0.924\!\to\!0.960$), so no suite population is becoming vacuous.
The pattern across the two constructions is therefore consistent with the mechanism and not with a blanket ``reasoning fixes authoring'': where the acceptable set is the extension of a rule the model can execute, a scratchpad recovers it (algorithmic, both models); where the set can only be approximated by enumerated cases, reasoning shrinks the gap for one frontier model without closing it and does not help the other. One boundary remains untested: WordNet, where no rule exists at all. \emph{Second}, it does not touch the downstream result: the RLVR keys of \S\ref{sec:propagation} were authored one-shot, which is how such keys are produced at scale, and the mitigation of \S\ref{sec:mitigation} is what one needs when they are. The honest summary is that the deficit is \textbf{real, mechanistically identified, and escapable}, by emitting rules or by paying test-time reasoning where a rule exists, and that it bites exactly when a pipeline does neither.

\subsection{The lexical construction and its confound}
\label{sec:confound}
The lexical (WordNet) construction is the most intuitive but the least trustworthy, and treating it carefully is what the complete-truth construction lets us do.
Its results (Table~\ref{tab:scissors}) reproduce the scissors, but with one number that must not be over-read.

\begin{table}[tbp]
\centering\small
\setlength{\tabcolsep}{5.5pt}
\begin{tabular}{lcccccc}
\toprule
 & \multicolumn{3}{c}{Authoring (open-set)} & Execution & \\
\cmidrule(lr){2-4}\cmidrule(lr){5-5}
Model & $P$ (lower bd.) & $R$ & F1 [95\% CI] & memb.\ F1 [95\% CI] & Gap [95\% CI]\\
\midrule
Qwen2.5-3B  & 0.153 & 0.233 & 0.153 [.13,.18] & 0.331 [.29,.38] & $+0.178$ [.13,.23]\\
Qwen2.5-7B  & 0.223 & 0.364 & \best{0.248} [.22,.28] & 0.699 [.66,.73] & $+0.451$ [.41,.49]\\
Qwen2.5-14B & 0.153 & 0.448 & 0.211 [.19,.24] & 0.599 [.56,.64] & $+0.388$ [.34,.43]\\
Qwen2.5-72B & 0.129 & \best{0.581} & 0.191 [.17,.21] & \best{0.786} [.76,.82] & $\best{+0.596}$ [.56,.63]\\
\midrule
gemma-4-12B & 0.132 & 0.518 & 0.191 [.17,.21] & 0.743 [.71,.78] & $+0.552$ [.52,.59]\\
\bottomrule
\end{tabular}
\caption{The scissors on \textbf{lexical} truth (WordNet, $n{=}240$, greedy; 2{,}000-resample bootstrap 95\% CIs). Reported for triangulation, with its confound made explicit: authoring \emph{precision} is a lower bound, because a valid synonym the lexicon lacks is scored as a false positive. The incompleteness-robust signal is \emph{recall} (a missing lemma cannot inflate it), which rises with scale ($0.233{\to}0.581$), so knowledge is increasingly present while F1 stays low. gemma-4-12B reproduces the gap.}
\label{tab:scissors}
\end{table}

The suspicious number is authoring precision, which \emph{falls} with scale ($0.153{\to}0.129$).
Read naively this says larger models emit more indiscriminate boundaries; but WordNet is incomplete, so a large model that correctly volunteers a real synonym the lexicon happens to lack is charged a false positive.
Precision here is therefore a \emph{lower bound}, and its apparent decline is what one expects if a stronger model proposes more genuine-but-lexicon-absent words.
We consequently do not rest any claim on WordNet precision.
The two WordNet signals we do use are both incompleteness-robust: recall (a missing lemma can only \emph{lower} it, so the observed rise $0.233{\to}0.581$ is a valid lower bound on how much the knowledge is present) and authoring recall itself. We must \emph{retract} a symmetry argument that would be convenient here: a missing lemma does not penalize the two interfaces equally. If the model volunteers a genuine-but-absent synonym while authoring, it is scored a false positive; but that same word is rarely \emph{queried} in the execution arm, whose positives are drawn from the synset itself, so execution pays no matching penalty unless the word happens to appear among the hard negatives. The lexical F1 gap can therefore be inflated by incompleteness, and we do not treat it as an unbiased comparison; we report the lexical arm as descriptive and let the complete-truth constructions carry the comparative claim.
Crucially, the complete-truth construction (Table~\ref{tab:algo}), where no synonym can be missing, shows the same gap, so the phenomenon survives removing the confound entirely, and WordNet's role is corroboration, not proof.
This is also why the ``authoring is flat'' reading of a WordNet-only study is an artifact: WordNet F1 looks flat because its confounded precision drags the rising recall back down; on complete truth, where precision is exact, authoring visibly scales.

\subsection{Prompt sensitivity: the gap narrows but does not close}
\label{sec:ablation}
A natural objection is that authoring merely needs a better prompt. We test this directly on the 14B model with four independently phrased authoring prompts per domain (the default plus three rewrites; Appendix~\ref{app:prompts}), holding the execution arm fixed (Table~\ref{tab:ablation}). Prompting does move authoring quality: code accept-correct ranges $0.29$--$0.51$ and algorithmic authoring F1 $0.46$--$0.57$ across prompts, so the effect is not a single unlucky wording. But no prompt closes the gap: the best code prompt still admits only $50.5\%$ of correct solutions against execution F1 $0.742$, and the best algorithmic prompt reaches authoring F1 $0.572$ against execution $0.761$. The scissors is prompt-sensitive in magnitude and prompt-invariant in direction; a better prompt buys a few points of authoring, not the boundary.

\begin{table}[htbp]
\centering\small
\setlength{\tabcolsep}{6pt}
\begin{tabular}{lccccc}
\toprule
 & default & rewrite 1 & rewrite 2 & rewrite 3 & execution (fixed)\\
\midrule
Algorithmic authoring F1 & 0.475 & 0.572 & 0.571 & 0.460 & 0.761\\
\quad gap vs.\ execution & $+0.286$ & $+0.188$ & $+0.190$ & $+0.301$ & ---\\
Code accept-correct & 0.289 & 0.376 & \best{0.505} & 0.387 & (F1 0.742)\\
\quad canonical acceptance & 0.268 & 0.390 & \best{0.470} & 0.384 & ---\\
\bottomrule
\end{tabular}
\caption{Prompt-sensitivity ablation (Qwen2.5-14B, four authoring prompts per domain, execution arm held fixed). Prompting shifts authoring quality but does not close the gap in any of the four phrasings: even the best prompt leaves authoring far below execution in both domains. The asymmetry is not a prompt artifact.}
\label{tab:ablation}
\end{table}

\subsection{Emission-format controls: how much is serialization?}
\label{sec:emission}
A fair objection to the authoring arm is that free-form listing bundles set construction with \emph{serialization}: the model must copy tokens, avoid duplicates, and decide when to stop. We hold the task, items, and gold sets fixed and vary only the emission format (Table~\ref{tab:emission}). \textbf{L} is the paper's free-list prompt; \textbf{J} replaces it with a JSON checkbox over every presented word, which removes both the copying-of-a-subset problem and the stopping decision, since the schema fixes the output's length and content; \textbf{S} re-runs L with the presented list permuted.

Three findings, one of which concedes part of the objection. \emph{(i) Format matters, and we were previously overcharging ``emission''.} Moving to the checkbox schema raises authoring F1 at every scale ($+0.09$/$+0.02$/$+0.13$ at 3B/7B/14B), so a real serialization component exists and the App.~\ref{app:matched} emission term should be read as set-construction \emph{plus} residual formatting. \emph{(ii) But it does not close the gap.} Even in the schema that eliminates copying and stopping, authoring stays far below the same model's pointwise execution ($0.300$ vs.\ $0.599$; $0.364$ vs.\ $0.657$; $0.528$ vs.\ $0.761$), a remaining $+0.23$ to $+0.30$. \emph{(iii) It is not an ordering artifact:} permuting the list moves F1 by $-0.02$ on average, within the bootstrap intervals.
The failure decomposition explains where the small-model deficit actually lives: at 3B, $42\%$ of emitted tokens are words \emph{not in the presented list} (hallucinated copies), against $5\%$ at 7B and $1\%$ at 14B. So a large part of the 3B gap is a copying failure, while the 14B gap, where copying is essentially clean and the schema still leaves $+0.23$, is set construction proper. The claim is therefore: the gap is partly serialization at small scale and almost entirely construction at large scale.

\begin{table}[tbp]
\centering\small
\setlength{\tabcolsep}{6pt}
\begin{tabular}{lcccc}
\toprule
Model & \shortstack{L free list\\(paper's format)} & \shortstack{J JSON checkbox\\(no copy/stop)} & \shortstack{S shuffled\\order} & \shortstack{Execution\\(pointwise)}\\
\midrule
Qwen2.5-3B  & 0.207 [.18,.24] & 0.300 [.26,.34] & 0.210 & \best{0.599}\\
Qwen2.5-7B  & 0.343 [.31,.38] & 0.364 [.33,.40] & 0.313 & \best{0.657}\\
Qwen2.5-14B & 0.396 [.36,.43] & 0.528 [.50,.56] & 0.373 & \best{0.761}\\
\bottomrule
\end{tabular}
\caption{Emission-format controls (algorithmic, $n{=}240$, 3B--14B; identical items, gold sets, and set-F1; item bootstraps). J replaces the free list with a JSON checkbox over every presented word, removing token copying and the stopping decision. It helps at every scale, so part of what App.~\ref{app:matched} charges to emission is serialisation, but it leaves a $+0.23$ to $+0.30$ deficit against the same model's judging, and shuffling the presentation order changes little. Caveats: L differs from Table~\ref{tab:algo} in token budget as well as format, so J-vs-L is the clean contrast; and J's parse success is $1.000/0.879/1.000$, so the 7B gain is partly schema non-compliance rather than set construction.}
\label{tab:emission}
\end{table}

\subsection{Intensional control: the deficit is enumeration, not specification}
\label{sec:intensional}
A central question of construct validity: does our authoring arm measure the inability to \emph{specify} a correctness boundary, or only the inability to \emph{enumerate} the set it induces? The two are separable, and separating them changes what the paper claims.
We therefore add a third interface on the same algorithmic items, gold sets, and scoring. Instead of listing the acceptable words (\emph{extensional} authoring), the model writes the specification as a Python predicate, \texttt{def accepts(w)}, which we then execute mechanically over the presented list (\emph{intensional} authoring). Nothing else changes.
The result is unambiguous (Table~\ref{tab:intensional}): intensional authoring is \textbf{near-perfect and scale-insensitive} ($\mathrm{F1}=0.999/0.978/0.999$ at 3B/7B/14B, $100\%$ of predicates executable, no runtime failures), \emph{far above} the extensional authoring of the identical specification ($0.259/0.405/0.475$) and above even pointwise execution ($0.599/0.657/0.761$).
Three consequences, stated carefully.
\emph{(i)} The gap our other constructions measure is an \textbf{enumeration/emission deficit, not a knowledge or specification deficit}: the same 3B model that emits only $26\%$ of the acceptable set can state the rule that generates it essentially perfectly. This is the mechanism behind the matched-budget decomposition (App.~\ref{app:matched}), where the residual at 72B is almost entirely the emission term.
\emph{(ii)} It bounds our claim. We do not show that LLMs cannot author correctness specifications in general; we show they fail to make an authored artifact's induced acceptance region match the intended one: extensionally for answer keys and enumerated rubrics, and intensionally for the test suites that actually require them. Where a compact intensional form exists \emph{and the model is allowed to emit it}, authoring is not the bottleneck.
\emph{(iii)} It makes an existing technique load-bearing for verifier construction. Emitting a program and delegating execution is the established idea of program-aided prompting \citep{gao2023pal,chen2023pot}; our contribution is not that technique but the \emph{measurement} that shows it is not an optimization here but the difference between $0.26$ and $0.99$, on the specific artifact, the acceptable set, that keys, rubrics, and suites consist of. The prescription (\S\ref{sec:implications}) is accordingly: never ask a model to enumerate an acceptable set it could instead \emph{characterize}. The catch is that this escape is unavailable exactly where verification is hardest: no predicate enumerates the synonyms of a gloss, and a test suite, though intensional in form, is not a restatement of any rule it was given.
\emph{Where a test suite sits on this axis.} Suite authoring is \emph{constrained} intensional authoring: it requires two operations at once, choosing which inputs to check, an extensional selection over the input space, and computing the expected output for each, which is the intensional part. Either can fail, and our data say which does. Of 164 one-shot 14B suites, $115$ ($70\%$) execute correctly and reject the canonical solution on an assertion, while only $15$ ($9\%$) crash; among rejection events from broken suites, $87.5\%$ are assertion failures against $11.3\%$ runtime errors (\S\ref{sec:code}). The dominant failure is therefore the intensional component, computing what the program should return, not the extensional one. This is also why the near-perfect predicate result and the poor suite result are consistent rather than contradictory: in \S\ref{sec:intensional} the rule to be expressed was supplied in the prompt and the model only had to restate it, whereas a suite author must \emph{derive} the expected behaviour from prose and examples. It is the derivation, not the act of writing a predicate, that these models fail at.

\begin{table}[tbp]
\centering\small
\setlength{\tabcolsep}{7pt}
\begin{tabular}{lccc}
\toprule
Model & \shortstack{Intensional authoring\\(write predicate) [95\% CI]} & \shortstack{Extensional authoring\\(enumerate set)} & \shortstack{Execution\\(pointwise)}\\
\midrule
Qwen2.5-3B  & \best{0.999} [.998,1.00] & 0.259 & 0.599\\
Qwen2.5-7B  & 0.978 [.971,.984] & 0.405 & 0.657\\
Qwen2.5-14B & \best{0.999} [.998,1.00] & 0.475 & 0.761\\
\bottomrule
\end{tabular}
\caption{The deficit is enumeration, not specification (algorithmic, identical items, gold sets, and set-F1; $n{=}240$; 3B--14B, where the effect is already saturated). Asked to \emph{write} the predicate rather than list its extension, every model is near-perfect and every generated predicate executes, above both its own extensional authoring and its own pointwise judging. This bounds the paper's claim to exhaustive acceptable-set construction and motivates the prescription of \S\ref{sec:implications}.}
\label{tab:intensional}
\end{table}

\subsection{Alignment control: the deficit predates instruction tuning}
\label{sec:base}
Every model measured so far is instruction-tuned, which admits an alternative reading of the whole paper: post-training rewards short, decisive answers, so what we call an inability to \emph{materialise} the acceptable set might be a learned disposition to \emph{stop early}. That reading would relocate the finding from a representational limit to an alignment artifact, and would change the remedy from externalised verification to a different fine-tuning recipe. It is testable, because a base checkpoint has never been under that pressure.
We therefore run the algorithmic construction on Qwen2.5-7B and Qwen2.5-7B-Instruct at one scale. A base model has no chat template, so both arms receive the \emph{identical} 3-shot plain-text prompt, the controlled comparison, and the instruct model is additionally run through its own template to show what the template is worth. Membership is computed by the predicate throughout.
The alternative reading does not survive (Table~\ref{tab:base}). Under the matched prompt the base and instruct checkpoints author the set \emph{equally well} ($0.384$ vs.\ $0.382$), and both retain a judging advantage ($+0.14$ and $+0.12$). Alignment is not the source of the deficit.
The mechanism evidence is sharper than the headline. If post-training had taught brevity, the never-aligned checkpoint should emit more of the set; instead it emits more than the set, $8.50$ words against a gold mean of $4.89$, a $74\%$ over-emission, and converts none of that surplus into accuracy. The base model is not withholding candidates; it is naming the wrong ones. Nor does the alignment stack hurt: the instruct model under its own chat template is the \emph{best} authoring arm of the three ($0.417$) and has the \emph{smallest} gap, the opposite of what the artifact hypothesis predicts.
Two caveats. First, the plain-text harness is weaker than the paper's main prompt for both arms (execution here is $0.50$--$0.52$ against $0.657$ for the same model in Table~\ref{tab:emission}), so the absolute gaps in this table are compressed and should not be compared to the headline; the load-bearing quantity is the base-vs-instruct contrast at a matched harness. Second, the base checkpoint frequently continues the few-shot pattern, fabricating further \texttt{List:}/\texttt{Rule:} blocks after its answer; we cut each completion at the first such marker, which is what keeps the fabricated blocks from inflating its score. Scoring only the first output line instead, which is correct for the base model but unfair to a chat model that opens with a preamble, leaves the base-vs-instruct comparison unchanged ($0.384$ vs.\ $0.374$) and depresses only the chat-template arm, an artifact of the parser rather than of the model. Both parsers are reported in the released receipt.

\begin{table}[tbp]
\centering\small
\setlength{\tabcolsep}{6pt}
\begin{tabular}{lccccc}
\toprule
Checkpoint & Prompt & \shortstack{Authoring\\F1} & \shortstack{Execution\\F1} & Gap & \shortstack{Mean emitted\\set size}\\
\midrule
Qwen2.5-7B (base)     & plain 3-shot   & 0.384 & \best{0.521} & $+0.138$ & 8.50\\
Qwen2.5-7B-Instruct   & plain 3-shot   & 0.382 & \best{0.499} & $+0.117$ & 6.78\\
Qwen2.5-7B-Instruct   & chat template  & 0.417 & \best{0.516} & $+0.100$ & 5.49\\
\bottomrule
\end{tabular}
\caption{The authoring deficit is not an instruction-tuning artifact (algorithmic construction, $n{=}240$, identical items, gold sets, and set-F1; gold sets average $4.89$ words). Rows 1--2 are the controlled comparison: one prompt, two checkpoints of the same family and scale, one of which has never been post-trained. They author the acceptable set equally well and both judge better than they author. The brevity hypothesis predicts the base model would emit more of the set; it emits $74\%$ \emph{more than the set} and is no more accurate, so its errors are commission rather than truncation. Absolute values are compressed relative to Table~\ref{tab:emission} because the plain-text harness required to make a base model comparable is weaker for both arms; the contrast, not the magnitude, is the result.}
\label{tab:base}
\end{table}

\subsection{What does not explain it}
Knowledge coverage cannot explain the gap at fixed scale, since both interfaces draw on the same weights over the same items, and enumerating capability itself rises with scale while judging stays ahead. Prompt format is bounded by the algorithmic construction, whose predicate is simple enough that no knowledge or format gap remains to blame.
\emph{Candidate access and compute.} The complete-truth constructions present the candidate list to both arms, so the classical free-recall-versus-recognition asymmetry is controlled by design. With candidates given and the assembly floor at zero, App.~\ref{app:matched} still finds a $0.24$--$0.26$ gap at every scale, and the decomposition is scale-revealing: the batching term vanishes by 72B while the emission term grows to $0.231$. Because the residual grows with scale it is not the length-driven enumeration failure of \citet{vinyals2016order,welleck2020consistency,hou2025seqenum}. Sampling budget does not explain it either, as $16\times$ authoring samples raise union recall only to ${\approx}0.80$; \emph{reasoning} budget partly does (\S\ref{sec:frontier}), so this claim is scoped to sampling at fixed decoding rather than to test-time compute in general. Part of the emission term is serialisation rather than set construction (\S\ref{sec:emission}).
Finally, enumeration-without-lookahead and retrieval limits are separated by the repair arm (\S\ref{sec:repair}), which finds both bind: sampled unions reach at most $55\%$ of the oracle set, and membership adjudication fails to restore precision on self-generated candidates.
\emph{Alignment} is ruled out at one scale by \S\ref{sec:base}: a never-post-trained base checkpoint authors no better than its instruct sibling under a matched prompt, and over-emits rather than truncates, so the deficit is not a learned disposition toward brevity.

\section{Why Local Review Is Asymmetric: Witness-Poor Omissions and Failed Self-Repair}
\label{sec:silent}

An authored specification can be wrong in two directions: it can include an invalid member or omit a valid one. The two are not symmetric to any reviewer, human or model, because review interrogates tokens that exist. An included error is present and can be challenged; an omitted member is an absence, and recovering it requires re-enumerating the open set, the operation \S\ref{sec:scissors} shows models perform poorly. The asymmetry is a reduction, not an accident (Figure~\ref{fig:silent}).

\begin{observation}[Omission is witness-blind to local review]
\label{prop:omission}
Grant a local reviewer even a \emph{perfect} membership oracle $m(x)=\mathbb{1}[x\in S^{*}]$, queryable only on candidates it can \emph{name}, with a query budget smaller than $|V\setminus\hat S|$ and no cardinality or coverage certificate for $S^{*}$. Then (i) every over-inclusion $x\in\hat S\setminus S^{*}$ is a named token, exposed by one query $m(x)=0$; but (ii) there exist two oracle sets $S_1^{*},S_2^{*}$ with $S_1^{*}\cap\hat S=S_2^{*}\cap\hat S$ and $S_1^{*}\setminus\hat S\neq S_2^{*}\setminus\hat S$ on which, unless the reviewer happens to name an element of $(S_1^{*}\triangle S_2^{*})\setminus\hat S$, its observations coincide, so it cannot certify recall completeness. Locating a specific omission requires producing a witness $x\in S^{*}\setminus\hat S$, an out-of-set search that is itself an instance of authoring. Stated with a perfect oracle the result is strongest: even an ideal judge cannot help, and the model's approximate judge only makes over-inclusion detection imperfect too, never improving the omission side.
\end{observation}

\noindent Three boundaries sharpen the claim rather than weaken it. First, Observation~\ref{prop:omission} is an information-theoretic indistinguishability statement, not an LLM-specific mechanism: no procedure, human or machine, can certify recall of an unknown set under a sub-universe query budget without a coverage certificate. It therefore licenses ``omission cannot be \emph{certified} absent'', not ``omission cannot be noticed''. Our own data make that distinction concrete: reviewers do volunteer out-of-set candidates, and at 72B they flood (5.46 spurious additions per item, Table~\ref{tab:silent}), so a reviewer is plainly not confined to named tokens. What we show empirically is the weaker, sufficient claim that subtractive review removes over-inclusions while generative review does not recover omissions at usable precision. Second, a cardinality or coverage certificate lets a reviewer detect incompleteness without a witness, but detection is not repair: fixing the gap still needs the missing witness, and such certificates are unavailable precisely when the acceptable set is the unknown being authored. Third, a reviewer that volunteers candidates may inject fresh errors. The net effect is a directional blind spot, so any subtractive review loop pushes an authored set toward under-acceptance. The rest of this section measures that asymmetry.

\begin{figure}[tbp]
\centering
\begin{tikzpicture}[
  >={Stealth[length=4.5pt,width=3.2pt]}, font=\footnotesize,
  tok/.style={draw=cgy, line width=0.6pt, rounded corners=2pt, minimum width=9mm, minimum height=6mm,
              align=center, fill=white, text=cink},
  good/.style={tok, draw=cok, fill=cok!8}, bad/.style={tok, draw=cwarn, fill=cwarn!12},
  ghost/.style={tok, draw=cgy!60, densely dashed, fill=gray!4, text=cgy},
  df/.style={->, line width=0.7pt, draw=cgy}, lab/.style={font=\scriptsize, text=cgy},
]
\node[good] (a1) at (0.7,1.6) {valid};
\node[good] (a2) at (2.0,1.6) {valid};
\node[good] (a3) at (3.3,1.6) {valid};
\node[bad]  (a4) at (4.9,1.6) {over-incl.};
\begin{scope}[on background layer]
\node[draw=cink!40, rounded corners=3pt, fit=(a1)(a2)(a3)(a4), inner sep=3mm] (setbox) {};
\end{scope}
\node[lab] at (3.0,2.5) {authored key $\hat S$ (tokens that exist)};
\node[tok, draw=cexec, fill=cexec!7, text width=22mm] (rev) at (7.7,1.6) {review queries $\hat m(x)$ on each \emph{named} token};
\draw[df] (a4.east) to[out=0,in=180] (rev.west);
\node[good, text width=20mm] (det) at (11.7,1.6) {over-inclusion\\\textbf{detected} \checkmark};
\draw[df] (rev.east) -- (det.west);
\node[lab] at (3.0,0.15) {oracle $S^{*}$ has a member not in $\hat S$};
\node[ghost] (om) at (3.0,-0.55) {omitted};
\node[tok, draw=cauth, fill=cauth!8, text width=22mm] (noq) at (7.7,-0.55) {no token exists to query};
\draw[df, densely dashed, draw=cgy!70] (om.east) to[out=0,in=180] (noq.west);
\node[bad, text width=20mm, draw=red!70, fill=red!5] (und) at (11.7,-0.55) {omission\\\textbf{undetected} $\times$};
\draw[df] (noq.east) -- (und.west);
\end{tikzpicture}
\caption{Why omission is witness-poor (Observation~\ref{prop:omission}). Review interrogates the tokens the artifact names, so an over-inclusion is exposed by a single membership query (top). An omitted member names no token; to query it the reviewer must first author it, the capability under audit (bottom). Subtractive review therefore removes over-inclusions and keeps omissions.}
\label{fig:silent}
\end{figure}
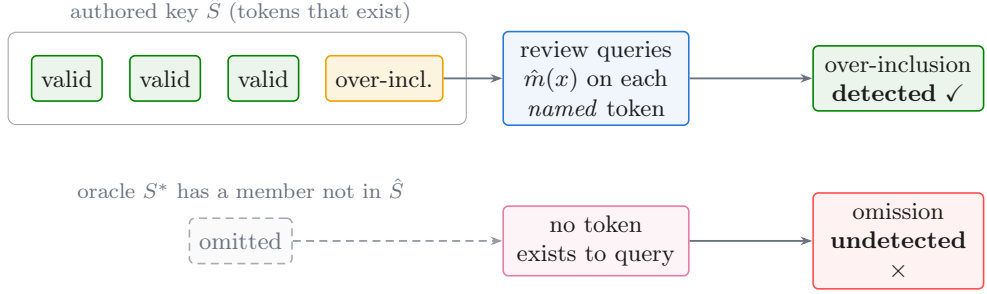

\subsection{Design}
For each of the 240 constructed items we present the model with a proposed answer key for the same gloss/example task under three conditions:
clean (the exact oracle set), \emph{over-corrupted} (oracle set plus two hard negatives), and \emph{omission-corrupted} (oracle set minus two valid members), instructing it to critique and repair the key in a fixed format (\texttt{REMOVE:} \dots\ / \texttt{ADD:} \dots).
Detection is scored objectively: an injected negative counts as caught iff it is named in \texttt{REMOVE}; a deleted member counts as recovered iff it is named in \texttt{ADD}.
Clean items measure false-alarm rates.

\subsection{Results}

\begin{table}[tbp]
\centering\small
\setlength{\tabcolsep}{6pt}
\begin{tabular}{lccccc}
\toprule
 & \multicolumn{2}{c}{Detection rate} & & \multicolumn{2}{c}{False alarms / clean item}\\
\cmidrule(lr){2-3}\cmidrule(lr){5-6}
Model & over-inclusion & omission & Asymmetry & \texttt{REMOVE} & \texttt{ADD}\\
\midrule
Qwen2.5-7B  & 0.706 & 0.100 & $7.1\times$ & 2.06 & 1.10\\
Qwen2.5-14B & \best{0.860} & 0.148 & $5.8\times$ & 1.83 & 1.18\\
Qwen2.5-72B & \best{0.860} & \best{0.375} & $2.3\times$ & 1.09 & \best{5.46}\\
\bottomrule
\end{tabular}
\caption{Silent omission (WordNet, $n{=}240$ per condition). Reviewers see a drafted key and may add or remove entries. Planted over-inclusions are caught $71$--$86\%$ of the time, planted omissions only $10$--$15\%$, a $6$--$7\times$ asymmetry that holds at every scale. The 72B model narrows it only by flooding the key with unsupported additions, which trades one error direction for the other rather than repairing recall. The $6$--$7\times$ figure is measured under the weakest review condition: one model, no auxiliary information; showing the full candidate pool and using a frontier reviewer changes the raw ratio but not the direction once clean-key false alarms are netted out (Table~\ref{tab:detectgen}), and further enriched conditions (\S\ref{sec:limitations}) may shrink it, and Observation~\ref{prop:omission} bounds only what local review can certify, not what a better-resourced reviewer could find.}
\label{tab:silent}
\end{table}

Table~\ref{tab:silent} confirms the prediction with a 6--7$\times$ asymmetry at 7--14B.
Three further observations sharpen it.
\emph{It is not reluctance to add:} on clean keys the models volunteer ${\approx}1.1$ additions per item (false alarms against the oracle); when two members are genuinely missing they recover them only 10--15\% of the time, the model is willing to extend the set but cannot find the specific absent members, blind search in an open space.
\emph{Scale improves the presence side first:} the 7B${\to}$14B gain concentrates on catching what is written ($+15.4$ vs.\ $+4.8$ points).
\emph{At 72B the asymmetry converts rather than closes:} the ratio drops to $2.3\times$, but the mechanism is flooding, not finding: false-alarm additions on clean specifications roughly quintuple ($4.6$--$5\times$; $1.1{\to}5.46$ per item) while genuine recovery reaches only 0.375.
A reviewer that ``fixes'' omissions by indiscriminate addition has traded the silent error for the visible one, the recall-for-precision trade of \S\ref{sec:scissors}; no scale we tested yields a reviewer that locates the specific missing members.

\paragraph{Is the asymmetry lexical-specific? A complete-truth replication.}
The measurement above uses the lexical construction, whose acceptable set is fuzzy: a reviewer
hunting a \emph{missing} synonym must also decide whether the missing word really carries the
intended sense. The asymmetry could therefore be an artifact of that fuzziness rather than a
property of omission. We close that escape by re-running the identical review protocol on the
algorithmic construction, where the reviewer is shown the full 15-word candidate list and a
mechanically decidable predicate, so every membership question has a checkable answer and finding
an omission requires only scanning the list. Both constructions are run with one frontier
reviewer (GPT-5.1) in the same session, so construction is the only variable
(Table~\ref{tab:detectgen}).

Two things follow, and the second reverses the first.
Raw detection rates make the frontier reviewer look far more symmetric than the 7--14B models
($1.4$--$1.5\times$ rather than $6$--$7\times$), and equally so on both constructions.
But those raw rates are not corrected for what the reviewer does to a clean key, and it
does a great deal: it volunteers additions on $64\%$ of clean algorithmic keys and $83\%$ of clean
lexical ones. Net of that base rate the picture inverts. On the algorithmic construction the
reviewer gains $+0.28$ over chance at catching planted over-inclusions but only $+0.04$ at
recovering planted omissions, a ratio of roughly seven; on the lexical construction the omission
figure is \emph{negative} ($-0.26$), meaning it proposes additions more readily on keys that need
none than it recovers members that are genuinely absent.
The asymmetry is therefore construct-general. It survives on complete truth, where lexical
fuzziness cannot explain it, and the frontier reviewer's apparent improvement is the same
flooding we already observed at 72B, now at the top of the capability range. What changes with
capability is the willingness to add, not the ability to find what is missing. Stated at the
strength the intervals support: on complete truth this reviewer detects over-inclusions well above
its own false-alarm rate and detects omissions at a rate we cannot distinguish from that rate.

\begin{table}[tbp]
\centering\small
\setlength{\tabcolsep}{5pt}
\begin{tabular}{lccccc}
\toprule
& \multicolumn{2}{c}{Raw detection} & \multicolumn{2}{c}{Clean-key false alarm} & \\
\cmidrule(lr){2-3}\cmidrule(lr){4-5}
Construction & $D_{\mathrm{over}}$ & $D_{\mathrm{omit}}$ & $F_{\mathrm{rem}}$ & $F_{\mathrm{add}}$ & Adjusted [95\% CI]\\
\midrule
\multirow{2}{*}{Algorithmic ($n{=}114$)}
 & \multirow{2}{*}{0.943} & \multirow{2}{*}{0.636} & \multirow{2}{*}{0.640} & \multirow{2}{*}{0.544}
 & $d_{\mathrm{over}}=\best{+0.303}$ [$+.215,+.395$]\\
 & & & & & $d_{\mathrm{omit}}=+0.092$ [$-.013,+.197$]\\
\midrule
\multirow{2}{*}{Lexical ($n{=}240$)}
 & \multirow{2}{*}{0.858} & \multirow{2}{*}{0.569} & \multirow{2}{*}{0.654} & \multirow{2}{*}{0.842}
 & $d_{\mathrm{over}}=\best{+0.204}$ [$+.133,+.273$]\\
 & & & & & $d_{\mathrm{omit}}=\mathbf{-0.273}$ [$-.337,-.208$]\\
\bottomrule
\end{tabular}
\caption{The detectability asymmetry replicated across constructions with one reviewer (GPT-5.1),
so construction is the only variable. Raw rates suggest near-symmetry, but the reviewer volunteers
removals on $64$--$65\%$ and additions on $54$--$84\%$ of clean keys, so a raw rate is
mostly base rate. $D$ is the detection rate for the planted error, $F$ the rate at which the reviewer volunteers the same edit on an uncorrupted key. We therefore report the false-alarm-adjusted quantities
$d_{\mathrm{over}}=D_{\mathrm{over}}-F_{\mathrm{rem}}$ and
$d_{\mathrm{omit}}=D_{\mathrm{omit}}-F_{\mathrm{add}}$, the hit-minus-false-alarm statistic of
signal-detection theory, with the clean condition supplying the false-alarm rate, each with a
$2{,}000$-resample item bootstrap. Reading the intervals rather than a ratio: over-inclusion
detection is reliably above its own false-alarm floor on \emph{both} constructions, whereas
omission detection is \emph{indistinguishable from} its floor on complete truth (the interval
covers zero) and significantly \emph{below} it on WordNet, where the reviewer adds spurious
members faster than it recovers real ones. We deliberately do not lead with the over/omit ratio:
its denominator's interval spans zero, so the ratio is unstable even though the ordering is not.
The asymmetry is thus not an artifact of lexical fuzziness. Item counts differ because the
algorithmic construction requires at least four accepted words per list, which its candidate pool
supplies for 114 of the 240 attempted lists. These figures come from a re-run that retains
per-item vectors so the intervals could be computed; an earlier run of the identical protocol gave
the same ordering with somewhat different magnitudes ($d_{\mathrm{over}}{=}{+}0.276$,
$d_{\mathrm{omit}}{=}{+}0.039$ algorithmic), which is run-to-run variation of the reviewer model,
not of the construction.}
\label{tab:detectgen}
\end{table}

\subsection{Consequence: subtractive review shifts error toward omission}
The asymmetry converts review pipelines from quality filters into \emph{directional} filters.
Any authored specification passed through self-critique, cross-model review, or adversarial revision will preferentially shed over-inclusions (visible, challengeable) while retaining omissions (invisible), so the surviving artifact is \emph{systematically under-accepting}, and the process that was meant to certify it is the process that skewed it.
This is a Goodhart mechanism one level above reward hacking: not a policy exploiting a reward, but the reward being born wrong in a direction its own audit cannot see.
Formally, a subtractive review round maps $\hat S_t\mapsto\hat S_{t+1}=\hat S_t\setminus D_t$, where by Observation~\ref{prop:omission} the detected set $D_t\subseteq\mathrm{OI}(\hat S_t)=\hat S_t\setminus S^{*}$ contains only over-inclusions (the sole errors it can name and challenge). Consequently
\begin{equation}\label{eq:reviewdrift}
  \pi_{t+1}\ge\pi_t,\qquad \rho_{t+1}=\rho_t,\qquad \mathrm{OM}(\hat S_{t+1})=\mathrm{OM}(\hat S_t)=S^{*}\setminus\hat S_t,
\end{equation}
so the omitted mass is a fixed point of review: iteration walks the artifact monotonically down-and-right on the precision--recall plane, sharpening precision while never recovering a missing member.
Section~\ref{sec:field} shows this signature verbatim in a production system whose generation pipeline included exactly such an adversarial review stage: over-inclusions falling round-over-round while omissions persist, exactly Eq.~\eqref{eq:reviewdrift}.

\subsection{Failed self-repair: sample-then-verify authoring}
\label{sec:repair}

If the scissors of \S\ref{sec:scissors} are a property of the one-shot generation interface (enumeration without lookahead, uncalibrated stopping) rather than of the model's knowledge, then the model's own execution ability should repair its authoring.
We re-architect authoring as \textbf{sample-then-verify}:
(1) sample the open-set task $K{=}10$ times at temperature 0.9 and take the union of normalized candidates (recall via diversity, the pass@$k$ logic);
(2) present each candidate to the same model as a closed-set membership query (precision via execution, the interface it is good at);
(3) the authored specification is the accepted subset (Figure~\ref{fig:repair}; full incident log in Appendix~\ref{app:negative}).
No external knowledge, no other model, no oracle access: the same weights, differently interfaced.

Registered predictions: if the asymmetry is interface-bound, repaired F1 substantially exceeds one-shot F1 at every scale, with precision recovered by filtering; if the union's recall stays low, retrieval binds; if filtering fails to restore precision, the boundary itself is diffusely represented.

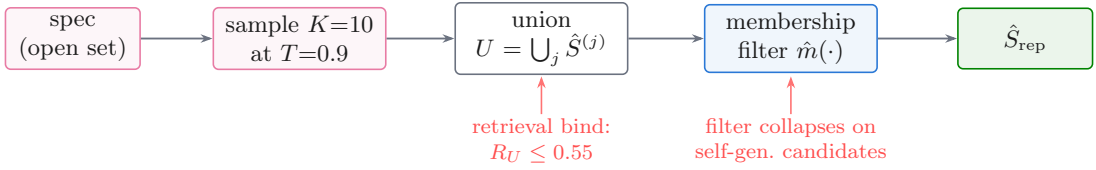
\begin{figure}[tbp]
\centering
\begin{tikzpicture}[
  >={Stealth[length=4.5pt,width=3.2pt]}, font=\footnotesize,
  bx/.style={draw=cgy, line width=0.6pt, rounded corners=2pt, align=center, inner xsep=4pt,
             inner ysep=3pt, minimum height=8mm, fill=white, text=cink},
  au/.style={bx, draw=cauth, fill=cauth!7}, ex/.style={bx, draw=cexec, fill=cexec!7},
  ok/.style={bx, draw=cok, fill=cok!8},
  df/.style={->, line width=0.7pt, draw=cgy}, bad/.style={font=\scriptsize, text=red!70, align=center},
]
\node[au, text width=15mm] (spec) at (0,0) {spec\\(open set)};
\node[au, text width=20mm] (samp) at (3.0,0) {sample $K{=}10$\\at $T{=}0.9$};
\node[bx, text width=20mm] (uni) at (6.2,0) {union $U=\bigcup_j \hat S^{(j)}$};
\node[ex, text width=20mm] (filt) at (9.5,0) {membership\\filter $\hat m(\cdot)$};
\node[ok, text width=15mm] (out) at (12.6,0) {$\hat S_{\mathrm{rep}}$};
\draw[df] (spec)--(samp);\draw[df] (samp)--(uni);\draw[df] (uni)--(filt);\draw[df] (filt)--(out);
\node[bad] at (6.2,-1.35) {retrieval bind:\\$R_U\le0.55$};
\node[bad] at (9.5,-1.35) {filter collapses on\\self-gen.\ candidates};
\draw[->,red!55,line width=0.6pt] (6.2,-1.0)--(6.2,-0.5);
\draw[->,red!55,line width=0.6pt] (9.5,-1.0)--(9.5,-0.5);
\end{tikzpicture}
\caption{Sample-then-verify authoring (\S\ref{sec:repair}) and its two failure points. Diversity sampling supplies recall and the model's stronger membership interface supplies precision, in principle. In practice the union misses nearly half of $S^{*}$ (retrieval bind) and the filter, strong on curated contrasts, collapses on the model's own near-boundary proposals; by Eq.~\eqref{eq:repairbound} the composed F1 cannot beat one-shot.}
\label{fig:repair}
\end{figure}

The pipeline has a hard ceiling. Writing $U=\bigcup_{j=1}^{K}\hat S^{(j)}$ for the sampled union with oracle recall $R_U=|U\cap S^{*}|/|S^{*}|$, and $\pi_f$ for the precision of the membership filter on $U$, the repaired set $\hat S_{\mathrm{rep}}=\{x\in U:\hat m(x)=1\}$ satisfies
\begin{equation}\label{eq:repairbound}
  \mathrm{F1}_{\mathrm{rep}}\;\le\;\frac{2\,R_U\,\pi_f}{R_U+\pi_f}\;\le\;2\min(R_U,\pi_f),
\end{equation}
because its recall cannot exceed the union's ($R_U$) and its precision cannot exceed the filter's ($\pi_f$). Repair therefore needs both factors high; the experiment finds both low ($R_U\le0.55$; $\pi_f$ near one-shot precision), which Eq.~\eqref{eq:repairbound} converts into the falsification below.

\begin{table}[tbp]
\centering\small
\setlength{\tabcolsep}{6pt}
\begin{tabular}{lcccccccc}
\toprule
 & \multicolumn{2}{c}{One-shot} & \multicolumn{2}{c}{Union ($K{=}10$, $T{=}0.9$)} & \multicolumn{3}{c}{Sample-then-verify}\\
\cmidrule(lr){2-3}\cmidrule(lr){4-5}\cmidrule(lr){6-8}
Model & F1 & ($P$/$R$) & $P$ & $R$ & $P$ & $R$ & F1\\
\midrule
Qwen2.5-3B  & 0.169 & (.13/.28) & 0.072 & 0.425 & 0.121 & 0.188 & 0.126\\
Qwen2.5-7B  & \best{0.229} & (.19/.34) & 0.100 & \best{0.548} & 0.161 & \best{0.491} & 0.210\\
Qwen2.5-14B & 0.192 & (.14/.40) & 0.057 & 0.521 & \best{0.167} & 0.393 & 0.207\\
Qwen2.5-72B & 0.168 & (.12/.42) & 0.045 & 0.490 & 0.094 & 0.442 & 0.143\\
\bottomrule
\end{tabular}
\caption{Sample-then-verify repair (algorithmic, $K{=}10$ at $T{=}0.9$, then self-filtered by membership). The registered prediction that self-verification closes the gap is falsified: repaired F1 does not beat one-shot authoring at any scale. Union recall saturates near half the gold set, and the membership judge degrades on the very self-generated near-boundary candidates it must filter. Precision/recall of the one-shot baseline are given for reference.}
\label{tab:repair}
\end{table}

\paragraph{Result: the prediction is falsified, and the falsification is informative.}
Table~\ref{tab:repair}: repaired F1 (0.126/0.210/0.207/0.143) shows no consistent improvement over one-shot authoring (0.169/0.229/0.192/0.168): it is lower at 3 of 4 scales and essentially tied at 14B ($0.207$ vs.\ $0.192$, overlapping intervals), so sample-then-verify does not close the scissors.
The failure decomposes into two binding constraints.
\emph{Retrieval binds}: even ten diverse samples surface at most 55\% of the acceptable set, the missing members of \S\ref{sec:silent} are not hiding in the sampling distribution.
\emph{The judge collapses off-distribution}: membership execution, strong on oracle members versus curated negatives (membership F1 $0.60$--$0.79$ for 7B--72B, Table~\ref{tab:scissors}), filters the self-generated candidate pool barely above chance because a model's own high-temperature proposals are concentrated exactly at the boundary where its calibration is worst \citep{xiong2024confidence,jacobs2024cloze}.
The strong-interface hypothesis is dead: the scissors are not an artifact of the one-shot prompt that a generation$+$execution decomposition repairs.
What survives is the stronger structural reading: \textbf{the acceptable-set boundary is not sharply represented in the weights under any interface we tested}: one-shot emission, high-temperature exploration, and pointwise self-adjudication all fail on it, each in its characteristic way.
Practically, this closes the escape hatch a skeptic (or a practitioner) would reach for first: authored specifications cannot be made trustworthy by asking the same model to check itself; external ground truth is not an optimization but a requirement.

\section{Propagation: Authored Keys as RLVR Rewards}
\label{sec:propagation}

The silent-omission mechanism matters most where authored specifications are consumed automatically: as reward functions.
RLVR theory treats verifier false negatives as a noise parameter \citep{cai2025noisyrewards}, and case studies of falsely-rejecting verifiers uniformly project that, used as rewards, they would penalize correct outputs \citep{huang2025pitfalls,chowdhury2024swebench}.
We execute that projection twice: once on the \emph{lexical} task (large effect, but scored against WordNet, an admittedly incomplete reference; see the caveat below) and once on a \emph{complete-truth} numeric task where the acceptable set is decidable by construction and no incompleteness is possible (\S\ref{sec:propagation-numeric}). The complete-truth run is the clean causal estimate; the lexical run is the larger but reference-relative one.
First, what the authoring gap does to a reward is not one-sided noise but a specific two-sided channel.

\begin{observation}[The authored reward is a two-sided channel]
\label{prop:reward}
Let the oracle set be $S^{*}$ and the authored key $\hat S$. Under the reward $R(x)=\mathbb{1}[x\in\hat S]$ (vs.\ oracle $R^{*}(x)=\mathbb{1}[x\in S^{*}]$), the two disagree exactly on $\hat S\triangle S^{*}$: outputs in $S^{*}\setminus\hat S$ get zero reward (\emph{starvation}) and outputs in $\hat S\setminus S^{*}$ get positive reward (\emph{pollution}). Two distortion measures follow, and it is important not to conflate them. The \emph{set-level} rates are $1-\rho$ (fraction of $S^{*}$ starved) and $1-\pi$ (fraction of $\hat S$ that pollutes), with $\rho,\pi$ the key's recall and precision. The \emph{policy-experienced} rates are the masses the policy actually places there, $\epsilon_{\text{starve}}=\Pr_{x\sim q_\theta}[x\in S^{*}\setminus\hat S]$ and $\epsilon_{\text{pollute}}=\Pr_{x\sim q_\theta}[x\in\hat S\setminus S^{*}]$; these coincide with the set-level rates only in degenerate cases (a policy uniform on the relevant set) and in general differ; the policy-experienced rates are what enter the RLVR gradient. The empirical ``reward pollution'' we report (Table~\ref{tab:propagation}, $0.179$) is the \emph{conditional} rate $\Pr_{q_\theta}[x\notin S^{*}\mid x\in\hat S]$, the policy-weighted analog of $1-\pi$, and the directly measurable ``fraction of rewarded rollouts that are wrong'', which differs from the unconditional mass $\epsilon_{\text{pollute}}$ by the factor $q_\theta(\hat S)$. In either case $R$ is blind in both directions at once, and no on-policy optimization against $R$ can distinguish its errors from the oracle intent, because $R$ is the loop's only definition of correct.
\end{observation}

\noindent The proposition dictates the measurements we report: set-level, the authored key has recall $\rho=0.507$ (so $1-\rho=0.493$ starved) and low precision; policy-level, we measure $\epsilon_{\text{pollute}}$ directly on trained-policy rollouts ($0.179$ under the authored key vs.\ $0.041$ under the oracle key, \S\ref{sec:propagation}). Reporting the policy-weighted mass, not just $1-\rho,1-\pi$, is what ties the identity to the observed training damage.

\begin{figure}[tbp]
\centering
\begin{tikzpicture}[
  >={Stealth[length=4.5pt,width=3.2pt]}, font=\footnotesize,
  bx/.style={draw=cgy, line width=0.6pt, rounded corners=2pt, align=center, inner xsep=4pt,
             inner ysep=3pt, minimum height=8mm, fill=white, text=cink},
  df/.style={->, line width=0.8pt, draw=cgy},
]
\node[bx, text width=17mm, draw=cink, line width=0.9pt] (pol) at (0,0) {policy $\pi_\phi$};
\node[bx, text width=18mm] (roll) at (3.1,0) {rollouts $x$};
\node[bx, draw=cwarn, fill=cwarn!12, text width=22mm] (key) at (6.6,0) {reward key\\$R(x)=\mathbb{1}[x\!\in\!\hat S]$};
\node[bx, draw=cink, line width=0.9pt, text width=17mm] (grad) at (10.4,0) {GRPO\\gradient};
\draw[df] (pol)--(roll); \draw[df] (roll)--(key); \draw[df] (key)--(grad);
\draw[df, rounded corners=3pt] (grad.north) -- (10.4,1.75)
      -- node[above, font=\scriptsize, text=cgy] {policy update} (0,1.75) -- (pol.north);
\node[red!75, font=\scriptsize, align=center] (star) at (4.6,-1.6) {$x\in S^{*}\!\setminus\!\hat S$:\\reward $0$ (\textbf{starvation}, $1{-}\rho{=}.49$)};
\node[cwarn!75!black, font=\scriptsize, align=center] (poll) at (8.7,-1.6) {$x\in\hat S\!\setminus\!S^{*}$:\\reward $1$ (\textbf{pollution}, $.18$)};
\draw[->, red!55, line width=0.6pt] (star)--(key.south west);
\draw[->, cwarn!70!black, line width=0.6pt] (poll)--(key.south east);
\node[cok, font=\scriptsize] at (6.6,0.95) {(oracle key $R^{*}=\mathbb{1}[x\!\in\!S^{*}]$: no leaks $\Rightarrow$ $+18.5$ pts)};
\end{tikzpicture}
\caption{Propagation as a two-sided reward channel (Observation~\ref{prop:reward}). An RLVR loop identical except for the reward key: the authored key withholds reward from correct rollouts in $S^{*}\setminus\hat S$ and grants it to incorrect rollouts in $\hat S\setminus S^{*}$. The reference key has neither leak. On complete truth the two arms end $1.9$ points apart; on the lexical task $18.5$ points apart in synset-membership accuracy.}
\label{fig:channel}
\end{figure}

\paragraph{Design.}
Task: given gloss and example, produce one word meaning $g$ (single-token-style rollouts, so RL is cheap and attribution clean).
Dataset: 2{,}400 WordNet items (train 2{,}040 / eval 360) built as in \S\ref{sec:construction}.
The 14B model authors an answer key per item (one-shot, greedy); measured against the reference synsets these keys have \textbf{recall 0.507}, and 2{,}157/2{,}400 items carry at least one omitted valid answer, the raw material of silent punishment.

\paragraph{What the lexical ``oracle'' is, and is not.}
Both arms here are scored by membership in the WordNet synset, which \S\ref{sec:confound} establishes is \emph{incomplete}. Three consequences follow. \emph{(i)} The evaluation metric is a \emph{reference-relative} quantity, synset-membership rate rather than semantic correctness; we name it accordingly below and reserve ``true accuracy'' for the complete-truth run. \emph{(ii)} Measured pollution is an \emph{upper bound}: a genuine, context-appropriate synonym that WordNet omits is counted as pollution. \emph{(iii)} Most importantly, the reference-key arm enjoys a \emph{criterion-alignment advantage}: it is trained on the same membership predicate used to score it, while the authored arm is trained on a proper subset of that predicate. This experiment therefore establishes that \emph{training on a key that disagrees with the evaluation criterion costs performance on that criterion}; it does not by itself establish an equally large loss in semantic correctness. We must be careful about how far this can be bounded. Writing $a$ and $r$ for the fraction of each arm's off-synset outputs that are genuinely correct, semantic accuracy is $0.722+0.278a$ (authored) versus $0.913+0.087r$ (reference). The measured gap survives for any plausible configuration; e.g.\ it persists even at $a=0.7,r=0.1$, but it is not mathematically excluded that the ordering reverses in the extreme ($a\!\to\!1$, $r\!\to\!0$ gives $1.000$ vs.\ $0.913$). Without human adjudication of off-synset answers we therefore cannot claim a one-sided bound, only that reversal requires the authored arm's off-reference outputs to be almost entirely valid while the reference arm's are almost entirely invalid. The complete-truth replication in \S\ref{sec:propagation-numeric}, where the acceptable set is exact and no such caveat applies, is what licenses the causal reading.
The authored key is wrong in both directions at once: evaluated over the base policy's own rollouts it is not only starved but \emph{polluted}: of every answer it pays reward for, $25.8\%$ are oracle-invalid (a dedicated base-policy evaluation; the multiseed base arm gives $0.264$, Table~\ref{tab:propagation}), so the training signal simultaneously withholds reward from correct behavior and grants it to incorrect behavior.
Two GRPO \citep{shao2024deepseekmath} arms on a Qwen2.5-3B policy, identical in everything except the reward:
\textbf{authored-key arm}, reward $=$ membership of the rollout in the authored key;
\textbf{oracle-key arm}, reward $=$ membership in the full synset.
Instrumented against the \emph{oracle} throughout: per-step frequency of correct-but-omitted answers (rollouts the oracle accepts but the authored key rejects), reward pollution (rollouts the authored key rewards but the oracle rejects), oracle-valid rate, and distinct-valid coverage (how many different correct answers the policy can still produce at evaluation).

Registered predictions: the authored-key arm suppresses correct-but-omitted answers and narrows distinct-valid coverage relative to the oracle arm; oracle-valid rate may remain superficially comparable.

\begin{table}[tbp]
\centering\footnotesize
\setlength{\tabcolsep}{4pt}
\begin{tabular}{lcccc}
\toprule
Arm (reward key) & Synset-membership & Reward pollution & Omitted frac.\ of valid & Distinct valid / item\\
\midrule
base (no training)   & 0.567 & 0.264 & 0.112 & 1.064\\
authored key         & 0.726 {\scriptsize$\pm$.026} & 0.179 {\scriptsize$\pm$.019} & 0.121 {\scriptsize$\pm$.008} & 1.032 {\scriptsize$\pm$.015}\\
oracle key           & \best{0.911} {\scriptsize$\pm$.003} & \best{0.041} {\scriptsize$\pm$.002} & 0.156 {\scriptsize$\pm$.001} & 0.970 {\scriptsize$\pm$.004}\\
\bottomrule
\end{tabular}
\caption{E6 lexical arm (mean$\pm$std, 6 seeds): identical GRPO runs differing only in the reward key, evaluated by WordNet synset membership. This is a \emph{reference-relative} metric, not semantic correctness, and the reference arm is additionally criterion-aligned (\S\ref{sec:propagation}); the complete-truth estimate is in \S\ref{sec:propagation-numeric}. Training on the authored key forfeits $18.5$ points relative to the reference key, in every one of six seeds. Reward pollution shows the two-sided damage. Both trained arms are \emph{less} diverse than the untrained base ($1.064\!\to\!1.032$ and $0.970$), the ordinary diversity collapse of RL fine-tuning; the reference key collapses it slightly more, which is what a sharper reward should do and is not evidence about key quality. The registered sub-predictions on omission suppression and diversity narrowing are not confirmed: at the $K{=}32$ sampling budget of App.~\ref{app:evalprec} the authored arm is in fact more diverse than the reference arm, the opposite of the prediction, and the diversity column is budget-dependent (it counts distinct valid answers among $K$ samples) so it is not comparable across $K$.}
\label{tab:propagation}
\end{table}

\begin{figure}[tbp]
\centering
\includegraphics[width=0.86\textwidth]{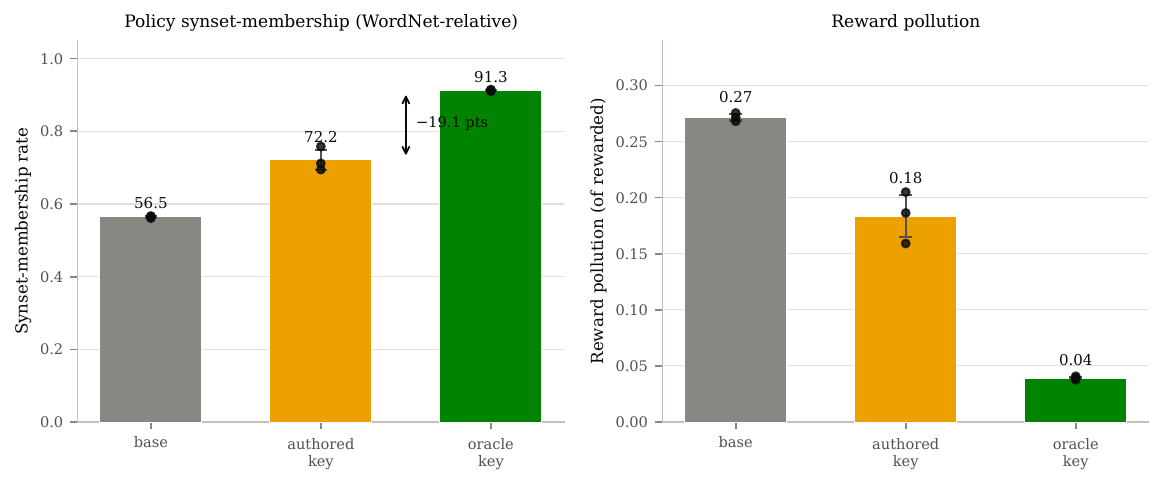}
\caption{E6 propagation (mean$\pm$std over 6 seeds; dots are per-seed values). \emph{Left:} synset-membership accuracy (WordNet-relative). Wiring the model-authored key into an otherwise-identical loop costs 18.5 points relative to the reference key. \emph{Right:} reward pollution (fraction of \emph{rewarded} rollouts the oracle rejects). The oracle key drives the trained policy to $0.04$, the authored key leaves it at $0.18$: the authored reward not only starves correct behavior (recall 0.507) but teaches the policy to emit oracle-invalid answers it happens to accept.}
\label{fig:propagation}
\end{figure}

\paragraph{Result: a 19-point WordNet-relative tax, unobservable in-loop.}
Table~\ref{tab:propagation} and Figure~\ref{fig:propagation}: with everything held fixed except which key pays the reward, the reference-key arm reaches $91.3\%$ synset-membership accuracy while the authored-key arm stops at $72.2\%$ (mean over 3 seeds; the effect is robust, the best authored seed, 0.758, lies below the worst oracle seed, 0.910; per-seed values in Appendix~\ref{app:e6seeds}).
The mechanism is the one the measured key predicts, and it is two-sided. The authored key omits $1-\rho=0.49$ of the synset lemma types, but the policy concentrates its correct rollouts on common in-key words, so the \emph{policy-experienced} starvation is milder: about $12\%$ of its correct rollouts go unrewarded (Table~\ref{tab:propagation}, omitted-fraction-of-valid). Meanwhile $25.8\%$ of the rollouts that are rewarded are oracle-invalid (pollution), so the damage is pollution-dominated and the authored arm trains on a signal that is both false-negative- and false-positive-noisy, the regime whose convergence cost RLVR theory anticipates \citep{cai2025noisyrewards}.
The pollution side leaves a fingerprint in the trained policies: training on the oracle key drives the policy's own reward pollution down to $0.041$, whereas the authored key halves the base rate only to $0.179$ (base $0.264$): it fails to purge the false-positive channel that the oracle key eliminates, leaving the trained policy nearly $5\times$ more polluted than under the oracle key.
Two honest notes.
\emph{First}, the sharper registered predictions were not confirmed at this budget: one LoRA epoch does not actively drive omitted-but-correct answers toward zero (they stay flat or rise) or collapse diversity; whether longer horizons produce that distributional damage remains open.
\emph{Second}, the 18.5-point gap is precisely the quantity the authored arm cannot observe from inside: by its own key's accounting it improved substantially, and nothing in its loop distinguishes ``the policy is wrong'' from ``the key is missing the answer.''
The cost of authoring is thus paid in capability and booked nowhere.

\paragraph{The tax persists at every scale, but does not grow without bound.}
The 18.5-point result is a single policy (Qwen2.5-3B); a natural hope is that a stronger base model escapes a weak key's flaws.
It does not.
We replicate the identical two-arm protocol (3 seeds each; same data, steps, and hyperparameters; same authored and oracle keys) across \emph{scale} (Qwen2.5-7B and -14B) and across \emph{family} (Phi-3.5-mini, a different lab's model) (Table~\ref{tab:propscale}, Figure~\ref{fig:propscale}).
The tax does not close at any scale tested: the reference-key arm reaches $0.91$--$0.97$ synset-membership accuracy while the authored-key arm never exceeds $0.72$, leaving a $19$--$32$ point gap across a $24{\times}$ parameter range, and $-28.2$ points for the different family. (These are lexical runs, so they inherit the reference-relative caveat above; the complete-truth estimate is in \S\ref{sec:propagation-numeric}.)
But it is not monotonic in scale: it rises from 3B to a peak at 14B ($-32.0$) then falls to $-26.6$ at 72B, whose much stronger base policy ($0.740$ vs.\ ${\approx}0.56$ for 3B--14B) lifts both arms (oracle $0.966$) yet still forfeits 27 points.
Reward pollution tracks the same shape: authored-key pollution $0.18/0.24/0.34/0.27$ at 3B/7B/14B/72B, peaking at 14B, against oracle pollution ${\le}0.04$ throughout, so a more capable policy optimizes harder onto what the key pays for (pollution included) up to mid scale, while the largest model's stronger prior pulls it partway back.
(An OLS fit of the tax on $\log_{10}$ parameters is positive but weak, slope $5.5$ pt/decade, $R^2{=}0.29$, but with only four scale points we report it as \emph{exploratory} and draw no scaling law from it.) the robustly supported claim is \emph{persistence, not monotone growth}: there is no scale we test at which training on a model-authored key matches training on the oracle.
Capability raises absolute accuracy but does not repair the authored-key tax.

\begin{table}[tbp]
\centering\small
\setlength{\tabcolsep}{7pt}
\begin{tabular}{llcccc}
\toprule
Policy & Axis & Base & Authored key & Oracle key & Tax (pts)\\
\midrule
Qwen2.5-3B   & anchor & 0.567 & 0.726 {\scriptsize$\pm$.026} & 0.911 {\scriptsize$\pm$.003} & $-18.5$\\
Qwen2.5-7B   & scale  & 0.562 & 0.695 {\scriptsize$\pm$.018} & 0.929 {\scriptsize$\pm$.003} & $-23.3$\\
Qwen2.5-14B  & scale  & 0.566 & 0.612 {\scriptsize$\pm$.007} & 0.932 {\scriptsize$\pm$.012} & $\mathbf{-32.0}$\\
Qwen2.5-72B  & scale  & 0.740 & 0.701 {\scriptsize$\pm$.020} & \best{0.966} {\scriptsize$\pm$.001} & $-26.6$\\
Phi-3.5-mini & family & 0.451 & 0.637 {\scriptsize$\pm$.038} & 0.919 {\scriptsize$\pm$.004} & $-28.2$\\
\bottomrule
\end{tabular}
\caption{E6 across scale and family (lexical arm: synset-membership accuracy, WordNet-relative; mean$\pm$std, 3 seeds; base is the untrained policy). Each row is an independent pair of GRPO runs differing only in the reward key. The reference arm reaches $0.91$--$0.97$ at every scale while the authored arm never exceeds $0.72$, a $19$--$32$ point gap over a $24\times$ parameter range and a different family. The tax is not monotonic: it peaks near 14B and is smaller at 72B, whose stronger base lifts both arms.}
\label{tab:propscale}
\end{table}

\begin{figure}[tbp]
\centering
\includegraphics[width=0.82\textwidth]{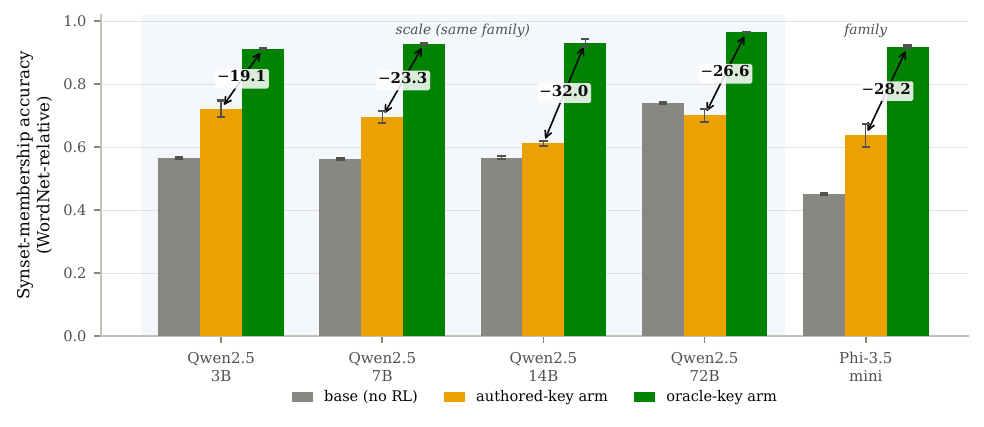}
\caption{The authoring tax persists across scale but is not monotonic. Base, authored-key, and reference-key synset-membership accuracy (WordNet-relative, $\pm$std over 3 seeds) for five policies spanning scale and family. The reference arm reaches $0.91$--$0.97$ at every scale while the authored arm never exceeds $0.72$, so the gap stays in a $19$--$32$ point band. It peaks at 14B rather than growing without bound.}
\label{fig:propscale}
\end{figure}

\subsection{Complete-truth replication: the clean causal estimate}
\label{sec:propagation-numeric}
The lexical run above is reference-relative. We therefore repeat the identical two-arm protocol on a \emph{numeric} construction, where the acceptable set is computed by a mechanical predicate, so the oracle is \emph{exact}: ``true accuracy'' is literally true, nothing can be missing, and the reference arm's criterion is ground truth rather than a proxy. Because finiteness is essential here we state the task fully (full specification in App.~\ref{app:numspec}): the candidate universe is the bounded set $V=\{1,\dots,100\}$ and every predicate is \emph{range-scoped}, e.g.\ ``between 20 and 80 that is divisible by 3 or by 5'', ``between 1 and 100 whose two digits differ by exactly 2'', so each acceptable set is finite (mean $|S^{*}|=10.3$, range $4$--$29$). Note these are compound, range-bounded predicates, not the four unbounded predicates used for the numeric scissors of \S\ref{sec:construction}. The policy is asked for one number satisfying the predicate; the reward is set membership in the arm's key. Train and evaluation predicates are \textbf{disjoint} (100 train, 60 held-out, zero overlap), so nothing is memorisable across the split. Measured on the held-out predicates the 14B-authored key has recall $0.926$ and precision $0.987$, far more accurate than its lexical counterpart ($0.507$), which is why the resulting tax is much smaller. We repeat the two-arm protocol on the \emph{numeric} construction, the same open-set generative task (``give a whole number that $P$'') but with mechanically decidable predicates, so the oracle is complete truth and the 14B-authored key is scored against it exactly (100 predicates, 6 seeds). Here authoring is far more accurate than on synonyms, and the downstream tax shrinks in lockstep:
\begin{center}\small
\begin{tabular}{lccc}
\toprule
Downstream task family & Authoring recall & Authoring pollution & RLVR tax\\
\midrule
Lexical (WordNet synonyms)   & 0.51 & 0.26         & $-18.5$\\
Numeric (arithmetic predicates) & 0.93 & ${\approx}0.04$ & $-1.9$\\
\bottomrule
\end{tabular}
\end{center}
\noindent The tax appears in both families (all six numeric seeds are positive, $+3.0/+1.6/+0.5/+3.2/+0.1/+3.1$; base $0.888\!\to\!$ authored $0.953\!\to\!$ oracle $0.972$) but is far larger where authoring is worse (lexical recall $0.51\to-18.5$; numeric recall $0.93\to-1.9$). Authoring quality is thus \emph{a} determinant of the tax, not the only one. The scale sweep of Table~\ref{tab:propscale} separates the two cleanly, and the design point is worth stating explicitly because it is easy to misread: \textbf{every arm in that sweep trains against one and the same authored key}, a single file written once by Qwen2.5-14B, so the key is a fixed input and the \emph{policy} is the only thing that varies across 3B, 7B, 14B, 72B and Phi. The tax nevertheless moves across that sweep and peaks at 14B, which it could not do if authored-key quality were the sole driver. The two families also differ in base accuracy ($0.90$ numeric vs.\ $0.57$ lexical), a second non-key determinant. We therefore read this as directional evidence, the downstream harm concentrates where authoring is hardest (fuzzy semantic sets; code suites that admit only $19$--$44\%$ of correct solutions, \S\ref{sec:code}--\ref{sec:frontier}) and is mild where the acceptable set is mechanically enumerable, rather than a precise scaling law. It still unifies the paper's two halves: the scissors (\S\ref{sec:scissors}--\ref{sec:code}) measures the authoring error, and the RLVR tax is what that error costs once the key becomes a reward.

\paragraph{What three seeds can and cannot support.}
Each arm is a set of paired runs (same seed, same data, same steps; only the reward key differs), so we report the \emph{paired} per-seed differences rather than arm means alone. The two anchor arms were extended to six seeds and re-scored with the precise evaluation of App.~\ref{app:evalprec}; the four scale-sweep arms remain at three seeds and their original scoring. Lexical 3B $\{21.8,14.7,20.3,15.7,20.4,18.1\}$, complete-truth numeric $\{3.0,1.6,0.5,3.2,0.1,3.1\}$; 7B $\{20.8,25.1,24.1\}$, 14B $\{31.5,30.5,33.8\}$, 72B $\{24.1,26.8,28.8\}$, Phi $\{31.1,23.4,30.0\}$ points.
\emph{First, the two anchor arms now clear conventional significance and the four scale arms cannot.} Both anchors have all six differences positive, giving an exact two-sided sign-permutation $p=2\cdot2^{-6}=0.031$, with paired $d_z=6.5$ (lexical, $18.51\pm2.83$ points) and $1.4$ (numeric, $1.91\pm1.40$ points). The four scale-sweep arms are still $n{=}3$, where the two-sided minimum attainable $p$ is $0.25$, so no $p<0.05$ is available for them by construction and we claim none; their evidence remains the direction, the non-overlap of seed ranges, and the effect size.
\emph{Second, precision mattered more than seed count here.} At the original $K{=}8$ sampling budget the evaluation's own binomial standard error was near $1.0$ point per arm, comparable to the entire numeric effect, and under it one numeric seed appeared to \emph{reverse} ($-0.8$). Re-scoring the identical checkpoints at $K{=}32$ with three seeded repeats reduces evaluation noise to $0.18$ points and that seed becomes $+0.05$: the reversal was measurement noise, not training variation. We report this because it also revises our own earlier estimates downward ($2.29\!\to\!1.91$ points, $d_z$ $2.4\!\to\!1.4$): a large $d_z$ at $n{=}3$ under a noisy instrument is not evidence of a large effect, and we had said so before it cost us.
\emph{Third}, we still deliberately do not rescue the scale arms by aggregating them. All paired runs favour the reference key, but the arms are not independent trials: the five lexical arms share the same dataset and the same authored and reference keys, four of them are the same model family differing only in scale, and the numeric arm is a different task family altogether. A sign test over them would presuppose an inferential unit we have not established, so we report their consistent direction as \emph{descriptive replication}

\section{Field Evidence at Production Scale}
\label{sec:field}

The constructions above are controlled; the phenomenon they isolate was first met in the wild. We operated a production K--12 assessment deployment of 43{,}227 scored items across 755 curriculum standards and ten generator LLMs, in which every fill-in-the-blank item requires an exhaustive authored answer key and a fixed commercial expert evaluator graded the results.
What follows is a \textbf{qualitative field signature}, not a replication, and the confounds come first. The format comparison is observational; we stratify on the two covariates the corpus records and the deficit survives both (App.~\ref{app:field}): holding difficulty fixed, fill-in trails the closed formats by $2.28/1.29/3.25$ points, pooling to $-2.27$ (Mantel--Haenszel-style, $95\%$ CI $[-2.93,-1.61]$, $z=-6.8$); holding grade fixed it is lower in $12/12$ grades ($p=2.4\times10^{-4}$). Subject, generator model, answer-space size and prompt length cannot be matched, so residual confounding remains. The error taxonomy is evaluator-diagnosed by the same grader we show to be $25\%$ self-consistent on open-format failures, so the $10{:}1$ ratio is directional, not a measured rate. This is the weakest link between the paper and a real deployment, and we mark it as such: a categorisation produced by an instrument that reproduces one fail verdict in four cannot be read as a rate, and no adjudication independent of that grader is reported here. Settling it needs only a small blinded re-check, on the order of $50$--$100$ sampled failures, re-categorised as omission or over-inclusion by raters who never see the grader's own label, and the blinded two-rater instrument for exactly this is built and released with the artifact, so the check is available to anyone who disputes the ratio. Read as corroboration of direction only:

\textbf{(1) Key errors are omission-dominated.}
Among all 2{,}586 open-format failures, evaluator-diagnosed answer-key errors split 316 omissions to 33 over-inclusions ($\approx$10:1), consistent across all ten generators.
\textbf{(2) The review stage predicts the direction.}
The pipeline ran an adversarial self-review over each drafted key, the operation Table~\ref{tab:silent} characterizes, and it behaved as predicted: over-inclusion counts fall while omission counts rise round-over-round, error mass sliding along a precision--recall frontier rather than shrinking.
\textbf{(3) Judging destabilizes at the boundary.}
The expert evaluator, re-run on identical items, re-affirms its fail verdicts on open-format items only 25\% of the time (41--42\% for closed formats), against 96--99\% self-consistency on passes: at the specification boundary, even the strongest available judge cannot reproduce its own decisions.
Table~\ref{tab:field} gives the numbers with Wilson intervals. The fill-in format, the only one requiring an authored exhaustive key, is both the lowest-passing and the least self-consistent on its failures ($0.25$ against $0.41$--$0.42$ for closed formats whose acceptable set is enumerated by construction): instability concentrates where a set must be authored.

\begin{table}[htbp]
\centering\small
\setlength{\tabcolsep}{6pt}
\begin{tabular}{lccc}
\toprule
Item format & Pass rate [95\% Wilson] & Fail self-consistency & Pass self-consistency\\
\midrule
fill-in (\emph{authored key}) & $96.48$ [95.84, 97.02] & \best{0.25} & 0.956\\
multiple-choice (enumerated) & $99.13$ [98.70, 99.42] & 0.409 & 0.976\\
multi-select (enumerated) & $98.61$ [98.07, 99.01] & 0.417 & 0.989\\
\bottomrule
\end{tabular}
\caption{Production corpus ($n{=}8{,}873$ re-scored items; 43{,}227 total). The measurement this table exists for is the \emph{self-consistency} pair: the fill-in format, the only one whose correctness requires an \emph{authored} exhaustive answer key, has the least reproducible fail verdicts, the evaluator re-affirming only $25\%$ of its open-format rejections against ${\geq}96\%$ of its passes, while the closed formats, whose acceptable set is enumerated by the item, are far more stable. Judge instability tracks the need to author a set. The pass-rate column is a descriptive statistic of the shared corpus and is also reported by the companion manuscript; the self-consistency columns are measured here and appear only here.}
\label{tab:field}
\end{table}

The field and the constructions agree: open-set boundaries are where enumeration fails, where review is least able to certify, and where judging itself becomes unstable.

\section{Discussion}
\label{sec:discussion}

\subsection{Related work}
\label{sec:related}
\paragraph{Generation--discrimination and verification asymmetry.}
A large literature contrasts a model's ability to \emph{generate} a solution with its ability to \emph{verify} one, and the recurring finding runs the reassuring way: verification is typically easier than generation \citep{west2024paradox,jiang2024selfincorrect,song2024mindgap}, which underwrites generate-then-verify and best-of-$n$ pipelines and process-reward supervision \citep{lightman2024verify}.
All of it compares generating a solution with verifying a solution against a given standard.
We measure one level up, where the object being generated is the standard itself.
This is not a reversal of that asymmetry but its meta-level extension: executing a specification is judgment given an object, while authoring one is generation from scratch, so the familiar recall-versus-recognition gap reappears \emph{on the verifier}.
The load-bearing consequence is that \emph{the verifier advantage does not transfer to verifier construction}: verification is cheap only after the specification exists, and the weakness of construction is then baked into the artifact that defines correctness downstream.

\paragraph{Model-generated specifications, rubrics, and verifiers.}
The closest work observes the same construction-versus-application gap in specific settings. \citet{zhou2026rubricbench} show that model-generated rubrics lag well behind human-authored ones and that the deficit lies in rubric \emph{formation}, not application; test-generation studies find LLM-authored suites too weak or misaligned to serve as reliable verifiers \citep{ma2025rethinking,ficek2025scoring,oracles2024actual,liu2023evalplus}. A longer-running line in open-domain QA evaluation establishes that \emph{reference} answer sets are systematically incomplete, so exact-match against them undercounts correct answers \citep{kamalloo2023evaluating,bulian2022tomayto}; we show the same incompleteness arises when the \emph{model itself} authors the acceptable set, pair it against that model's own execution ability, and trace it into an RLVR reward.
We do not claim to discover this bottleneck; we \emph{explain and generalize} it.
Relative to this line our increments are: (i) \emph{complete-truth} constructions (finite string/arithmetic predicate sets) that remove the human-rubric incompleteness confound those benchmarks inherit; (ii) a representation-agnostic \emph{measurement} that scores enumerated sets, test suites, and reward keys on one footing (\S\ref{sec:formal}), a unification of protocol rather than a new formalism; (iii) \emph{silent omission} as the identified dominant failure \emph{mechanism} (\S\ref{sec:silent}), not merely an error statistic; and (iv) a controlled \emph{execution} of a downstream cost prior work only projected, measuring it with the mechanism (starvation$+$pollution) isolated against an oracle, what an authored verifier does inside an RLVR loop (\S\ref{sec:propagation}). A parallel line has models author reward signals or verifiers directly: self-rewarding training \citep{yuan2024selfreward}, generative verifiers \citep{zhang2024genrm}, and model critics of code \citep{mcaleese2024critics}, and finds they can help when trained or grounded on labels; we study the harder zero-/one-shot regime where the authored specification has no such grounding. Test-generation benchmarks further report LLMs detect planted bugs well \citep{mundler2024swtbench}, a complementary axis to our finding that authored suites over-reject correct alternatives.

\paragraph{Test-oracle generation, and a contrary result we must address.}
The software-engineering literature on LLM-authored oracles bounds our claims in three ways, and we state each rather than leave it implicit. \emph{First, stronger authoring pipelines exist.} Fine-tuning raises oracle quality substantially \citep{hossain2025togll}, filtering candidate tests on token-probability/semantic-entropy signals improves validity \citep{taherkhani2024valtest}, and multi-agent consensus improves oracle correctness over the fine-tuned state of the art \citep{xu2026candor}. Our falsified repair arm (\S\ref{sec:repair}) is therefore a claim about \emph{same-model, one-shot sample-then-verify}, not a claim that authored verifiers cannot be improved, they demonstrably can be, by training, external signals, or ensembling, all of which add information our setting withholds. \emph{Second, a result points the other way.} \citet{konstantinou2024oracles} report that LLMs are better at \emph{generating} oracles than at \emph{classifying} which oracles are correct, the opposite ordering to our scissors. The settings differ in what is judged: their classification task asks a model to adjudicate externally supplied oracle \emph{code}, whereas our execution arm asks whether a presented \emph{candidate solution} is correct; and their generation target is a single assertion rather than an exhaustive acceptable set. We read the two results as compatible under our sharpened construct (\S\ref{sec:intensional}): writing one specification-like artifact is easy, and can beat meta-judging that artifact; exhaustively enumerating an acceptable set is what is hard. \emph{Third}, that same paper's primary finding, that generated oracles tend to encode the \emph{actual} rather than the \emph{expected} behaviour, independently corroborates our modal code failure, over-specification by invented requirement (\S\ref{sec:code}, App.~\ref{app:examples}).

\paragraph{The generation--verification gap, and why our sign is the interesting one.}
That verifying is easier than generating is the working premise of RLVR, best-of-$n$, and test-time scaling, and a literature exists on exploiting and shrinking that gap \citep{saadfalcon2025weaver,zhang2024genrm,lightman2024verify}. Our scissors is an instance of the same asymmetry, but the sign of its consequence is reversed, and that reversal is the contribution. In the standard framing the asymmetry is a \emph{resource}: because judging is cheap and reliable, one can filter, rerank, or reward with it. We study the case where the artifact a pipeline needs \emph{defines the judged set itself}, an answer key, an enumerated rubric, a test suite, so the cheap capability cannot substitute for the expensive one, and the asymmetry becomes a \emph{liability} that is inherited by everything downstream of the artifact. Two further specializations distinguish our claim from the generic gap: it is \emph{set-valued} (the failure is coverage, not correctness of a single answer), and its errors are \emph{directional} (omission-dominated, hence witness-poor and review-resistant, Observation~\ref{prop:omission}).

\paragraph{Enumeration as a capability.}
That LLMs enumerate poorly is independently established: they do not spontaneously count when asked to enumerate, and free-form enumeration lacks any stopping criterion for completeness \citep{hou2025seqenum,vinyals2016order,welleck2020consistency}. We take that as given rather than as our finding. What we add is (i) the paired comparison against the same model's pointwise judging and against its own intensional specification of the same set, which localizes the deficit, and (ii) the consequence when the enumerated object is a correctness criterion.

\paragraph{Reward hacking, Goodhart, and RLVR.}
Imperfect rewards are a foundational safety concern \citep{amodei2016concrete}, formalized as reward hacking \citep{skalse2022defining}, mapped empirically for misspecified rewards \citep{pan2022effects}, and quantified as over-optimization scaling laws \citep{gao2023scaling}.
Reinforcement learning with verifiable rewards \citep{lambert2024tulu3,guo2025deepseekr1}, typically optimized with GRPO-style objectives \citep{shao2024deepseekmath,yu2025dapo}, makes an executable check the reward, so its defects propagate into the policy.
The dominant concern is that checks are too \emph{weak}: hacking exploits verifiers that accept wrong behavior \citep{helff2026gaming}, fuzzing surfaces false \emph{positives} before training \citep{ray2026fuzzing}, and RLVR theory treats the false-negative rate as a given channel parameter \citep{cai2025noisyrewards,huang2025pitfalls}.
Our result is the complementary direction: machine-authored verifiers are first of all too \emph{strict}: omission-dominated, false-\emph{negative}-heavy, failing closed, so the danger is not that the policy games them but that they \emph{starve} correct behavior of reward while simultaneously \emph{polluting} it (\S\ref{sec:propagation}, Observation~\ref{prop:reward}). We supply the missing upstream cause: whether a model can author the verifier at all.

\paragraph{{LLM}-as-judge and the limits of self-checking.}
Model judges are now standard evaluators \citep{zheng2023judging,gu2024judgesurvey}, and the natural remedy for a weak authored specification is to have the model check or refine its own output: self-consistency \citep{wang2023selfconsistency}, self-verification \citep{weng2023selfverification}, and self-refinement \citep{madaan2023selfrefine}.
A growing counter-literature shows these do not reliably work: models often cannot self-correct reasoning without external signal \citep{huang2024selfcorrect} and are poor verifiers of their own reasoning and plans \citep{stechly2024selfverification}.
Our repair arm (\S\ref{sec:repair}) is the specialization of this finding to \emph{verifier authoring}, with a mechanism (Eq.~\ref{eq:repairbound}): self-filtering fails because the union under-covers $S^{*}$ and the membership judge collapses on self-generated near-boundary candidates; Observation~\ref{prop:omission} explains why review is \emph{directional} rather than merely weak.

\paragraph{Test, oracle, and specification synthesis.}
Generating tests to drive or check code is well studied, from execution-agreement filtering \citep{chen2023codet} to LLM unit-test generation \citep{schafer2024testgen}, on benchmarks such as HumanEval \citep{chen2021codex} and MBPP \citep{austin2021mbpp}, with hardened oracles exposing that generated and benchmark tests are too weak to reject wrong code \citep{liu2023evalplus} and that generated oracles encode the author's misreadings \citep{oracles2024actual}.
The formal-methods analogue, synthesizing specifications and proofs \citep{wu2022autoformalization,first2023baldur}, likewise reports that specification quality, not proof search, is the binding constraint.
We use these hardened oracles as \emph{ground truth} and measure the authoring direction against them, paired with the same model's execution ability, the contrast that turns ``generated tests are noisy'' into a measured authoring--execution scissors with a canonical-conditioned decomposition (\S\ref{sec:code}).

\paragraph{Mechanisms: enumeration, calibration, and structured decoding.}
Set-construction deficits of autoregressive decoders are documented: order sensitivity, non-terminating or premature stopping, enumeration failures are documented \citep{vinyals2016order,welleck2020consistency,hou2025seqenum}; token distributions misestimate human cloze acceptability \citep{jacobs2024cloze}; and calibration degrades exactly at ambiguity boundaries \citep{kadavath2022know,xiong2024confidence,tian2023calibration}.
Constrained/guided decoding \citep{willard2023guided} can enforce output \emph{format} but not the \emph{completeness} of an enumerated set, which is the failure we isolate.
Finally, human label-variation work \citep{nie2020chaosnli} measures disagreement over given labels; our boundary is constructed, so disagreement with it is error, not opinion, which is what lets us attribute the scissors to the model rather than to annotation noise.

\subsection{Design implications}
\label{sec:implications}

\paragraph{Elicit the rule, not the roster.}
Where the acceptable set has a compact characterization, ask the model for the predicate and let machinery enumerate (\S\ref{sec:intensional}). Offloading execution to an interpreter is program-aided prompting \citep{gao2023pal,chen2023pot}; what our measurement adds is that for acceptable sets this is not a marginal gain but the whole gap. The boundary is sharp: the escape is unavailable where no compact rule exists, and a finite test suite only partially escapes, since it is a program whose checked cases are still enumerated by hand.

\paragraph{Never trust a one-shot authored specification as ground truth.}
An authored key, suite, or rubric is not a noisy-but-unbiased artifact; it is directionally wrong in a way its own review cannot certify. Treat authored specifications as proposals requiring constructed ground truth, execution-based adjudication, or explicit recall auditing, rather than more review passes, which \S\ref{sec:silent} shows shed the wrong error direction.

\paragraph{Self-verification is not a substitute for ground truth.}
Having the model check its own specification, or assemble it by filtering its own samples, is what our repair arm falsifies (\S\ref{sec:repair}): the membership judge that looks strong on curated contrasts collapses on the self-generated near-boundary candidates that need adjudication. Same-model self-critique therefore buys the appearance of validation; a known-correct probe (\S\ref{sec:mitigation}) is the cheapest real gate.

\paragraph{Controls we did not run.}
For the record, five checks a skeptical reader is entitled to and we do not provide. \emph{(i) Constrained decoding.} We ran structured-emission and order controls (\S\ref{sec:emission}) but not grammar-constrained decoding, which would force the output to be a syntactically legal subset by construction; the residual formatting component we measure is therefore an upper bound on what constrained decoding would leave. \emph{(ii) Decoding sweeps.} Everything is greedy and one-shot, and we did not sweep temperature/top-$p$/beam. We did run the reasoning-enabled frontier condition on two constructions: it closes the algorithmic gap for both models but does not reliably repair test-suite authoring (shrinking Opus's gap without closing it, not helping GPT-5.1; \S\ref{sec:frontier}), so our claim is bounded to the no-reasoning regime only where a compact executable rule exists. The lexical construction under reasoning is untested. \emph{(iii) Suite strength beyond catch-wrong.} We report rejection of held-out oracle-wrong solutions, not mutation score or the number of independent faulty implementations caught, so ``the gated suite is usable'' is supported only in that weaker sense. \emph{(iv) Pool diversity.} We now measure this rather than only flagging it: a second pool from three non-Qwen families raises gated false rejection $0.010\!\to\!0.064$ and repaired false rejection to $0.075$--$0.161$ (\S\ref{sec:repair-trace}), so part of the low figure was stylistic proximity. \emph{Human-written} and adversarially diverse pools remain untested, and both would likely push it higher still. \emph{(v) Enriched review conditions.} Our review arm gives the reviewer the drafted key and nothing else. We did not run the informationally richer variants that would separate ``omission is witness-poor'' from ``our reviewer was under-equipped'': presenting the full candidate universe, supplying the correct cardinality, supplying a coverage checklist of semantic categories, or comparing independent-model, multi-agent, and retrieval-augmented review against the single-model loop we test. Observation~\ref{prop:omission} is stated to be robust to these (it assumes a perfect membership oracle and only bounds the query budget), but the \emph{empirical} $6$--$7\times$ asymmetry is a property of the review condition we ran, and richer conditions could shrink it.

\paragraph{Budget for omission when authored artifacts feed RLVR.}
False-negative-aware RLVR analyses \citep{cai2025noisyrewards} need a direction and a magnitude for verifier noise; for machine-authored verifiers our measurements supply both. Two neighbouring results belong alongside ours. \citet{wang2025tinyv} measure verifier false negatives directly and show they cost RL training real prompt efficiency, but theirs are dominated by formatting brittleness and remedied by a learned verifier, whereas ours arise from an acceptable set incomplete by omission and are isolated by a controlled A/B; the two are complementary. \citet{helff2026gaming} supply the dual of our finding, showing RLVR policies abandon rule induction and enumerate instance labels to satisfy verifiers that check only extensional correctness. The intensional/extensional distinction thus cuts both ways: policies are rewarded for enumerating when they should generalize, while verifier authors fail at enumerating and succeed at characterizing. Reward designers should assume correct-but-omitted behavior exists, measure it against held-out oracles where possible, and prefer continuous or rubric-style credit \citep{gunjal2025rubrics,chandak2025answermatching} over exhaustive enumerated keys.

\subsection{Limitations}
\label{sec:limitations}

\textbf{Construct validity.}
WordNet synsets approximate ``all acceptable words'' imperfectly, which is precisely why the paper does not rest on them: the central measurement is the algorithmic construction (Table~\ref{tab:algo}), where the acceptable set is computed mechanically over a presented list and incompleteness is impossible, and the executable construction (\S\ref{sec:code}), where correctness is decided by execution.
The three construction types' confounds are orthogonal (lexical incompleteness, oracle weakness, predicate triviality) and the scissors survives on all four constructions, so no single confound explains it; residual risks are that all three share some deeper artifact of set-enumeration prompting, which the repair arm (\S\ref{sec:repair}) partially addresses by varying the interface.
\textbf{Scale and family coverage.}
The scaling ladder is one family (Qwen2.5, a deliberate control so scale is the only variable) with cross-family replication across five further families and generations (Qwen3, Llama, gemma, Mistral, Phi-4); closed-weight frontier-API models and non-English constructions remain open.
\textbf{Predicate difficulty and extrapolation.}
The complete-truth predicates are deliberately simple (single-character and small-arithmetic tests) because their purpose is a mechanically decidable reference, not difficulty realism. This establishes the gap's existence and direction, not its magnitude at deployment complexity: with harder predicates both interfaces should degrade, and which degrades faster is untested, so the numbers here should not be extrapolated as effect sizes for compositional specifications.
\textbf{The escape assumes the rule is given.}
Predicate emission is near-perfect (\S\ref{sec:intensional}) under a condition worth stating plainly: the defining rule appears in the prompt, so the model restates rather than infers it. Real verification rarely supplies the rule; it must be derived from prose and examples, which is the operation the suite results show these models failing (\S\ref{sec:code}). We have not run the intermediate condition, inferring a predicate from examples alone, and it is the experiment that would locate the boundary between restatement and inference.
\textbf{One-shot authoring baseline.}
Greedy single-pass authoring is the deployed default we critique, not the strongest conceivable authoring prompt; the repair arm is the strongest same-model alternative we know, and it failed, but tool-augmented or multi-model authoring pipelines remain untested.
\textbf{Propagation budget.}
E6 is one LoRA epoch on single-word rollouts: good for attribution, narrow in task realism. Its two arms differ in what the effect size means. The \emph{lexical} arm's $19$--$32$ point gaps are WordNet-relative and additionally favour the reference arm through criterion alignment (\S\ref{sec:propagation}), so they bound rather than measure a semantic loss; the \emph{complete-truth} arm is exact but yields a much smaller $2.3$ points (\S\ref{sec:propagation-numeric}), and we treat that as the honest causal magnitude on an easy-authoring task. Effect size therefore depends on how badly the key is authored and on the task, and neither arm establishes a single universal tax; the unconfirmed suppression/diversity sub-predictions may yet emerge at longer horizons.
\textbf{Trace repair is measured on one author model.}
\S\ref{sec:repair-trace} now covers four author families and one benchmark. The recovery is established on HumanEval+ only; MBPP+ and non-Python targets are untested, and every number assumes a reference implementation exists (\S\ref{sec:mitigation}).
\textbf{The alignment control is one scale, one family.}
\S\ref{sec:base} rules out instruction tuning as the cause at 7B on Qwen2.5, using a plain-text harness that is weaker than the paper's main prompt for both arms. It shows the deficit exists before post-training; it does not establish that the base-vs-instruct equivalence holds at every scale, nor for post-training recipes more aggressive than Qwen's.
\textbf{Field evidence is judge-relative.}
Corpus facts in \S\ref{sec:field} come from a commercial evaluator whose own boundary instability we document; we use them as convergent field signature, never as ground truth for any headline claim.

\textbf{Relation to a companion manuscript.}
A companion paper by the same author studies the deployment of \S\ref{sec:field} as its primary object: the generation pipeline that produced it, its pass rate over the full curriculum, a cross-validation of the evaluator against judges from other vendors, and a distillation probe on the evaluator itself.
The two manuscripts therefore share one dataset, 43{,}227 scored items over 755 curriculum standards and ten generator LLMs, graded by one fixed commercial evaluator, and two format pass rates recur in both (fill-in 96.48\%, MCQ 99.13\%).
The analyses of that dataset do not overlap. The omission-dominant key-error split (316 to 33), the evaluator's 25\% self-consistency on open-format failures, and the stratified confound analysis of Appendix~\ref{app:field} are reported here and not there; the pipeline components, the production run's headline rate, the independent-judge validation and the distillation probe are reported there and not here.
Nothing else is common to the two: the four reference constructions, the detectability experiment, the RLVR propagation study and the gated-verifier mitigation, which is this paper's entire evidence base, appear only here, and neither paper's claims depend on the other's results.
We flag the shared dataset so reviewers can judge the overlap directly.

\subsection{Conclusion}
\label{sec:conclusion}

Language models are being promoted from examinees to examiners: they write the tests, keys, and rubrics that define correctness for other systems. Measured against constructed references, that promotion outruns one specific capability. The same models that judge membership well, and that can state the rule generating a set almost perfectly, cannot reliably materialise that set in one pass; enumeration lags judging across a $24\times$ scale range and does not catch up over it. Their dominant failure is omitting valid behavior, which is witness-poor: an over-inclusion is a written token a reviewer can challenge, whereas locating a missing member is itself the enumeration problem. Local subtractive review therefore shifts an artifact toward under-acceptance rather than repairing it, and a production corpus shows an observational signature consistent with that (10:1 omission-dominated, with the expert judge least self-consistent at the boundary).

The cheapest repair, letting the model verify its own proposals, fails where it is needed: four in five one-shot suites reject even the canonical solution, and sample-then-verify cannot beat the baseline it was designed to fix, because self-generated candidates sit where the model's own membership judgment is worst. Stronger pipelines do help, and prior work shows it; our negative result concerns the same-model, one-shot procedure. The downstream price is measured and bounded: an authored key costs $1.9$ points of genuine accuracy where the acceptable set is mechanically exact and $18.5$ points of WordNet-relative accuracy on the lexical task, where the reference is incomplete and favours the reference arm through criterion alignment.

What the paper does not claim is equally definite, because our own controls bound it. Two interfaces escape the deficit: emitting an executable predicate recovers the set almost perfectly, and test-time reasoning closes the algorithmic gap at the frontier. Both require that a compact executable rule exist, and reasoning does not reliably repair test-suite authoring. The deficit we characterise is therefore one-shot materialisation of an acceptance region with no compact rule to restate, extensional where the set is finite and intensional where it is not, which is nonetheless how answer keys, enumerated rubrics, and test suites are produced at scale today. The practical reading is constructive: ask a model for the rule, not the roster; where no rule exists, budget for omission and gate every authored verifier on behavior known to be correct, then repair rather than discard what the gate rejects, since the same execution that gates can also supply the expected values the model got wrong. A model-authored verifier is not ground truth, but the model's weakest interface wearing the authority of its strongest. Configurations, protocols, and statistical procedures are specified in App.~\ref{app:details}.


\clearpage
\appendix
\numberwithin{table}{section}
\numberwithin{figure}{section}

\section{Proofs}
\label{app:proofs}
Observations~\ref{prop:omission} and~\ref{prop:reward} are consequences of Definition~\ref{def:interfaces}; we give the arguments in full because their force is structural, not quantitative. (Observation~\ref{prop:setpenalty} is a measurement identity: its decomposition of $\Delta$ into batching, coverage, and stopping is established \emph{empirically} in Appendix~\ref{app:matched}, not proved here.)

\begin{proof}[Proof of Observation~\ref{prop:omission}]
Let $\hat S\subseteq V$ and let the reviewer query $\hat m(x)=\mathbb{1}[x\in S^{*}]$ only on candidates it names; write $\mathrm{OI}=\hat S\setminus S^{*}$, $\mathrm{OM}=S^{*}\setminus\hat S$.
\emph{(i) Over-inclusion.} Every $x\in\mathrm{OI}$ is named by $\hat S$, so querying $\hat m$ on each element of $\hat S$ flags exactly $\mathrm{OI}$ in $O(|\hat S|)$ queries.
\emph{(ii) Witness-blindness.} The reviewer's observations are a function only of $\hat S$ and of $\hat m$ on the finite set of candidates it names. Take two oracles $S_1^{*},S_2^{*}$ with $S_1^{*}\cap\hat S=S_2^{*}\cap\hat S$ (so $\hat m$ agrees on all of $\hat S$) and $S_1^{*}\setminus\hat S\neq S_2^{*}\setminus\hat S$. Because the query budget is smaller than $|V\setminus\hat S|$, at least one candidate in $V\setminus\hat S$ is never named; place the oracles' distinguishing element there. The reviewer then names no element of $(S_1^{*}\triangle S_2^{*})\setminus\hat S$, so its observations coincide under the two oracles, and it cannot distinguish them or certify which members are omitted. To break the tie it must name (emit) some $x\in S^{*}\setminus\hat S$, which by Definition~\ref{def:interfaces} is the authoring problem restricted to $\mathrm{OM}$. Thus \emph{repairing} an omission requires authoring. This is weaker than claiming omission cannot be \emph{detected}: an external cardinality or coverage certificate (``$|S^{*}|=10$''; ``category $C$ uncovered'') can reveal that \emph{some} omission exists, but such a certificate is unavailable when $S^{*}$ is the unknown being authored, and detection without a witness still cannot repair the set.
\end{proof}

\begin{figure}[htbp]
\centering
\begin{tikzpicture}[font=\footnotesize, >={Stealth[length=4pt,width=3pt]},
  bx/.style={draw, rounded corners=2pt, align=center, inner sep=3pt, fill=white}]
  \node[bx, draw=cexec, fill=cexec!7] (rev) at (0,0) {local reviewer\\queries $\hat m(x)$};
  \node[bx, draw=cink!55, text width=26mm] (S) at (3.9,0) {$\hat S$: named tokens\\(the only queryable set)};
  \draw[->, line width=0.7pt, draw=cgy] (rev)--(S);
  \node[bx, draw=cok, fill=cok!8, text width=40mm] (w1) at (9.2,1.15)
    {world $S_1^{*}=(\hat S\cap S^{*})\cup\{\textsf{a}\}$};
  \node[bx, draw=cauth, fill=cauth!10, text width=40mm] (w2) at (9.2,-1.15)
    {world $S_2^{*}=(\hat S\cap S^{*})\cup\{\textsf{b}\}$};
  \draw[->, densely dashed, draw=cgy!70] (S.east) to[out=35,in=180] (w1.west);
  \draw[->, densely dashed, draw=cgy!70] (S.east) to[out=-35,in=180] (w2.west);
  \node[cgy, font=\scriptsize, align=center] at (3.9,-1.5)
    {$\hat m$ agrees on all of $\hat S$;\\ the witnesses \textsf{a},\textsf{b} lie in the\\ un-queried region $V\setminus\hat S$};
  \node[red!70, font=\scriptsize] at (9.2,-2.4) {observations identical $\Rightarrow$ recall uncertifiable};
\end{tikzpicture}
\caption{The indistinguishability core of Observation~\ref{prop:omission}. A local reviewer can query the membership oracle only on tokens the artifact names. Two oracles that agree with $\hat S$ on the intersection but differ by a witness placed in the un-queried region (non-empty by the query budget) produce identical observations, so the reviewer cannot tell which members are omitted: certifying recall would require first authoring the missing witness.}
\label{fig:twoworlds}
\end{figure}
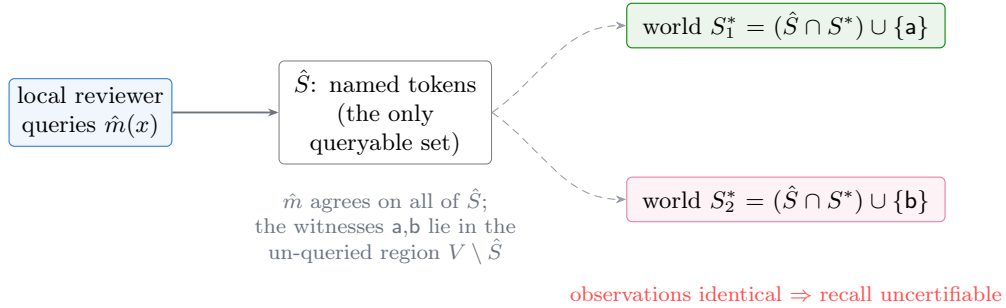

\noindent\textbf{Corollary (directional review; Figure~\ref{fig:twoworlds}).} A review loop that only \emph{subtracts} flagged errors (removes members of $\hat S$ it can challenge) strictly reduces $|\mathrm{OI}|$ and leaves $|\mathrm{OM}|$ unchanged in expectation up to its authoring ability; iterating it moves $\hat S$ monotonically toward higher precision and no-higher recall, i.e.\ toward under-acceptance. This is the mechanism behind the field trace of \S\ref{sec:field} (over-inclusions fall round-over-round while omissions persist).

\begin{proof}[Proof of Observation~\ref{prop:reward}]
Write the two rewards as indicator functions $R=\mathbb{1}[\,\cdot\in\hat S]$ and $R^{*}=\mathbb{1}[\,\cdot\in S^{*}]$ on $V$. For any $x$, $R(x)\neq R^{*}(x)$ iff $x\in\hat S\triangle S^{*}$, with two disjoint cases: (a) $x\in S^{*}\setminus\hat S$, where $R(x)=0<1=R^{*}(x)$ (a withheld reward on correct behavior, \emph{starvation}); (b) $x\in\hat S\setminus S^{*}$, where $R(x)=1>0=R^{*}(x)$ (a granted reward on incorrect behavior, \emph{pollution}). By Definition~\ref{def:interfaces}, $|S^{*}\setminus\hat S|=(1-\rho)|S^{*}|$ and $|\hat S\setminus S^{*}|=(1-\pi)|\hat S|$, giving the stated fractions. For a policy $p$ over $V$, the reward it collects is $\mathbb{E}_{p}[R]=p(\hat S)$ whereas its true accuracy is $p(S^{*})$; these differ by $p(\hat S\setminus S^{*})-p(S^{*}\setminus\hat S)$, so the gradient of expected reward is insensitive to moving mass between $S^{*}\setminus\hat S$ and its complement within the unrewarded region, and rewards moving mass into $\hat S\setminus S^{*}$. Hence no stationary point of $\mathbb{E}_p[R]$ need be a stationary point of $p(S^{*})$: $R$ and $R^{*}$ disagree exactly on the two error masses $\hat S\triangle S^{*}$, so the two optima need not coincide and the reward's disagreement with the oracle is supported precisely on those masses, which \S\ref{sec:propagation} measures ($1-\rho=0.493$ set-level starvation; trained-policy pollution $0.179$).
\end{proof}

\section{Prompts}
\label{app:prompts}
All prompts are reproduced verbatim; greedy decoding throughout except the sample-then-verify union (\S\ref{sec:repair}, $T{=}0.9$) and the RLVR rollouts (\S\ref{sec:propagation}, on-policy). Qwen3 models are run with thinking disabled so authoring/execution remain one-shot and comparable to the other families.

\paragraph{Intensional control (\S\ref{sec:intensional}), verbatim.}
\emph{Authoring (intensional):} ``Write a Python function \texttt{def accepts(w):} that returns True exactly when the lowercase word \texttt{w} \{pred\}, and False otherwise. Reply with ONLY a Python code block containing the function.'' The emitted function is then executed over the presented list in a fresh subprocess (8\,s timeout); a predicate that fails to parse or run scores $0$.
\paragraph{Emission-format controls (\S\ref{sec:emission}), verbatim.}
\emph{L (free list, the paper's format):} ``Here is a list of words: \{lst\}.\textbackslash nList every word from this list that \{pred\}. Reply with ONLY a comma-separated list of words from the given list.''
\emph{J (JSON checkbox):} ``Here is a list of words: \{lst\}.\textbackslash nFor EVERY word in the list, decide whether it \{pred\}. Reply with ONLY a JSON object mapping each word to true or false, including every word exactly once, in the order given.''
\emph{S} is L with the presented list order permuted by a fixed seed.

\paragraph{Lexical (WordNet).}
\emph{Authoring:} ``List every single English word that can mean `\{gloss\}' (as in: `\{example\}'). Reply with only a comma-separated list.''
\emph{Execution:} ``Can the word `\{cand\}' mean `\{gloss\}' (as in: `\{example\}')? Answer yes or no.''
\emph{Review (\S\ref{sec:silent}):} the proposed key is shown with instruction to critique and repair in a fixed \texttt{REMOVE:}/\texttt{ADD:} format.

\paragraph{Algorithmic (complete truth).}
\emph{Authoring:} ``Here is a list of words: \{list\}. List every word from this list that \{predicate\}. Reply with ONLY a comma-separated list of words from the given list.''
\emph{Execution:} ``Consider the word ``\{w\}''. Does it \{predicate\}? Answer with exactly one word: yes or no.'' Predicates: \emph{contains a double letter}; \emph{has 8 or more letters}; \emph{contains the letters `th'}; \emph{ends with the letter `s'}.

\paragraph{Code (HumanEval+/MBPP+).}
\emph{Execution (judge):} the docstring and a candidate implementation are shown with ``Is the candidate a correct implementation of the specification? Answer with exactly one word: yes or no.''
\emph{Authoring:} ``Write a thorough test suite for this function as a single function \texttt{def check(candidate):} containing assert statements $\dots$ Reply with ONLY a Python code block.''

\paragraph{RLVR (\S\ref{sec:propagation}).}
``Give ONE English word that means ``\{gloss\}'' (as in: ``\{example\}''). Reply with only the word, nothing else.'' Reward $=$ membership of the normalized first line in the arm's key (authored or oracle), computed by the deterministic plugin.

\section{Construction and training details}
\label{app:details}
The four construction pipelines are summarized in Figure~\ref{fig:pipeline}. All measurements ran single-GPU (bfloat16, PyTorch/Transformers; executable grading via EvalPlus; RLVR via ms-swift's GRPO). \textbf{Lexical:} WordNet synsets with 3--8 single-word lemmas and $\geq1$ usage example, target polysemy $\leq4$; $n{=}240$ items, deterministic seed 0; hard negatives from co-hyponyms, hypernym lemmas, antonyms.
\textbf{Algorithmic:} $n{=}240$ items, each a list of 15 WordNet lemmas (length 3--11, shuffled seed 0) plus one of four predicates cycled by item index; lists resampled to keep a $2$--$13$ accept split.
\textbf{Code:} HumanEval+ (164 problems) and MBPP+ (378 problems); an oracle-verified pool built from $K{=}8$, $T{=}0.8$ samples of Qwen2.5-7B/14B graded by the full EvalPlus suite.
\textbf{RLVR:} GRPO (ms-swift), Qwen2.5-3B policy, LoRA rank 8 ($\alpha{=}32$, dropout $0.05$, target modules \texttt{all-linear}); KL coefficient $\beta{=}0.04$; \texttt{num\_generations}=8, per-device batch 16, grad-accum 2, \texttt{max\_completion\_length}=16, lr $1{\times}10^{-5}$ with cosine schedule, rollouts at $T{=}0.9$, top-$p$ $1.0$, 1 epoch (510 steps), identical at every scale; 2{,}400 items (2{,}040 train / 360 eval). Seeds $\{1,2,3\}$ for the four scale-sweep arms, $\{1,\dots,6\}$ for the two anchor arms; evaluation $k{=}8$ per item at $T{=}1.0$ as originally run, superseded for the anchor arms by $K{=}32$ with three seeded repeats (App.~\ref{app:evalprec}); per-step logs and fully resolved configurations retained for all 27 runs.
\textbf{Determinism and sandboxing:} constructions are sampled at fixed seed 0; execution and one-shot authoring are greedy; the repair union uses $T{=}0.9$, $K{=}10$; every authored code suite runs in a fresh subprocess with an 8\,s timeout, and a suite accepts a solution only if the process prints the success sentinel, so crashes and timeouts count as rejections, a conservative choice that if anything understates authoring failure.
\textbf{Statistics:} all CIs are 2{,}000-resample item bootstraps, with additional \emph{problem-cluster} bootstraps for the code constructions (solutions are nested in problems); production pass rates use Wilson intervals; RLVR arms are paired per-seed differences; the two anchor arms were extended to $n{=}6$ (minimum attainable two-sided permutation $p{=}0.031$), the four scale-sweep arms remain at $n{=}3$, where $0.25$ is the two-sided minimum and no $p<0.05$ is attainable. The primary endpoint, fixed before the sweep, is the authoring--execution gap on the two complete-truth constructions with the preregistered predicates and the single authoring prompt of App.~\ref{app:prompts}; the prompt-rephrasing ablation, per-predicate breakdown, closed-frontier point estimates, OLS scale fit, and cross-family comparisons beyond Qwen2.5 are exploratory and reported without correction.
\textbf{Closed-model protocol:} frontier runs go through OpenRouter to \texttt{openai/gpt-5.1} and \texttt{anthropic/claude-opus-4.8} (accessed 2026-07; no finer immutable snapshot pins are exposed, a reproducibility limitation of any closed-model measurement); no system prompt, \texttt{temperature}$=0$, six retries with linear backoff, empty completions counted as failures. Reasoning-enabled conditions use OpenRouter's provider-normalised \texttt{reasoning} field rather than each vendor's native control, so ``matched'' means matched \emph{request}, not verified-identical internal budget; visible-output budgets were 512 (algorithmic authoring), 900 (test-suite authoring), and 64/512 for one-word execution judgments without/with reasoning; subsets (algorithmic $n{=}120$, code $n{=}80$) are seed-0 draws, with reasoning-off baselines re-run on the same subsets so off/on are paired. Gemini~3.x could not be run: its API rejects disabling reasoning, and with reasoning on it returned empty content for the one-word judgments.
Raw per-item vectors are retained alongside each result for recomputation without GPUs.

\begin{figure}[htbp]
\centering
\resizebox{\textwidth}{!}{%
\begin{tikzpicture}[font=\scriptsize, >={Stealth[length=3.5pt,width=2.6pt]},
  s/.style={draw=cgy, rounded corners=2pt, align=center, inner sep=3pt, fill=white, text width=19mm, minimum height=8mm},
  pr/.style={s, draw=cexec, fill=cexec!6, text width=35mm},
  o/.style={s, draw=cok, fill=cok!8, text width=25mm},
  a/.style={->, draw=cgy, line width=0.6pt}]
 \node[s]  (l0) at (0,3.3) {WordNet synsets};
 \node[pr] (l1) at (4.6,3.3) {filter: 3--8 lemmas, polysemy $\le4$, $+$ hard negatives};
 \node[o]  (l2) at (10.0,3.3) {240 \textbf{lexical} items};
 \draw[a](l0)--(l1); \draw[a](l1)--(l2);
 \node[s]  (g0) at (0,2.2) {WordNet lemmas};
 \node[pr] (g1) at (4.6,2.2) {15-word lists $\times$ 4 mechanical predicates, seed 0};
 \node[o]  (g2) at (10.0,2.2) {240 \textbf{algorithmic}/\textbf{numeric} items};
 \draw[a](g0)--(g1); \draw[a](g1)--(g2);
 \node[s]  (c0) at (0,1.0) {HumanEval+ / MBPP+};
 \node[pr] (c1) at (4.6,1.0) {$K{=}8$, $T{=}0.8$ Qwen-7B/14B $\to$ EvalPlus grade};
 \node[o]  (c2) at (10.0,1.0) {oracle pool 1560$\checkmark$ / 904$\times$};
 \draw[a](c0)--(c1); \draw[a](c1)--(c2);
 \node[s]  (r0) at (0,-0.2) {2{,}400 items (2040/360)};
 \node[pr] (r1) at (4.6,-0.2) {GRPO: Qwen2.5-3B LoRA, 510 steps, 3 seeds};
 \node[o, draw=cauth, fill=cauth!8] (r2) at (10.0,-0.2) {tax vs.\ oracle key};
 \draw[a](r0)--(r1); \draw[a](r1)--(r2);
 \draw[a, densely dashed, draw=cgy!60] (l2.east) -- (11.5,3.3) -- (11.5,-1.15) -- node[below, font=\tiny, text=cgy] {lexical items feed RLVR} (0,-1.15) -- (r0.south);
\end{tikzpicture}}
\caption{The four construction pipelines and the RLVR loop (App.~\ref{app:details}). Each construction produces items carrying a \emph{complete} oracle set (mechanical predicate, executable oracle, or WordNet synset) plus, for execution, hard negatives; the lexical items also feed the downstream GRPO experiment (\S\ref{sec:propagation}). All sampling seeds and filters are fixed so the pipeline is reproducible end-to-end.}
\label{fig:pipeline}
\end{figure}
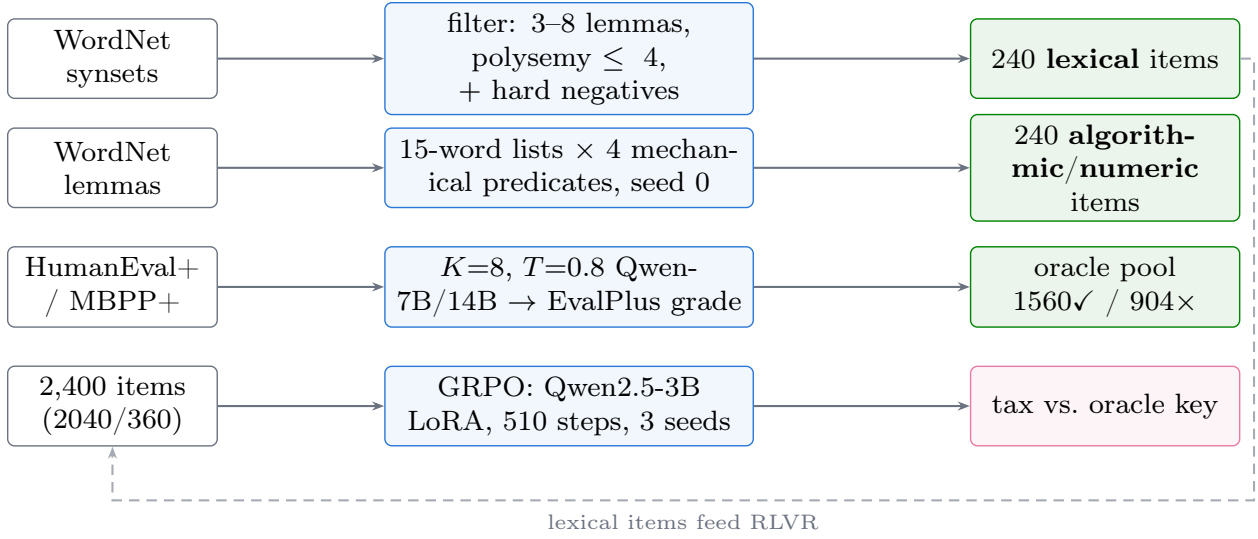

\section{Worked examples}
\label{app:examples}
Two real, unedited instances make the failure concrete.

\paragraph{Lexical (both errors at once).}
For the gloss \emph{``something that serves as a model or a basis for making copies''} (example: \emph{``this painting is a copy of the original''}), the oracle synset is $S^{*}=\{\textsf{archetype},\textsf{original},\textsf{pilot}\}$. The 14B model authored
\[
\hat S=\{\textsf{archetype},\textsf{copy},\textsf{master},\textsf{model},\textsf{original},\textsf{prototype},\textsf{specimen},\textsf{standard},\textsf{template},\dots\}.
\]
It \emph{omits} \textsf{pilot} (starvation: a correct member absent, and no review of $\hat S$ can surface it; Observation~\ref{prop:omission}) while \emph{over-including} \textsf{copy}, \textsf{model}, \textsf{prototype}, \textsf{specimen} (pollution: near-synonyms outside this sense). Recall and precision both suffer, and the two errors are exactly the two crescents of Figure~\ref{fig:geometry}.

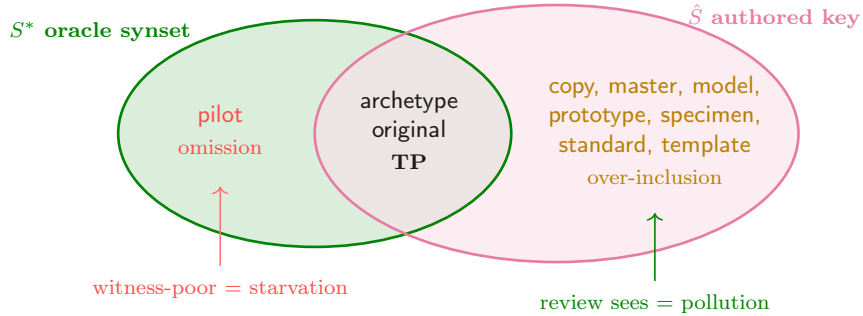
\begin{figure}[htbp]
\centering
\begin{tikzpicture}[font=\footnotesize]
  \fill[cok!14] (-1.5,0) ellipse (2.6 and 1.5);
  \fill[cauth!14] (1.7,0) ellipse (3.2 and 1.7);
  \begin{scope}\clip (-1.5,0) ellipse (2.6 and 1.5);\fill[cok!28!cauth!20] (1.7,0) ellipse (3.2 and 1.7);\end{scope}
  \draw[cok, line width=1pt] (-1.5,0) ellipse (2.6 and 1.5);
  \draw[cauth, line width=1pt] (1.7,0) ellipse (3.2 and 1.7);
  \node[cok, font=\scriptsize\bfseries, anchor=east] at (-3.0,1.35) {$S^{*}$ oracle synset};
  \node[cauth, font=\scriptsize\bfseries, anchor=west] at (3.3,1.55) {$\hat S$ authored key};
  \node[align=center, text=red!70] at (-2.75,0.05) {\textsf{pilot}\\[1pt]{\scriptsize omission}};
  \node[align=center, text=cink] at (-0.25,0.05) {\textsf{archetype}\\\textsf{original}\\[1pt]{\scriptsize\bfseries TP}};
  \node[align=center, text=cwarn!78!black] at (3.0,0.05) {\textsf{copy, master, model,}\\\textsf{prototype, specimen,}\\\textsf{standard, template}\\[1pt]{\scriptsize over-inclusion}};
  \node[red!70, font=\scriptsize] at (-2.75,-2.05) {witness-poor $=$ starvation};
  \node[cok, font=\scriptsize] at (3.0,-2.25) {review sees $=$ pollution};
  \draw[->, red!55, line width=0.6pt] (-2.75,-1.75)--(-2.75,-0.7);
  \draw[->, cok, line width=0.6pt] (3.0,-1.95)--(3.0,-1.05);
\end{tikzpicture}
\caption{The lexical worked example as the geometry of Figure~\ref{fig:geometry} on real words. The oracle synset $S^{*}$ (green) and the 14B-authored key $\hat S$ (pink) share only \textsf{archetype} and \textsf{original} (TP). \textsf{pilot} sits in the oracle-only crescent (omission: no query on $\hat S$ can surface it, Observation~\ref{prop:omission}), while seven near-synonyms of the wrong sense sit in the authored-only crescent (over-inclusion: each is a named token a reviewer or reward can act on). The single item drives both a recall and a precision loss.}
\label{fig:worked}
\end{figure}

\paragraph{Code (over-specification by invented requirement).}
Figure~\ref{fig:worked} shows the pattern end to end; for \texttt{HumanEval/1} (\texttt{separate\_paren\_groups}), whose spec says nothing about malformed input, the 14B model authored a suite that invents a \texttt{ValueError} requirement and rejects the problem's own canonical solution:
\begin{verbatim}
def check(candidate):
    assert candidate('() (( )) (( )( ))') == ['()', '(())', '(()())']
    assert candidate('()') == ['()']
    assert candidate('') == []
    try:                                 # <- requirement not in the spec
        candidate('(()')
        assert False, "Expected ValueError for unbalanced parentheses"
    except ValueError:
        pass
\end{verbatim}
The suite runs cleanly (no crash) but asserts behavior the specification never requires; the canonical solution does not raise, so it is rejected. This is the modal code failure: a syntactically healthy verifier encoding an expected behavior its author could not correctly compute.

\paragraph{Error taxonomy (HumanEval+, 14B suites).}
Classifying all 164 authored suites by how they treat the canonical solution (Table~\ref{tab:taxonomy}) shows the failure is semantic, not mechanical: 70\% run and reject on an assertion (over-specification), only 9\% crash, and 21\% accept. Among rejection \emph{events} from broken suites, 87.5\% are assertion failures and 11.3\% runtime errors.

\begin{table}[htbp]
\centering\small
\setlength{\tabcolsep}{8pt}
\begin{tabular}{lcc}
\toprule
Outcome on canonical solution & \# suites & fraction\\
\midrule
Accept (sane) & 34 & 0.207\\
Reject by assertion (over-specification) & 115 & 0.701\\
Reject by runtime error (harness bug) & 15 & 0.091\\
\bottomrule
\end{tabular}
\caption{Code authoring error taxonomy (HumanEval+, 14B). Seven in ten authored suites are wrong-by-assertion, they execute but encode incorrect expected behavior, while fewer than one in ten fail mechanically, confirming the failure is a boundary the author cannot compute, not a coding accident.}
\label{tab:taxonomy}
\end{table}

\section{Per-predicate algorithmic breakdown}
\label{app:predicates}
The algorithmic authoring difficulty is predicate-dependent but the ordering (authoring $<$ execution) holds within every predicate. Table~\ref{tab:predicates} gives authoring F1 by predicate: enumerating ``$\geq8$ letters'' is easiest (it is monotone in a single salient feature), ``ends in `s'\,'' is hardest (a sparse suffix condition), yet all scale sub-critically.

\begin{table}[htbp]
\centering\small
\setlength{\tabcolsep}{7pt}
\begin{tabular}{lcccc}
\toprule
Model & double letter & $\geq8$ letters & contains ``th'' & ends in ``s''\\
\midrule
Qwen2.5-3B  & 0.264 & 0.348 & 0.231 & 0.194\\
Qwen2.5-7B  & 0.289 & 0.608 & 0.444 & 0.279\\
Qwen2.5-14B & 0.356 & 0.755 & 0.515 & 0.273\\
Qwen2.5-72B & 0.362 & 0.730 & 0.524 & 0.316\\
\bottomrule
\end{tabular}
\caption{Authoring F1 by predicate (algorithmic construction). Difficulty varies by predicate structure, but authoring never approaches the corresponding execution F1 ($0.60$--$0.77$, Table~\ref{tab:algo}) for any predicate.}
\label{tab:predicates}
\end{table}

\section{Full results: all models $\times$ constructions}
\label{app:full}
Table~\ref{tab:master} consolidates every model$\times$construction cell behind the paper's figures, with 95\% bootstrap CIs on the gaps. The authoring$<$execution ordering holds in every cell that was run; ``---'' marks a cell not run for that model (algorithmic and code cover the four Qwen2.5 scales, both Qwen3 models, both Llama models, Mistral-7B, and Phi-4; lexical covers the Qwen2.5 ladder and gemma). Per-model numeric results are in Table~\ref{tab:numeric}; the two closed-frontier models are in Table~\ref{tab:closedfrontier}.

\begin{table}[htbp]
\centering\small
\setlength{\tabcolsep}{5pt}
\begin{tabular}{llccc}
\toprule
& & Algorithmic & Code (exec F1 / & Lexical\\
Family & Model & gap [95\% CI] & accept-correct) & gap [95\% CI]\\
\midrule
\multirow{4}{*}{Qwen2.5}
 & 3B  & $+0.340$ [.29,.39] & 0.763 / 0.331 & $+0.178$ [.13,.23]\\
 & 7B  & $+0.252$ [.22,.29] & 0.853 / 0.186 & $+0.451$ [.41,.49]\\
 & 14B & $+0.286$ [.25,.32] & 0.742 / 0.289 & $+0.388$ [.34,.43]\\
 & 72B & $+0.286$ [.25,.32] & 0.895 / 0.422 & $+0.596$ [.56,.63]\\
\midrule
\multirow{2}{*}{Qwen3}
 & 14B & $+0.298$ [.27,.32] & 0.800 / 0.324 & ---\\
 & 32B & $+0.309$ [.28,.34] & 0.783 / 0.262 & ---\\
\midrule
\multirow{2}{*}{Llama-3}
 & 3.2-3B & $+0.128$ [.09,.17] & 0.580 / 0.209 & ---\\
 & 3.1-8B & $+0.190$ [.16,.22] & 0.766 / 0.101 & ---\\
\midrule
gemma & 4-12B & --- & --- & $+0.552$ [.52,.59]\\
\midrule
Mistral & 7B-v0.3 & $+0.044$ [.01,.08] & 0.799 / 0.054 & ---\\
Phi & 4 (14B) & $+0.339$ [.31,.37] & 0.856 / 0.238 & ---\\
\bottomrule
\end{tabular}
\caption{Master results: authoring--execution gap (algorithmic and lexical) and code execution-F1/accept-correct, for all eleven models. Every run shows execution above authoring; the gap is present across four scales and five model families (Qwen, Llama, gemma, Mistral, Phi) spanning two Qwen generations (Qwen2.5/Qwen3).}
\label{tab:master}
\end{table}

\section{Two defects we found in our own instrumentation}
\label{app:bosfix}
Both were found by us during a late audit, not by review, and both are reported here with the
before/after numbers rather than silently corrected.

\textbf{(1) A double-BOS tokenization defect affecting three of ten open-weight models.}
Our measurement scripts build a prompt with \texttt{apply\_chat\_template(tokenize=False)} and then
encode it with the tokenizer's default \texttt{add\_special\_tokens=True}. For model families whose
chat template already emits a beginning-of-sequence token, that prepends a second one. We
established the affected set by token-level comparison rather than by assumption: with the fix, the
encoded ids are \emph{byte-identical} for Qwen2.5 (3B/7B/14B/72B), Qwen3 (14B/32B) and Phi-4,
whose tokenizers define no BOS, so those seven models provably cannot change and were not re-run.
Llama-3.2-3B, Llama-3.1-8B and Mistral-7B differ by exactly one token and were re-measured on both
constructions.

The paper's central result is unaffected, since the scaling ladder is Qwen2.5, and the direction
survives everywhere: the execution$>$authoring ordering holds for all three re-measured models,
and for both Llama models the gap \emph{widened} (algorithmic $+0.091\!\to\!+0.128$ and
$+0.148\!\to\!+0.190$), because the defect had been depressing the judging arm more than the
authoring arm. Mistral's algorithmic gap narrowed ($+0.079\!\to\!+0.044$) but stayed positive.
What did change materially is cross-family commentary: Llama-3.2-3B's code judging F1 rises
$0.354\!\to\!0.580$ and its authored suites' acceptance of correct solutions $0.078\!\to\!0.209$,
so its gated yield nearly triples ($0.073\!\to\!0.213$), its gated catch-wrong rises
$0.450\!\to\!0.770$ and its vacuous fraction falls $0.250\!\to\!0.086$. The claim that weak authors
produce vacuous survivors therefore rests on Mistral alone, not on Llama, and \S\ref{sec:mitigation}
now says so.

\textbf{(2) An evaluation too imprecise for the effect it measured.}
\label{app:evalprec}
The RLVR arms were originally scored by sampling $K{=}8$ completions per held-out item at
temperature $1.0$ with no fixed RNG seed. At an oracle-valid rate near $0.95$ over $60$ numeric
items that is $480$ samples, a binomial standard error of about $1.0$ point per arm and roughly
$1.4$ points on the paired difference, against a true effect near $2$ points. The measurement was
therefore of the same order as the quantity being measured, and under it one numeric seed appeared
to reverse sign ($-0.8$ points).

We re-scored the identical checkpoints, with no retraining, at $K{=}32$ with three explicitly seeded
repeats per adapter, which both shrinks the noise and lets it be measured rather than assumed. The
observed within-adapter spread across eval seeds is $0.18$ points (numeric) and $0.16$ points
(lexical), an order of magnitude below the effects. Under this instrument all six numeric seeds are
positive and the apparent reversal becomes $+0.05$. \textbf{To avoid any ambiguity about which
numbers are which:} the superseded $K{=}8$ numeric differences were
$\{+3.3,+2.1,+1.5,+1.9,-0.8,+4.2\}$ and the $K{=}32$ ones reported everywhere in the paper are
$\{+3.0,+1.6,+0.5,+3.2,+0.1,+3.1\}$. Wherever the body cites per-seed numeric differences it cites
the second list; the first appears only here, as the discarded measurement. The corrected estimates are \emph{smaller} than
our earlier ones: numeric $2.29\!\to\!1.91$ points with $d_z$ $2.4\!\to\!1.4$, lexical
$19.10\!\to\!18.51$, and both arms now reach two-sided $p{=}0.031$ at $n{=}6$. One column is not
comparable across the change: \texttt{distinct\_valid\_per\_item} counts distinct valid answers
among $K$ samples and so scales with $K$; at $K{=}32$ its ordering across arms reverses relative to
$K{=}8$, which is a further reason the registered diversity sub-prediction is reported as
unconfirmed.

\section{E6 per-seed results}
\label{app:e6seeds}
Table~\ref{tab:e6seeds} lists the six seeds behind Table~\ref{tab:propagation}. All figures are the precise re-evaluation of App.~\ref{app:evalprec} ($K{=}32$, three seeded repeats), so the base row is identical across training seeds by construction. The oracle arm is tight ($\sigma{=}0.003$); the authored arm varies more ($\sigma{=}0.026$) but every authored seed lies below every oracle seed (the best authored run, $0.764$, still falls short of the worst oracle run, $0.904$), so the gap is not a seed artifact. Seeds 4--6 were added after the first three to lift the arm past the $n{=}3$ significance floor (\S\ref{sec:propagation}); they were produced by the same script, hyperparameters, datasets, and keys, verified by checksum before the run.

\begin{table}[htbp]
\centering\small
\setlength{\tabcolsep}{8pt}
\begin{tabular}{lccc}
\toprule
Seed & base & authored key (pollution) & oracle key (pollution)\\
\midrule
1 & 0.567 & 0.694 (0.204) & 0.912 (0.039)\\
2 & 0.567 & 0.764 (0.155) & 0.911 (0.041)\\
3 & 0.567 & 0.711 (0.187) & 0.914 (0.039)\\
4 & 0.567 & 0.754 (0.162) & 0.911 (0.043)\\
5 & 0.567 & 0.700 (0.196) & 0.904 (0.041)\\
6 & 0.567 & 0.733 (0.170) & 0.914 (0.041)\\
\midrule
mean & 0.567 & 0.726 (0.179) & 0.911 (0.041)\\
\bottomrule
\end{tabular}
\caption{E6 per-seed oracle-valid rate (and trained-policy reward pollution). Training on the oracle key drives pollution to $\approx0.04$; the authored key holds it near $0.18$.}
\label{tab:e6seeds}
\end{table}

\section{Mitigation: probe-strength tradeoff}
\label{app:mitigation}
Beyond the canonical reference, holding out $k$ pooled known-correct solutions as probes gives the same collapse in false rejection (Table~\ref{tab:kprobe}, 14B): a single probe already suffices; additional probes trade a little yield for marginally lower FRR.

\begin{table}[htbp]
\centering\small
\setlength{\tabcolsep}{9pt}
\begin{tabular}{lccc}
\toprule
Gate & FRR & Yield & \\
\midrule
ungated              & 0.711 & 1.000 &\\
canonical reference  & 0.025 & 0.268 &\\
$k{=}1$ known-correct & 0.013 & 0.238 &\\
$k{=}2$ known-correct & 0.017 & 0.226 &\\
$k{=}3$ known-correct & 0.024 & 0.213 &\\
\bottomrule
\end{tabular}
\caption{Probe-strength tradeoff (Qwen2.5-14B, HumanEval+). One known-correct probe is enough to collapse false rejection; more probes slightly reduce yield.}
\label{tab:kprobe}
\end{table}

\section{Negative result and incident log}
\label{app:negative}
In the spirit of full reporting we record what did not work.
\textbf{Boundary entropy (N1).} We tested whether token-level entropy at the specification boundary predicts authoring error per item, as a cheap self-signal; the correlation was negligible ($r{=}{-}0.06$), so entropy is not a usable proxy for ``this authored member is wrong.''
\textbf{Registered sub-predictions (E6).} Active suppression of omitted-but-correct answers and diversity narrowing were registered but not confirmed at the one-epoch budget (omitted fraction and distinct-valid coverage stayed flat or rose); we report the first-order effect (starvation$+$pollution) and leave the distributional dynamics to longer horizons.
\textbf{Repair (E5).} The sample-then-verify prediction was registered and falsified (\S\ref{sec:repair}).
\textbf{Infrastructure incidents.} vLLM memory-profiling races on shared GPUs (mitigated by single-GPU $+$ enforce-eager $+$ free-memory guard); a remote GRPO run OOM'd under neighbor-job contention and was re-run locally; ms-swift required an added \texttt{msgspec} dependency. Full claim$\to$script$\to$log$\to$artifact provenance is maintained in the project records.

\section{Matched-information decomposition of the algorithmic gap}
\label{app:matched}
To test whether the complete-truth gap reflects verifier \emph{construction} or the classical free-recall-vs-recognition and compute asymmetries, we decompose it on the algorithmic construction with the candidate universe $U$ (the 15-word list) \emph{presented to every condition}, so candidate discovery is controlled. We score four ways of producing $\hat S$ as set-F1 against the gold set:
\textbf{A} free authoring (``list every word that $P$'', one call);
\textbf{B} one-shot label-all (``for each word say yes/no'', one call, take the yes-set);
\textbf{C} pointwise judging (one call per word, take the yes-set);
\textbf{D} $=$ C's yes-set assembled programmatically (identical set to C, so the pure-assembly floor is $0$ by construction).
The comparisons isolate: A$\to$B emission-vs-labeling, B$\to$C batching/context, C$\to$D pure assembly ($\approx0$); Table~\ref{tab:matched} reports the decomposition.

\begin{table}[htbp]
\centering\small
\setlength{\tabcolsep}{6pt}
\begin{tabular}{lccccc}
\toprule
Model & A free-author & B label-all & C pointwise & batching (C$-$B) & emission (B$-$A)\\
\midrule
Qwen2.5-7B  & 0.405 & 0.471 & 0.646 & $0.174$ & $0.067$\\
Qwen2.5-14B & 0.475 & 0.590 & 0.734 & $0.144$ & $0.116$\\
Qwen2.5-72B & 0.483 & 0.714 & 0.718 & \best{$0.004$} & \best{$0.231$}\\
\bottomrule
\end{tabular}
\caption{Matched-information decomposition (algorithmic, universe $U$ presented to all conditions; set-F1; assembly floor C$\to$D $=0$ by construction). With candidate discovery controlled, one-shot authoring still lags pointwise judging by $0.24$--$0.26$ at every scale. The decomposition is scale-revealing: the batching term vanishes ($0.174\!\to\!0.004$) while the emission term grows ($0.067\!\to\!0.231$), so at 72B the residual is almost entirely emission. Paired-bootstrap CIs support the trend. Part of the emission term is serialisation (\S\ref{sec:emission}).}
\label{tab:matched}
\end{table}

\begin{table}[htbp]
\centering\small
\setlength{\tabcolsep}{9pt}
\begin{tabular}{lcccccc}
\toprule
Model & $K{=}1$ & $2$ & $4$ & $8$ & $16$ & execution ceiling\\
\midrule
Qwen2.5-7B  & 0.381 & 0.472 & 0.556 & 0.646 & 0.707 & $\approx0.95$\\
Qwen2.5-14B & 0.594 & 0.679 & 0.719 & 0.771 & 0.797 & $\approx0.97$\\
Qwen2.5-72B & 0.544 & 0.635 & 0.714 & 0.798 & 0.842 & $\approx0.98$\\
\bottomrule
\end{tabular}
\caption{Compute--recall frontier (algorithmic; free authoring sampled $K$ times at $T{=}0.9$, union recall against gold). More authoring inference does raise recall, but it stays below the pointwise arm's \emph{recall} at every scale, even at $16\times$ budget, so the gap is not an artifact of granting judging more calls. The ceiling is a recall, not the set-F1 of Table~\ref{tab:matched}; recall upper-bounds F1, which is what makes this an a-fortiori argument.}
\label{tab:frontier_recall}
\end{table}

\noindent Three conclusions. (i) With candidates presented to both arms and the assembly floor at zero, the complete-truth gap is a set-\emph{emission} deficit, not recognition-vs-recall or programmatic assembly. (ii) The decomposition is scale-revealing: batching cost vanishes by 72B ($0.004$) while emission cost grows to $0.231$: scaling fixes labeling-under-context but not the act of emitting the set, which is the durable bottleneck. (iii) On the compute axis, scaling authoring samples raises union recall but saturates below execution at every scale, so ``the gap does not close'' is not an artifact of giving execution more calls; it holds even at $16\times$ the authoring sampling budget (Table~\ref{tab:frontier_recall}).

\section{Complete-truth numeric RLVR: full specification}
\label{app:numspec}
The complete-truth propagation run (\S\ref{sec:propagation-numeric}) carries the paper's causal claim, so we specify it exhaustively. Generator: a deterministic construction script, seed 0.

\begin{table}[htbp]
\centering\small
\setlength{\tabcolsep}{7pt}
\begin{tabular}{ll}
\toprule
Item & Value\\
\midrule
Candidate universe $V$ & $\{1,\dots,100\}$ (bounded, so every $S^{*}$ is finite)\\
Predicate form & range-scoped, ``between $a$ and $b$ that/whose \ldots''\\
Predicate families & divisible by 3 or 5; odd and not prime; two digits differ by 2;\\
 & prime; multiple of 4 not 8; perfect square; digit-sum $=S$; multiple of $K$\\
Windows $(a,b)$ & $a\in\{1,11,\dots,51\}$, widths $\{30,45,70\}$, plus $(1,100),(20,80),(30,90)$\\
Item filter & keep items with $4\leq|S^{*}|\leq30$, de-duplicated by description\\
Unique predicates & 160 total: \textbf{100 train / 60 held-out, zero overlap}\\
$|S^{*}|$ (held-out) & mean $10.3$, min $4$, max $29$\\
Train instances & 2{,}000 (the 100 train predicates replicated $\times20$, shuffled)\\
Authored key & Qwen2.5-14B, one-shot greedy, ``list every whole number \ldots'';\\
 & parsed integers intersected with $V$; empty parse backed off to $\{\min S^{*}\}$\\
Authored key quality & held-out recall $\mathbf{0.926}$, precision $\mathbf{0.987}$\\
Reference (oracle) key & the exact predicate extension, computed in code\\
Reward & $\mathbb{1}[\text{first integer in the completion}\in\text{arm's key}]$\\
Policy / algorithm & Qwen2.5-3B-Instruct, GRPO$+$LoRA, identical config to \S\ref{sec:propagation}\\
Evaluation & $K{=}32$ samples per held-out predicate at $T{=}1.0$, three seeded\\
 & repeats per adapter, averaged (App.~\ref{app:evalprec}); scored by $S^{*}$\\
Per-seed authored & $0.9432$, $0.9547$, $0.9554$, $0.9469$, $0.9731$, $0.9434$\\
Per-seed reference & $0.9734$, $0.9708$, $0.9604$, $0.9786$, $0.9736$, $0.9745$\\
Base (seed-independent) & $0.8877$\\
Paired differences & $+3.0$, $+1.6$, $+0.5$, $+3.2$, $+0.1$, $+3.1$ points; all positive,\\
 & mean $\mathbf{1.9}$, two-sided permutation $p{=}0.031$\\
\bottomrule
\end{tabular}
\caption{Full specification of the complete-truth numeric RLVR arm. Two properties matter for the causal reading: the universe is \emph{bounded}, so an authored key can in principle be complete (unlike WordNet); and train/eval predicates are \emph{disjoint}, so the gap is not memorisation. The authored key here is nearly complete (recall $0.926$), which is precisely why the tax is small, the effect scales with how badly the key is authored, and this task is easy to author.}
\label{tab:numspec}
\end{table}

\noindent\textbf{How to read the magnitude.} With 60 held-out predicates, 3 seeds, and an authored key of recall $0.926$ (full specification in Table~\ref{tab:numspec}), the $2.3$-point result should be read as a \emph{small-scale existence estimate} on an easy-authoring task, not as a general causal magnitude: it establishes that the effect is present and positive under exact ground truth (3/3 seeds), and its size is bounded by how good the key is. The larger lexical figure is reference-relative and cannot be substituted for it (\S\ref{sec:propagation}).

\section{Field corpus: stratified analysis}
\label{app:field}
The production comparison in \S\ref{sec:field} is observational. Here we hold the two covariates the corpus records fixed and ask whether the fill-in (authored-key) deficit survives within strata. It does, on both.

\begin{table}[htbp]
\centering\small
\setlength{\tabcolsep}{7pt}
\begin{tabular}{lccccc}
\toprule
Difficulty & $n$ fill-in & Pass \% & $n$ closed & Pass \% & Diff (pts)\\
\midrule
easy & 1267 & 96.53 & 1764 & 98.81 & $-2.28$\\
medium & 1261 & 97.46 & 1766 & 98.75 & $-1.29$\\
hard & 1250 & 95.44 & 1766 & 98.69 & $-3.25$\\
\midrule
\multicolumn{5}{l}{Pooled (Mantel--Haenszel-style common risk difference)} & $\mathbf{-2.27}$\\
\multicolumn{5}{l}{\quad $95\%$ CI $[-2.93,-1.61]$, $z=-6.76$} & \\
\bottomrule
\end{tabular}
\caption{Fill-in vs.\ closed formats (mcq $+$ msq pooled), stratified by item difficulty. The deficit is present at every level, so it is not an artifact of fill-in items being harder.}
\label{tab:field_diff}
\end{table}

\begin{table}[htbp]
\centering\small
\setlength{\tabcolsep}{9pt}
\begin{tabular}{cccc}
\toprule
Grade & Fill-in pass \% & MCQ pass \% & Diff (pts)\\
\midrule
1 & 95.83 & 99.07 & $-3.24$\\
2 & 96.81 & 98.41 & $-1.60$\\
3 & 95.89 & 99.58 & $-3.69$\\
4 & 96.80 & 99.26 & $-2.46$\\
5 & 97.71 & 99.09 & $-1.38$\\
6 & 97.10 & 99.52 & $-2.42$\\
7 & 96.46 & 98.99 & $-2.53$\\
8 & 98.04 & 99.02 & $-0.98$\\
9 & 95.29 & 99.49 & $-4.20$\\
10 & 95.81 & 97.95 & $-2.14$\\
11 & 97.40 & 99.48 & $-2.08$\\
12 & 96.70 & 99.46 & $-2.76$\\
\midrule
\multicolumn{3}{l}{$12/12$ grades lower; mean $-2.46$; sign test $p=2.4\times10^{-4}$} & \\
\bottomrule
\end{tabular}
\caption{Fill-in vs.\ MCQ pass rate within each grade. The authored-key format is lower in every one of the twelve grades, so the deficit is not driven by grade composition. Grade-wise pass rates are themselves flat within format ($\chi^2(11)=6.2$ mcq, $6.4$ fill-in; $p>0.84$), i.e.\ neither format's difficulty drifts systematically with grade.}
\label{tab:field_grade}
\end{table}

\noindent What this does and does not license (Tables~\ref{tab:field_diff} and~\ref{tab:field_grade}). It rules out difficulty and grade as explanations of the format gap; it does not rule out subject, generator model, answer-space size, evaluator-prompt, or length effects, which the recorded aggregates do not resolve. And the error \emph{taxonomy} (the $10{:}1$ omission ratio) remains evaluator-diagnosed, so it is directional evidence about the \emph{kind} of failure, not a calibrated rate.

\end{document}